\documentclass[11pt,titlepage]{article}
\usepackage{amsmath, amssymb}
\usepackage{mathtools}
\usepackage{algorithm}
\usepackage{algpseudocode}
\usepackage{fmtcount}
\usepackage{graphicx}
\usepackage{subcaption}
\usepackage{afterpage}
\usepackage[noadjust]{cite}
\usepackage[font=small,labelfont=bf]{caption}
\usepackage[labelfont=bf]{caption}
\usepackage{setspace}
\allowdisplaybreaks
\usepackage[title,titletoc]{appendix}
\usepackage{epsfig}
\usepackage{graphicx}
\usepackage{hhline}
\usepackage{latexsym}
\usepackage{amssymb}
\usepackage{amsmath,amscd}
\usepackage{multirow} 
\usepackage{array}
\usepackage{caption}
\usepackage{subcaption}
\usepackage{color}
\usepackage{natbib}
\usepackage{rotating}
\usepackage{graphicx,lscape}
\usepackage{enumitem}
\usepackage{cleveref}
\usepackage{amsmath, amsthm, amssymb}
\usepackage[table]{xcolor}

\newtheorem{prop}{Proposition}
\newtheorem{lemma}{Lemma}

\def\dodref#1#2!#3\relax{#1{#2}\ref{#2!#3}}

\long\def\comment#1{} 

\usepackage{graphicx}
\usepackage{epsfig}
\usepackage{graphics,color, graphicx}
\usepackage{latexsym}
\usepackage{fullpage}
\usepackage{booktabs,fixltx2e}
\usepackage[flushleft]{threeparttable}
\usepackage{amssymb,amsmath,amscd}
\usepackage{fancyvrb}
\usepackage{pdfpages}
\usepackage{fmtcount}
\usepackage{url}
\usepackage[margin=3.0cm]{geometry}

\DeclareMathOperator*{\argmin}{arg\!\min}
\title{\bf Robust Geodesic Regression
\medskip
}

\author{
\sc  Ha-Young Shin and Hee-Seok Oh\\ 
Department of Statistics\\
Seoul National University\\
Seoul 08826, Korea \\
\\
}

\begin{document}
\maketitle
	
\begin{abstract}

\noindent
This paper studies robust regression for data on Riemannian manifolds. Geodesic regression is the generalization of linear regression to a setting with a manifold-valued dependent variable and one or more real-valued independent variables. The existing work on geodesic regression uses the sum-of-squared errors to find the solution, but as in the classical Euclidean case, the least-squares method is highly sensitive to outliers. In this paper, we use M-type estimators, including the $L_1$, Huber and Tukey biweight estimators, to perform robust geodesic regression, and describe how to calculate the tuning parameters for the latter two. We show that, on compact symmetric spaces, all M-type estimators are maximum likelihood estimators, and argue in favor of a general preference for the $L_1$ estimator over the $L_2$ and Huber estimators on high-dimensional spaces. A derivation of the Riemannian normal distribution on $S^n$ and $\mathbb{H}^n$ is also included. Results from numerical examples, including analysis of real neuroimaging data, demonstrate the promising empirical properties of the proposed approach.
\vspace{\baselineskip}

\noindent
\textbf{Keywords}: Geodesic regression;  Manifold statistics;  M-type estimators;  Riemannian manifolds; Robust statistics. 
 
\end{abstract}

\section{Introduction} \label{intro}

\pagenumbering{arabic}

Much work has been done to generalize classical statistical methods for Euclidean data to manifold-valued data. Examples include principal geodesic analysis \citep{Fletcher2004}, analogous to principal component analysis, and geodesic regression \citep{Fletcher2013}, analogous to linear regression.

It is possible to conceptualize many types of data as lying on manifolds. Directional data in $\mathbb{R}^3$ can be visualized as lying on $S^2$; three-dimensional rotations can be represented as unit quaternions on $S^3$. Diffusion in the brain can be modeled by orientation distribution functions on $S^\infty$, which is approximated by $S^n$ for a high value of $n$. Hyperbolic space is well-suited for encoding hierarchical structures like graphs and trees. The space of symmetric positive-definite (SPD) matrices has many useful applications: In neuroimaging, diffusion tensor imaging data can be modeled as $3\times3$ SPD matrices \citep{Kim2014,Zhang2019}, and in computer vision, covariance matrices, which are SPD matrices, are used in appearance tracking \citep{Cheng2013}. For shape analysis, two-dimensional shape data can be represented as points on the complex projective space \citep{Fletcher2013,Cornea2017}, and the medial manifolds, $M(n)=(\mathbb{R}^3\times\mathbb{R}^+\times S^2\times S^2)^n$, provide models for the shapes of organs, such as the hippocampus \citep{Fletcher2004}.

Geodesic regression, which generalizes linear regression to manifolds, has been studied in recent years \citep{Fletcher2013,Kim2014,Cornea2017}. In this study, we explore a new robust approach to geodesic regression that accounts for potential outliers by using M-type estimators, such as the $L_1$, Huber, and Tukey biweight estimators. The key step of implementing robust geodesic regression is to solve the score (estimating) equations to estimate parameters in the regression model. We propose a gradient descent algorithm to carry out robust regression on Riemannian manifolds, calculating the gradients by considering Jacobi fields for simple regression and parallel transport for multiple regression. We further show that M-type estimators are equivalent to maximum likelihood estimators on certain manifolds as a theoretical justification for the proposed method. Thus, the proposed method can be considered as an extension of M-type estimators in Euclidean space to Riemannian manifolds. In addition, we provide the theoretical values of the cutoff parameters for Huber's and Tukey's biweight functions under certain situations. \cite{Zhang2019} addressed the issue of regression on manifolds in the presence of grossly corrupted data, but using a different method: explicitly modeling gross errors in the data and removing them to produce a corrected data set.

Beyond the aforementioned works, many other approaches to regression on manifolds have been proposed in the literature. \cite{Hinkle2014} developed a framework for polynomial regression on Riemannian manifolds that provides greater flexibility than geodesics. \cite{Du2014} studied geodesic regression on orientation distribution functions as elements of a Riemannian manifold. \cite{Hong2016} addressed the problem of intrinsic parametric regression on the Grassmannian manifold. As for nonparametric approaches to regression on manifolds, \cite{Davis2010} developed a regression analysis method that generalizes the conventional Nadaraya-Watson kernel method to manifold-valued data using the Fr\'echet expectation. \cite{Banerjee2016} presented a novel non-linear kernel-based nonparametric regression method for manifold-valued data with applications to real data collected from patients with Alzheimer's disease and movement disorders. \cite{Steinke2008}, \cite{Hein2009}, and \cite{Steinke2010} studied nonparametric regression between Riemannian manifolds. Of particular relevance to the current study is \cite{Hein2009}, who proposed a family of robust nonparametric kernel-smoothing estimators with metric space-valued outputs including a robust median-type estimator and the classical Fr\'echet mean. 

The rest of this paper is organized as follows. Section 2 briefly reviews the required background knowledge of differential geometry and geodesic regression. Section 3 presents the proposed methods for robust geodesic regression and a practical algorithm. A theoretical property of M-type estimators, their cutoff parameters, and some advantages of the $L_1$ estimator on high-dimensional spaces are also discussed. In Section 4, numerical experiments are presented, including simulation studies and a real data analysis of the shape of the corpus callosum in females with Alzheimer's disease. A summary and possible avenues for future research are provided in Section 5. Appendix \ref{appena} contains proofs and derivations, including the details of calculating the cutoff parameter for the Huber and Tukey biweight estimators, the efficiency of the $L_1$ estimator and the Riemannian normal distribution, including random generation, on $S^n$ and $\mathbb{H}^n$. Appendix \ref{appenb} provides details on the sphere, hyperbolic space, and Kendall's two-dimensional shape space. The data and R code used for the experiments are available at \url{https://github.com/hayoungshin1/Robust-Geodesic-Regression}.

This paper is based on the master's thesis \citep{Shin2020} of one of the authors, completed under the supervision of the other.

\section{Background}

\subsection{Differential Geometry Preliminaries} \label{preliminaries}

For a smooth manifold $M$ and a point $p\in M$, the tangent space $T_pM$ is the subspace consisting of all vectors tangent to $M$ at $p$. The elements of the tangent bundle of $M$, $TM$, take the form $(p,v)\in M\times T_pM$, so $TM$ is the disjoint union of the tangent spaces of $M$. A Riemannian manifold $M$ is a smooth manifold with a \textit{Riemannian metric}; that is, a family of inner products on the tangent spaces that smoothly vary with $p$. This metric can be used to measure lengths on $M$. A geodesic between two points on $M$ is the shortest length curve on $M$ that connects them; in Euclidean space, geodesics are straight lines. The \textit{geodesic} (or \textit{Riemannian}) \textit{distance} between two points is the length of this geodesic segment.

A geodesic $\gamma$ is defined by its initial point, $p=\gamma(0)\in M$ and velocity, $v=\gamma^\prime(0)\in T_{\gamma(0)}M$, where $\gamma^\prime(t)=d\gamma(t)/dt$. Then the \textit{exponential maps}, $\mathrm{Exp}_p:T_pM\rightarrow M$, are defined by $\mathrm{Exp}_p(v)=\gamma(1)$, and the \textit{logarithmic maps}, $\mathrm{Log}_p$, are the inverses of the exponential maps. The exponential and logarithmic maps are analogous to vector addition and subtraction in $\mathbb{R}^n$.  If $p_2$ is in the domain of $\mathrm{Log}_{p_1}$, then the geodesic distance between $p_1$ and $p_2$ is defined as $d(p_1,p_2)=\lVert\mathrm{Log}_{p_1}(p_2)\rVert$. It will useful to denote $\mathrm{Exp}_p(v)$ and $\mathrm{Log}_{p_1}(p_2)$ as $\mathrm{Exp}(p,v)$ and $\mathrm{Log}(p_1,q_2)$, respectively, taking $\mathrm{Exp}$ and $\mathrm{Log}$ as bivariate functions.

Take a differentiable curve $\gamma:[a,b]\rightarrow M$, not necessarily a geodesic, and a tangent vector $v\in T_{\gamma(a)}M$. The unique vector field $X$ along $\gamma$ that satisfies $X(a)=v$ and $\nabla_{\gamma^\prime}X=0$, where $\nabla$ is the Levi-Civita connection, is called the \textit{parallel transport} of $v$ along $\gamma$. For $p_1, p_2 \in M$, if there exists a uniquely minimizing connecting geodesic, we will denote parallel transport of the tangent vector $v \in T_{p_1}M$ to $T_{p_2}M$ along this geodesic by $\Gamma_{p_1 \rightarrow q_2}(v)$.

Given a family of geodesics $\{\gamma_s\}$, parametrized by and varying smoothly with respect to $s\in \mathbb{R}$, a Jacobi field is a vector field along the geodesic $\gamma_0$ describing how the geodesic family varies at each point $\gamma_0(t)$ of $\gamma_0$ with respect to $s$:
\vspace{-3mm}
\begin{equation}
J(t) = \frac{\partial \gamma_s(t)}{\partial s}\Big\rvert_{s=0}. \notag
\end{equation} 
The Jacobi field $J$ satisfies a second-order differential equation called the Jacobi equation:
\vspace{-3mm}
\begin{equation}
\frac{D^2}{dt^2}J(t)+R(J(t), \gamma^\prime(t))\gamma^\prime(t)=0,\notag
\end{equation}
where $R$ is the Riemann curvature tensor. Jacobi fields are important in the context of geodesic regression because they can be used to calculate the derivative of the exponential map on symmetric spaces. For details on Jacobi fields and their relation to the derivative of the exponential map, refer to \cite{doCarmo1992} and \cite{Fletcher2013}. 

\subsection{Geodesic Regression}

Given a dependent variable $y$ on a Riemannian manifold $M$ and an independent variable $x\in \mathbb{R}$, the simple geodesic regression model of \cite{Fletcher2013} is 
\begin{equation} \label{model1}
y=\mathrm{Exp}\big(\mathrm{Exp}(p,xv),\epsilon\big), 
\end{equation}
where $p\in M, v\in T_pM$, and $\epsilon\in T_{\mathrm{Exp}(p,xv)}M$. \cite{Kim2014} extended the simple model of (\ref{model1}) to a multiple regression model with several independent variables $x^1,\ldots,x^k\in \mathbb{R}$,
\begin{equation} \label{model2}
y=\mathrm{Exp}\big(\mathrm{Exp}(p,\sum_{j=1}^{k}x^jv^j),\epsilon\big),
\end{equation}
where $v^1,\ldots,v^k\in T_pM$ and $\epsilon$ is in the tangent space at $\mathrm{Exp}(p,\sum_{j=1}^{k}x^jv^j)$ (the superscripts are indices, not exponents). For convenience, let $V=(v^1,\ldots,v^k)$ and $Vx:=\sum_{j=1}^{k}x^jv^j$. Note that we follow the notations of \cite{Fletcher2013} and \cite{Kim2014}.  

Now given $N$ data points $(x_i,y_i)\in \mathbb{R}^k\times M$, we define the squared loss function $L$ by 
\begin{equation}
\label{loss}
L(p,V)=\sum_{i=1}^{N} \frac{1}{2} d\big(\mathrm{Exp}(p,Vx_i),y_i\big)^2, 
\end{equation}
where $d$ is the geodesic distance between points on $M$. Then the least-squares, or $L_2$, estimator $(\hat{p},\hat{V})\in M\times (T_pM)^k$ is
\begin{equation} \label{lsestimator}
(\hat{p},\hat{V})=\argmin_{(p,V) \in M\times \textcolor{blue}{(}T_pM\textcolor{blue}{)}^k} L(p,V).
\end{equation}
Unlike in the Euclidean case, the $L_2$ estimator of (\ref{lsestimator}) is generally obtained by a gradient descent algorithm because an analytical solution is typically not available. Letting $V=0$ in (\ref{lsestimator}), the resulting $\bar{y}$ is called the (sample) \textit{intrinsic} (or \textit{Karcher}) mean of the data points $\{y_i\}$, and their (sample) \textit{Fr\'echet} variance $s_y^2$ is defined as the corresponding loss in (\ref{loss}) at $p=\bar{y}$, multiplied by $2/N$, 
\begin{equation} \label{frechetvariance}
s_y^2=\frac{1}{N}\sum_{i=1}^{N} d\big(\bar{y},y_i\big)^2.
\end{equation}
Differentiating $L$ with respect to $p$ and each $v^j$ yields
\[
\nabla_pL=-\sum_{i=1}^Nd_p\mathrm{Exp}(p,Vx_i)^\dag e_i,~~\mbox{and}~~
\nabla_{v^j}L=-\sum_{i=1}^Nx_i^jd_{\textcolor{blue}{v}}\mathrm{Exp}(p,Vx_i)^\dag e_i
\]
for $j=1,\ldots,k$ and $e_i=\mathrm{Log}(\hat{y}_i,y_i)$. Here $\hat{y_i}=\mathrm{Exp}(\hat{p},\hat{V}x)$, $d_p\mathrm{Exp}(p,v):T_pM\rightarrow T_{\mathrm{Exp}(p,v)}M$ is the derivative of the exponential map with respect to the first argument, $p$, and $d_p\mathrm{Exp}(p,v)^\dag:T_{\mathrm{Exp}(p,v)}M\rightarrow T_pM$ represents the adjoint of this derivative with respect to the Riemannian metric; that is, $\langle d_p\mathrm{Exp}(p,v)(u),w\rangle_{p}=\langle u,d_p\mathrm{Exp}(p,v)^\dag(w)\rangle_{\mathrm{Exp}(p,v)}$ for $u\in T_pM$ and $w\in T_{\mathrm{Exp}(p,v)}M$. $d_v\mathrm{Exp}(p,v):T_pM\rightarrow T_{\mathrm{Exp}(p,v)}M$ and $d_v\mathrm{Exp}(p,v)^\dag:T_{\mathrm{Exp}(p,v)}M\rightarrow T_pM$ are analogously defined as the derivative of the exponential map with respect to the second argument, $v$, and its adjoint. On Riemannian symmetric spaces  (see Section \ref{symm}), these operators can be calculated explicitly using Jacobi fields, as in \cite{Fletcher2013}. For manifolds on which it is intractable to obtain the exact values of the adjoint derivatives in $\nabla_pL$ and $\nabla_{v_j}L$, parallel transports provide a practical alternative. \cite{Kim2014} and \cite{Zhang2019} approximate the gradients by
\[
\nabla_pL\approx-\sum_{i=1}^N\Gamma_{\hat{y}_i\rightarrow p}e_i~~\mbox{and}~~
\nabla_{v_j}L\approx-\sum_{i=1}^Nx_i^j\Gamma_{\hat{y}_i\rightarrow p}e_i, 
\]
using the notation for parallel transports introduced in Section \ref{preliminaries}. For manifolds on which parallel transports also have no analytic expression, they themselves can be approximated.

\section{M-type Estimators on Riemannian Manifolds}

We consider the classical linear regression model 
$
y=\beta_0+\beta_1x^1+\cdots+\beta_dx^k+\epsilon, 
$
where $y\in\mathbb{R}$, and $\beta_0$ and $\beta=(\beta_1,\ldots,\beta_k)^T\in \mathbb{R}^k$ take the roles of $p$ and $V$, respectively. The distribution of the errors $\epsilon$ can potentially be heavy-tailed, motivating the need for a robust estimator. It is well known that the $L_2$ estimator for $\beta_0$ and $\beta$ is sensitive to the presence of outliers. 

To avoid this problem, one can replace the least-squares criterion by a robust M-type criterion. The robust estimate of $(\beta_0$,$\beta)$ is defined as 
\[
(\hat{\beta_0},\hat{\beta})=\argmin_{(\beta_0,\beta)}\sum_{i=1}^{N} \rho(y_i-\beta_0-x_i^T\beta) 
\]
for $x_i=(x_i^1,\ldots,x_i^k)^T$, which for differentiable $\rho$ can be found by solving
\vspace{-3mm}
\[
\sum_{i=1}^{N} x_i\psi(y_i-\beta_0-x_i^T\beta)=0,
\]
where $\psi:=\rho^\prime$. The function $\rho(t)$ is typically convex and symmetric about 0, quadratic in the neighborhood of 0 and increasing at a rate slower than $t^2$ for large $t$. The robustness comes from the fact that, compared to the squared loss, $\rho(t)$ downweights extreme residuals. A common choice of $\rho$ is Huber's loss function which is a continuous function constructed piecewise from quadratic and linear segments,
\begin{displaymath}
\rho_H(t)=\left\{\begin{array}{ll} \frac{1}{2}t^2 &~~\mbox{if $|t|<c$}\\
c(|t|-\frac{1}{2}c) &~~ \mbox{otherwise}.\end{array} \right.
\end{displaymath}
Another popular loss function, Tukey's biweight function, is defined as
\[
\rho_T(t)=\left\{\begin{array}{ll}\frac{c^2}{6}\Big\{1-\big[1-(\frac{t}{c})^2\big]^3\Big\} &~~ \mbox{if $|t|<c$}\\
\frac{c^2}{6} &~~  \mbox{otherwise}. \end{array}  \right.
\]

To account for possible outliers, we now consider the use of M-type estimators to estimate $p$ and $V$. Generalizing from the above Euclidean setting to the manifold setting, we define a robust loss $L_\rho$ by
\begin{equation} \label{rholoss}
L_\rho(p,V)=\sum_{i=1}^{N} \rho\big(d(\mathrm{Exp}(p,Vx_i),y_i)\big).
\end{equation}
Then the M-type estimator is defined as the minimizer of (\ref{rholoss}), that is, 
\vspace{-3mm}
\begin{equation} \label{mestimator}
(\hat{p}_\rho,\hat{V}_\rho)=\argmin_{(p,V)\in M\times (T_pM)^k}L_\rho(p,V).
\end{equation}
For a fixed point $y\in M$, the gradient is expressed as 
\vspace{-3mm}
\begin{equation}
\nabla_p\rho(d(y,p))=-\frac{\rho^\prime(\lVert \mathrm{Log}(p,y)\rVert)}{\lVert \mathrm{Log}(p,y)\rVert}\mathrm{Log}(p,y), \notag
\end{equation}
so the M-type estimator is a solution to
\vspace{-3mm}
\begin{eqnarray*}
\nabla_pL_\rho&=&-\sum_{i=1}^N\frac{\rho^\prime(\lVert e_i\rVert)}{\lVert e_i\rVert}d_p\mathrm{Exp}(p,Vx_i)^\dag e_i=0,\\
\nabla_{v^j}L_\rho&=&-\sum_{i=1}^Nx_i^j\frac{\rho^\prime(\lVert e_i\rVert)}{\lVert e_i\rVert}d_{\textcolor{blue}{v}}\mathrm{Exp}(p,Vx_i)^\dag e_i=0
\end{eqnarray*}
for $j=1,\ldots,k$ and $e_i=\mathrm{Log}(\hat{y}_i,y_i)$. As in the least-squares case, gradients can be approximated using parallel transport:
\[
\nabla_pL_\rho\approx-\sum_{i=1}^N\frac{\rho^\prime(\lVert e_i\rVert)}{\lVert e_i\rVert}\Gamma_{\hat{y}_i\rightarrow p}e_i,~~\mbox{and}~~
\nabla_{v^j}L_\rho\approx-\sum_{i=1}^Nx_i^j\frac{\rho^\prime(\lVert e_i\rVert)}{\lVert e_i\rVert}\Gamma_{\hat{y}_i\rightarrow p}e_i.
\]

In this study, we consider the $L_1$ estimator with $\rho_{L_1}(t)=\lvert t\rvert$, Huber's estimator, and Tukey's biweight estimator as robust alternatives to the least squares estimator. Four loss functions that we consider are shown in Figure \ref{fig:losses}. For Huber's and Tukey's biweight estimators, it is necessary to determine the cutoff parameter $c$. The discussion of this topic is continued in Section \ref{calculating}.  

\begin{figure}[h!]
	\centering
	\begin{subfigure}[b]{0.24\linewidth}
		\includegraphics[width=\linewidth]{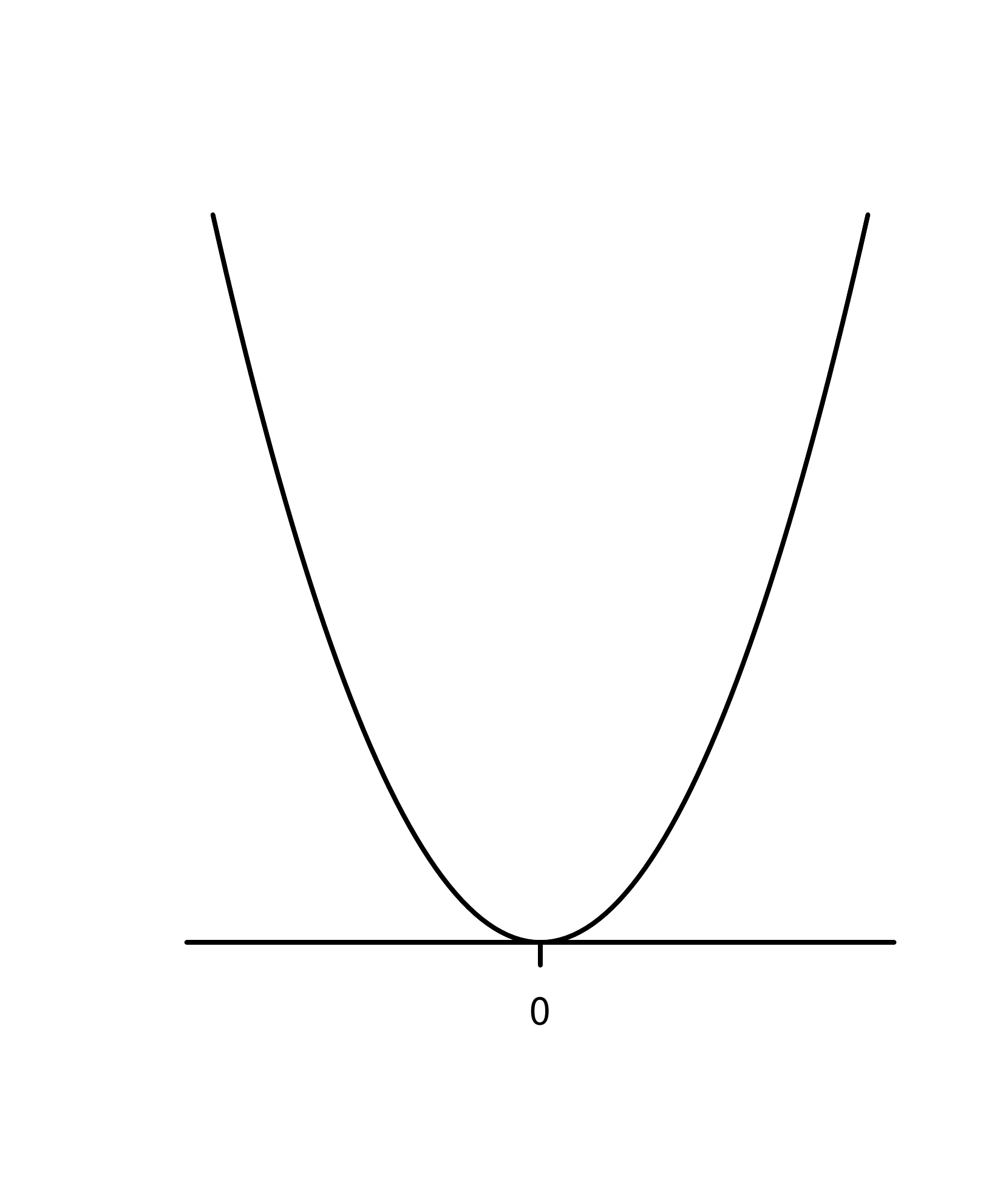}
		\caption{$L_2$ loss}
		\label{fig:l2loss}
	\end{subfigure}
	\begin{subfigure}[b]{0.24\linewidth}
		\includegraphics[width=\linewidth]{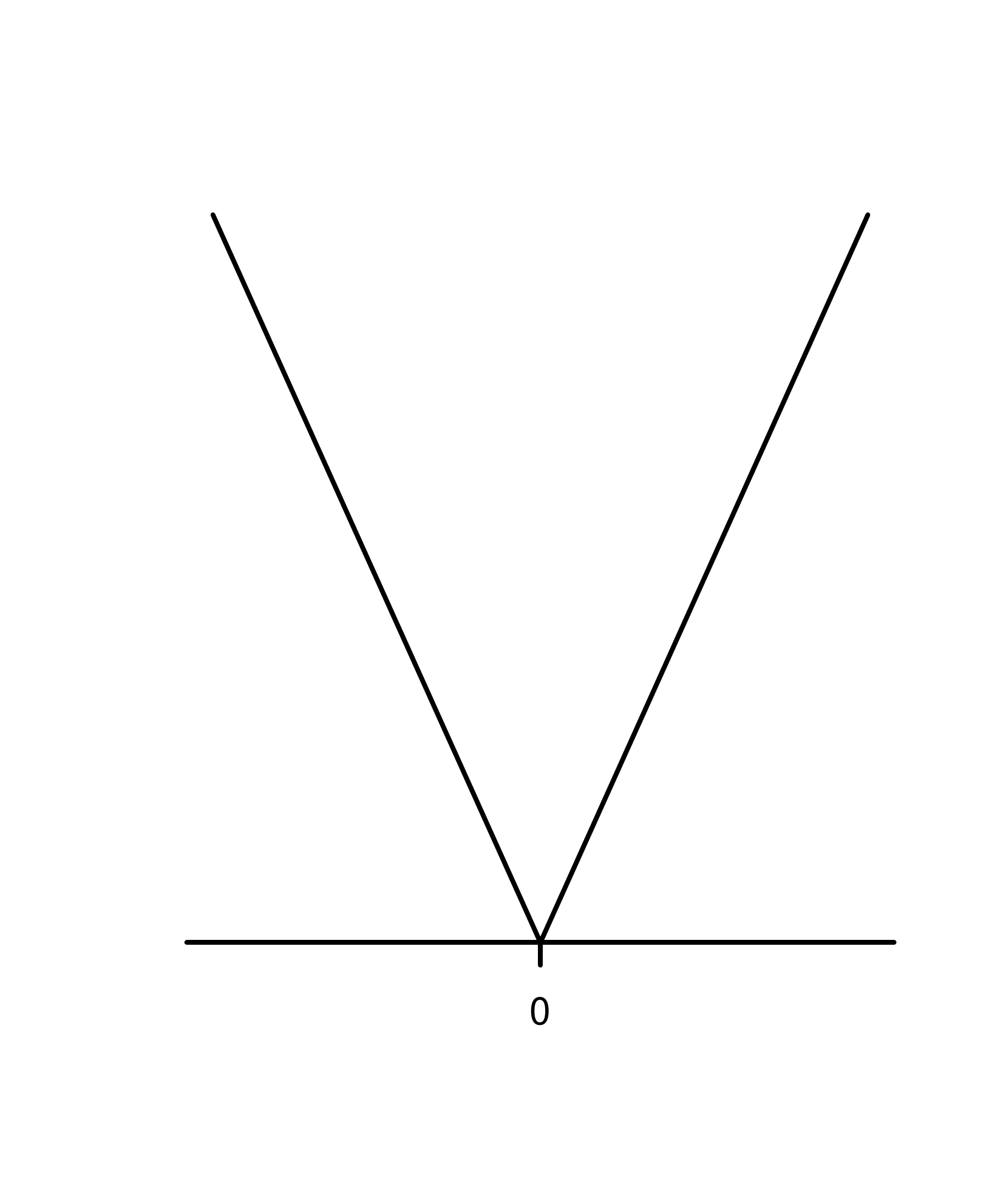}
		\caption{$L_1$ loss}
		\label{fig:l1loss}
	\end{subfigure}
	\begin{subfigure}[b]{0.24\linewidth}
		\includegraphics[width=\linewidth]{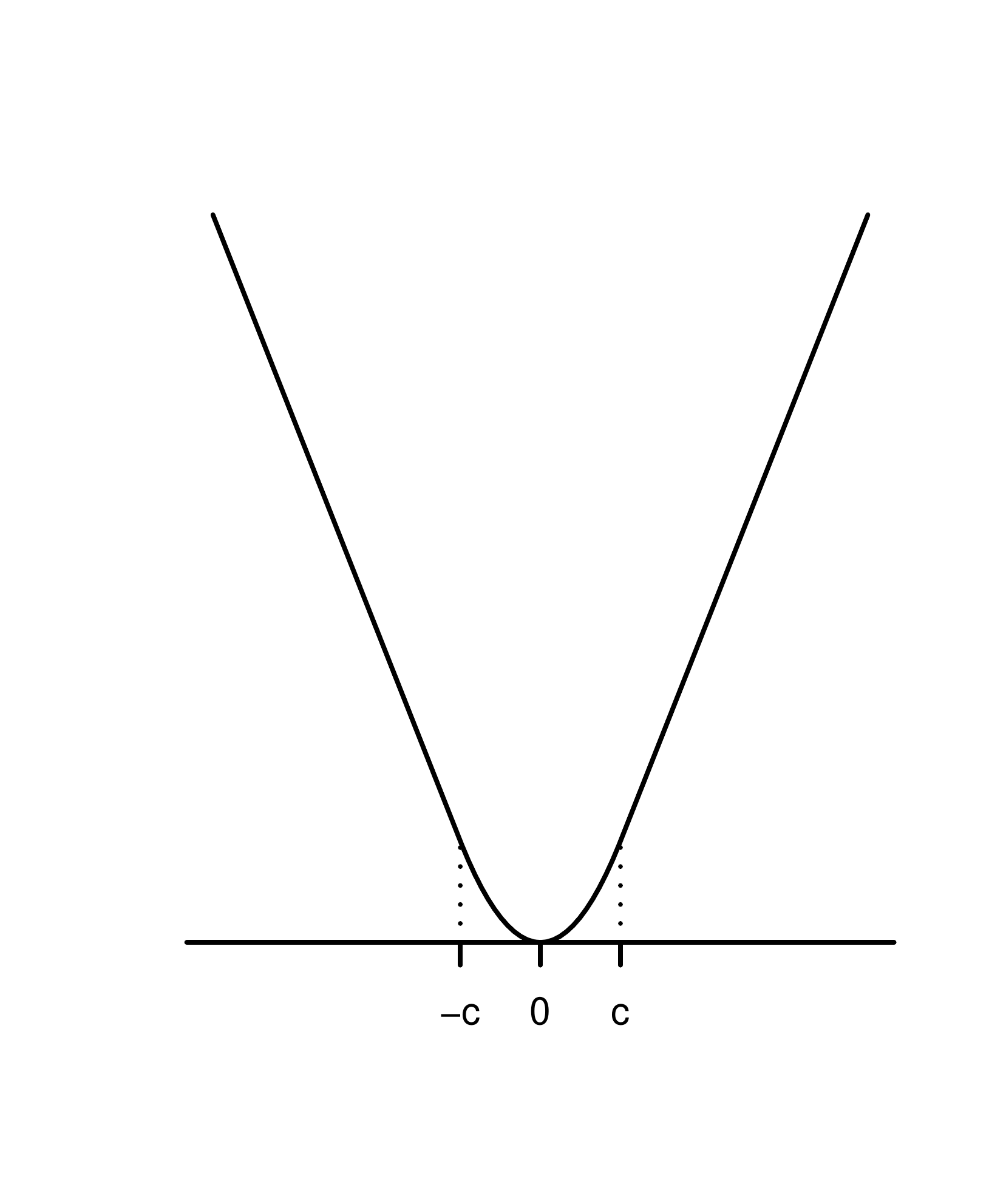}
		\caption{Huber's loss}
		\label{fig:huberloss}
	\end{subfigure}
	\begin{subfigure}[b]{0.24\linewidth}
		\includegraphics[width=\linewidth]{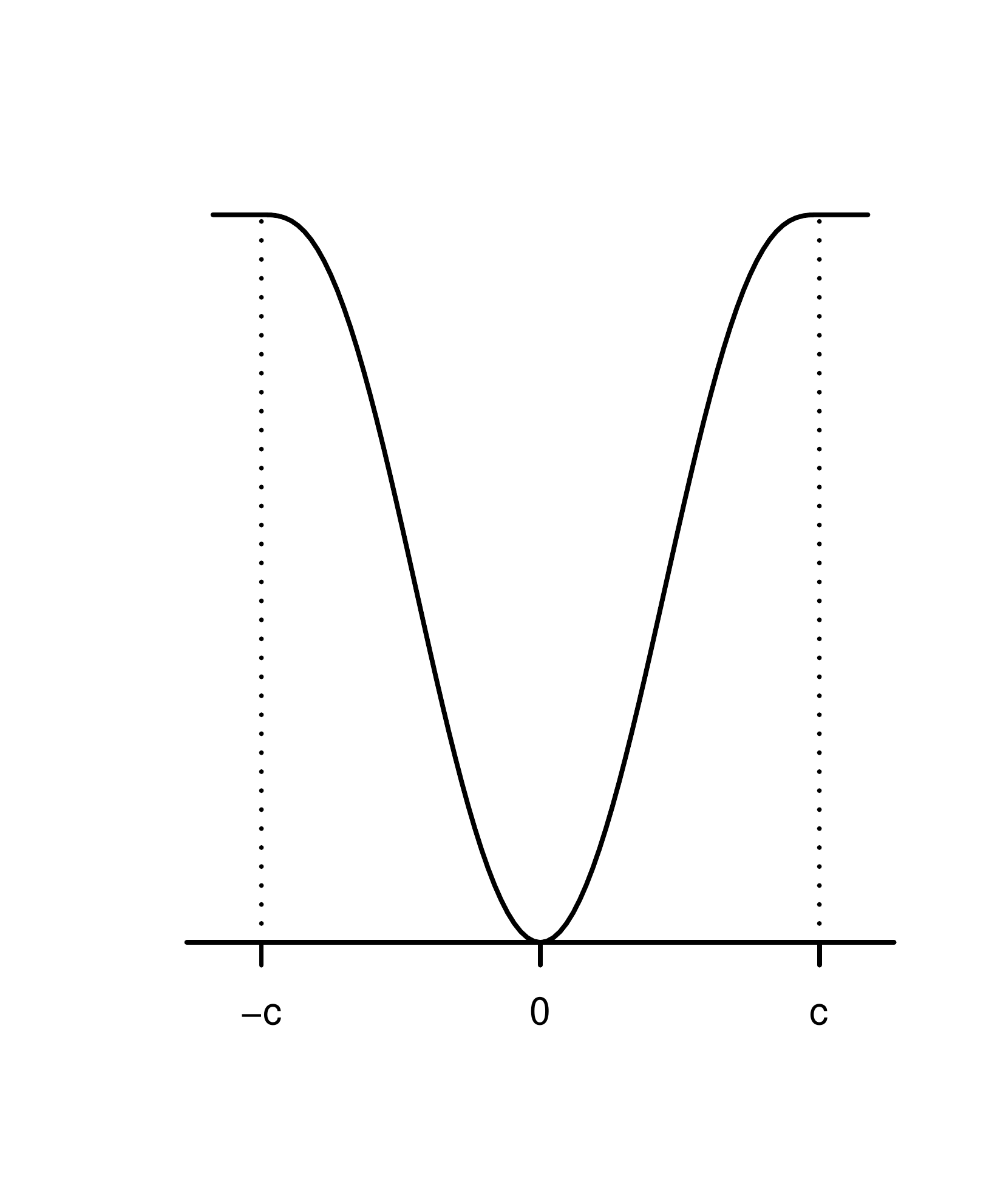}
		\caption{Tukey's biweight loss}
		\label{fig:tukeyloss}
	\end{subfigure}
	\caption{Graphs of the four loss functions.}
	\label{fig:losses}
\end{figure}

\subsection{M-type Estimators on Symmetric Spaces} \label{symm}

A symmetric space is a Riemannian manifold $M$ such that, for all $p\in M$, there exists an involutive isometry that fixes $p$ and reverses the geodesics that pass through $p$. Here, an isometry is a diffeomorphism that preserves the Riemannian distance, and an involutive isometry is an isometry that is its own inverse. The diameter of a manifold $M$ is defined as $\mathrm{diam}(M)=\sup_{p_1,p_2\in M} d(p_1,p_2)$. One of the properties of symmetric spaces is completeness, and it is a consequence of the Hopf-Rinow theorem that a complete manifold is compact if and only if it has finite diameter.

Important examples of symmetric spaces are the complete, simply-connected Riemannian manifolds of constant sectional curvature (spheres $S^n$, Euclidean spaces $\mathbb{R}^n$, and hyperbolic spaces $\mathbb{H}^n$) and the spaces of symmetric positive-definite matrices. Examples of compact symmetric spaces include the spheres $S^n$, compact Lie groups, Grassmanians, and Kendall's two-dimensional shape spaces $\Sigma_2^K$, which are equivalent to the complex projective spaces $\mathbb{C}P^{K-2}$.

For ordinary Euclidean data, some M-type estimators, such as the $L_1$ and Huber estimators, can be expressed as maximum likelihood (ML) estimators under a certain distribution for the errors, but others, including Tukey's biweight estimator and other estimators like Hampel's M-estimator, the truncated quadratic estimator, and Andrew's estimator, cannot. The best known example is the $L_2$ estimator, which is the ML estimator when the errors have a Gaussian distribution. On the other hand, on compact symmetric spaces, it can be shown that all M-type estimators of the geodesic regression model are ML estimators.

\begin{prop} \label{symmspaceprop}
	Let $M$ be a compact symmetric space with $x_1,\ldots,x_N\in\mathbb{R}^k$ and $y_1,\ldots,y_N\in M$. Take an M-type estimator whose loss function $\rho$ is bounded below, as any reasonable loss function would be. All solutions to (\ref{mestimator}), if any exist, are also maximum likelihood estimates for the geodesic regression model in (\ref{model2}) with $y_1,\ldots,y_N$ independent and conditionally distributed according to
	\begin{equation*}
	p(y_{i}|X=x_{\textcolor{blue}{i}})=f(y_{i};\mathrm{Exp}(p,Vx_{i}),b,\rho) 
	\end{equation*}
	for $i=1,\ldots,N$ and any $b>0$, where
	\vspace{-3mm}
	\begin{equation} \label{density}
	f(y_{i};\mu,b,\rho)=\frac{1}{D_{M}(\mu,b,\rho)}\mathrm{exp}\Bigg(-\frac{\rho(d(\mu,y_{\textcolor{blue}{i}}))}b\Bigg),
	\end{equation}
	with
	\begin{equation} \label{factor}
	D_M(\mu,b,\rho)=\int_M\mathrm{exp}\Bigg(-\frac{\rho(d(\mu,y))}{b}\Bigg)dy.
	\end{equation}
\end{prop}
The $L_2$, $L_1$, and Huber estimators can also be shown to be equivalent to ML estimators on complete and simply-connected Riemannian manifolds of constant sectional curvature; that is, spheres, Euclidean spaces, and hyperbolic spaces.

\begin{prop} \label{constcurvspaceprop}
	Let $M$ be a complete and simply-connected Riemannian manifold of constant sectional curvature with $x_1,\ldots,x_N\in\mathbb{R}^k$ and $y_1,\ldots,y_N\in M$. Let $\rho$ be either the $L_2$, $L_1$, or Huber loss function. All solutions to (\ref{mestimator}), if any exist, are also maximum likelihood estimates for the geodesic regression model in (\ref{model2}) with $y_1,\ldots,y_N$ independent and conditionally distributed according to
	\[
	p(y_i|X=x_i)=f(y_i;\mathrm{Exp}(p,Vx_i),b,\rho) 
	\]
	for $i=1,\ldots,N$ and some $b>0$, where $f(y_i;\mu,b,\rho)$ is defined as in Proposition \ref{symmspaceprop}.
\end{prop}

Proofs of Propositions \ref{symmspaceprop} and \ref{constcurvspaceprop} are provided in Appendix \ref{symmspacepropsproof}. In ($\ref{density}$), $b$ plays the role of a scale parameter. For example, $\rho_{L_2}(x)=\frac{1}{2}x^2$ and $b=\sigma^2$ for the $L_2$ estimator, so the estimator is equivalent to the ML estimator of the geodesic regression model with normal errors as defined in (\ref{normal}). We remark that Proposition \ref{symmspaceprop} is true for any manifold with finite volume that is homogeneous. Also, recall that the breakdown point of an estimator, which is the proportion of the data that can be arbitrarily changed without giving an arbitrarily wrong estimate, is a commonly used tool to quantify robustness. We note that the concept of the breakdown point is not meaningful on compact manifolds as distances between points on the manifold are bounded from above, so outliers cannot be made to be arbitrarily distant.

\subsection{Cutoff Parameters for the Huber and Tukey Estimators, and Efficiency of the $L_1$ Estimator} \label{calculating}
For univariate Euclidean data, the cutoff parameters for the Huber and Tukey biweight estimators are typically chosen to be $1.345\hat{\sigma}$ and $4.685\hat{\sigma}$, where $\hat{\sigma}=MAD/0.6745$, $MAD=\mathrm{Median}(\lvert e_1\rvert,\ldots,\lvert e_N\rvert)$ is the median absolute deviation, and $e_i=y_i-\hat{y}_i$. Here the value of 0.6745 is chosen because, for $Y \sim N(\mu,\sigma^2)$, $Pr(\lvert Y-\mu\rvert<0.6745\sigma)=1/2$, and the values of 1.345 and 4.685 are chosen so that, given i.i.d $Y_i \sim N(\mu,\sigma^2),~i=1,\ldots,N$, the asymptotic relative efficiency (ARE) of the sample M-type estimator for $\mu$, $\hat{Y}$, to the least-squares estimator, the sample mean $\bar{Y}$, is 95\% (i.e., $\lim_{N\to\infty}[\mathrm{Var}(\bar{Y})/\mathrm{Var}(\hat{Y})]=0.95$). By analogy, determining the cutoff parameter $c$ for the Huber and Tukey biweight estimators on a Riemannian manifold also requires two steps: (a) estimating $\sigma$ by $MAD/\xi$, and (b) finding the multiple of $\sigma$ that would give an ARE of the M-type estimator of location to the sample intrinsic mean of 95\% under a normal distribution. In the manifold case, $MAD=\mathrm{Median}(\lVert e_1\rVert,\ldots,\lVert e_N\rVert)$, with $e_i=\mathrm{Log}(\mathrm{Exp}(p,x_iv),y_i)$, and we have defined the variance of a manifold-valued random variable as in (\ref{frechetvariance}) and the relative efficiency as the ratio of two variances, as in the univariate Euclidean case. In deriving these parameters, we will use tangent space approximations, as is common in the literature on geometric statistics \citep{Fletcher2004, Kim2014}. To the best of our knowledge, the results and formulae in this section and Appendix A related to the approximate efficiencies of our robust estimators on manifolds have not appeared elsewhere in the literature.

The Riemannian normal distribution, as defined in \cite{Fletcher2020}, on a $k$-dimensional connected manifold $M$ has the following density:
\begin{equation}
\label{normal} 
f(y;\mu,\sigma^2)=\frac{1}{C_{M}(\mu,\sigma^2)}\mathrm{exp}\bigg(-\frac{d(y,\mu)^2}{2\sigma^2}\bigg),
\end{equation}
where
\begin{equation*}
C_{M}(\mu,\sigma^2)=\int_{M} \mathrm{exp}\bigg(-\frac{d(y,\mu)^2}{2\sigma^2}\bigg)dy.
\end{equation*}
This distribution exists and is well-defined only if $C_{M}(\mu, \sigma^2)$ is finite.

Given i.i.d $Y_i~(i=1,\ldots,N)$ distributed according to (\ref{normal}), we approximate the M-type estimator $\hat{Y}$ on the manifold by $\mathrm{Exp}(\mu,\hat{Y}^*)$, where $\hat{Y}^*$ is the M-type estimator for the points $Y_i^*:=\mathrm{Log}(\mu,Y_i)$ in the tangent space at $\mu$. As the tangent space is isomorphic to $\mathbb{R}^n$, we treat these points as belonging to $\mathbb{R}^n$ and consider the $Y_i^*$ to be distributed according to an isotropic multivariate Gaussian distribution with mean $0$ and variance $\sigma^2I_n$. Then letting $\sigma=1$  without loss of generality by choosing an appropriate length scale, the density of $Y_i^*$ is given by 
$f(y)=\phi_n(y)$ 
for $y\in\mathbb{R}^n$, where $Z_n\sim N_n(0,I_n)$ is the standard $n$-variate Gaussian random variable and $\phi_n=(2\pi)^{-\frac{n}{2}}\mathrm{exp}(-\frac{1}{2}\sum_{j=1}^n(y^j)^2)$ is its density. Here $y^j$ denotes the $j$th coordinate of $y$, not the $j$th power of $y$. These approximations are particularly reasonable for distributions with relatively little dispersion; how little will be examined in Section \ref{syntheticexperiments}. We will also assume $n\geq2$; the numbers when $n=1$, provided in Table \ref{table}, are already well known.

The calculations involved in determining the $c$ values are very tedious and lengthy; for details, refer to Appendix \ref{cutoffderivations}. Because of the aforementioned tangent space approximation, these results are exact when $M=\mathbb{R}^n$. Ultimately, the value of the constant $\xi$ in $MAD/\xi$ is 
\begin{equation}\label{xi}
\xi=\sqrt{2P^{-1}\Big(\frac{n}{2},\frac{1}{2}\Big)}, 
\end{equation}
where $P^{-1}(a,z)$ is the inverse of the lower regularized gamma function $P(a,z):=\gamma(a,z)/\Gamma(z)$, $\Gamma(z)$ is the gamma function, and $\gamma(a,z)$ is the lower incomplete gamma function.  In addition, the approximate AREs of the sample Huber and Tukey biweight estimators to the sample mean are, respectively,
\begin{equation} \label{huberare}
\mbox{ARE}_{H,L_2}(c,n)\approx A_H(c,n):=\frac{\Big\{\frac{n}{2}\gamma\big(\frac{n}{2},\frac{c^2}{2}\big) +2^{-\frac{3}{2}}c(n-1)\Gamma\big(\frac{n-1}{2},\frac{c^2}{2}\big)\Big\}^2}{\Gamma\big(\frac{n+2}{2}\big)\Big\{\gamma\big(\frac{n+2}{2},\frac{c^2}{2}\big)+\frac{c^2}{2} \Gamma\big(\frac{n}{2},\frac{c^2}{2}\big)\Big\}},
\end{equation}
and 
\begin{equation} \label{tukeyare}
\mbox{ARE}_{T,L_2}(c,n)\approx A_T(c,n):=\frac{\Big\{\frac{2(n+4)}{c^4}\gamma\big(\frac{n+4}{2},\frac{c^2}{2}\big)-\frac{2(n+2)}{c^2}\gamma\big(\frac{n+2}{2},\frac{c^2}{2}\big)+\frac{n}{2}\gamma\big(\frac{n}{2},\frac{c^2}{2}\big)\Big\}^2}{\splitfrac{\Gamma\big(\frac{n+2}{2}\big)\Big\{\gamma\big(\frac{n+2}{2},\frac{c^2}{2}\big)-\frac{8}{c^2}\gamma\big(\frac{n+4}{2},\frac{c^2}{2}\big) +\frac{24}{c^4}\gamma\big(\frac{n+6}{2},\frac{c^2}{2}\big)}{-\frac{32}{c^6}\gamma\big(\frac{n+8}{2},\frac{c^2}{2}\big)+\frac{16}{c^8}\gamma\big(\frac{n+10}{2},\frac{c^2}{2}\big)\Big\}}},
\end{equation}
where $c$ is the cutoff parameter and $\Gamma(a,z)$ is the upper incomplete gamma function. Note that these two equations assume without loss of generality that $\sigma=1$. Finally, we calculate the partial derivatives of (\ref{huberare}) and (\ref{tukeyare}) with respect to $c$, and then use the Newton-Raphson method to find $c_H$ and $c_T$, the values of $c$ for which the approximate $\mbox{ARE}_{H,L_2}$ and $\mbox{ARE}_{T,L_2}$, respectively, are 95\%.

Several nice properties of $A_{L_1}$, the approximate ARE of the $L_1$ estimator to the $L_2$ estimator calculated by letting $c\rightarrow 0$ for $A_H$ in (\ref{huberare}), are given in the following proposition.

\begin{prop} \label{l1prop}
(a) Defining $A_H(c,n)$ as in (\ref{huberare}),
		\begin{equation} \label{l1are}
		A_{L_1}(n):=\lim_{c\rightarrow 0}A_H(c,n)=\frac{\Gamma^2\big(\frac{n+1}{2}\big)}{\Gamma\big(\frac{n}{2}\big)\Gamma\big(\frac{n+2}{2}\big)}.
		\end{equation}
(b) $A_{L_1}(n)$, as defined in (\ref{l1are}), is increasing in $n\in \mathbb{Z}^+$. (c) $\lim_{n \rightarrow \infty}A_{L_1}(n)=1$.
\end{prop}
A proof of Proposition \ref{l1prop} is provided in Appendix \ref{l1propproof}. When $n=10$, the approximate $\mbox{ARE}_{L_1,L_2}$ is 0.95131, over 95\%. So in higher dimensions, the Huber estimator becomes unnecessary as the $L_1$ estimator is sufficiently efficient, and in very high-dimensional cases, even the $L_2$ estimator becomes unnecessary. The usual reasons for favoring the $L_2$ in the univariate Euclidean case are efficiency and ease of computation, but as Proposition \ref{l1prop} shows, on high-dimensional spaces the improvement in efficiency from using the $L_2$ over the $L_1$ estimator is negligible even with normal errors. For example, the approximate ARE $A_{L_1}(50)=0.99005$, over 99\%. Regarding computation, the geodesic regression problem is solved with a gradient descent algorithm regardless of choice of estimator, so this disadvantage of the $L_1$ estimator is also mitigated. On the other hand, the $L_1$ estimator is clearly more robust than the $L_2$ estimator. We thus argue that in general, the $L_1$ estimator is preferable to the $L_2$ estimator on high-dimensional spaces. Similarly, given its superior efficiency and considerable robustness, the $L_1$ estimator also has advantages over Tukey's estimator on high-dimensional spaces. 

Table \ref{table} gives the values of $\xi$ of (\ref{xi}), and the cutoff parameters for the Huber and Tukey biweight estimators $c_H$ and $c_T$, which are the multiples of $\hat{\sigma}$ for these estimators, respectively, for $n=1,2,3,4,5,6$ in $\mathbb{R}^k$. We also include the approximate $\mathrm{ARE}_{L_1,L_2}$, which in lower dimensions rapidly improves as $k$ increases.
\begin{table}[!h]
	\centering
	\caption{$\xi$, $c_H$ and $c_{\textcolor{blue}{T}}$ according to $n=1,\ldots,6$.}
	{\small
		\begin{tabular}{|c||c|c|c|c|c|c|}\hline
			$n$ & $1$ & $2$ & $3$ & $4$ & $5$ & $6$\\\hhline {|=||=|=|=|=|=|=|}
			$\xi$ & 0.67449 & 1.17741 & 1.53817 & 1.83213 & 2.08601 & 2.31260 \\\hline
			$c_H$ & 1.34500 & 1.50114 & 1.62799 & 1.73107 & 1.81202 & 1.86934 \\\hline
			$c_T$ & 4.68506 & 5.12299 & 5.49025 & 5.81032 & 6.09627 & 6.35622 \\\hline
			$A_{L_1}$ & 0.63662 & 0.78540 & 0.84883 & 0.88357 & 0.90541 & 0.92039 \\\hline 
		\end{tabular}
	}
	\label{table}
\end{table}

Even though the formulae in this section are calculated under the assumption that $n\geq2$, (\ref{xi}), (\ref{tukeyare}), and (\ref{l1are}) happen to still be valid when $n=1$, producing the figures in the $n=1$ column of Table \ref{table}, as are Proposition \ref{l1prop}(a) and \ref{l1prop}(b). The approximate ARE in (\ref{huberare}) can also be adjusted to work by removing the second summand in the curly brackets of the numerator.

\subsection{Implementation for M-type Estimators on Riemannian Manifolds}  

Here we discuss the implementation of the proposed M-type estimator on Riemannian manifolds. The gradient descent algorithm to find the solution of the robust geodesic regression problem in (\ref{mestimator}) is outlined in Algorithm 1 below. The purpose of $\lambda_{\max}/\rVert \nabla_pL_{\rho}\lVert$ in lines 14 and 28 are to prevent the steps for $p$, $-\lambda\nabla_pL_{\rho}$, from getting too large.

\begin{algorithm}[!h]
	\caption{Gradient descent algorithm for geodesic regression}
	\begin{algorithmic}[1]
		\State Input: $x_1,\ldots,x_N\in \mathbb{R}^k$, $y_1,\ldots,y_N\in M$ for $n$-dimensional $M$ and $\rho:\mathbb{R}\rightarrow\mathbb{R}^+$.
		\State Output: $p\in M, V\in \textcolor{blue}{(}T_pM\textcolor{blue}{)}^k$
		\State Initialize $p$ as the intrinsic mean of $\{y_1,\ldots,y_N\}$, $V$ as 0, and $\lambda_{max}$, and center $x$.
		\If {$\rho=\rho_H$ or $\rho_{T}$}
		\State Calculate $\xi$ using (\ref{xi}).
		\For {i in 1 to N}
		\State $e_i=\mathrm{Log}(\mathrm{Exp}(p,Vx_i),y_i)$
		\EndFor
		\State $MAD=\mathrm{Median}(\lVert e_1\rVert,\ldots,\lVert e_N\rVert)$
		\State Calculate $c_H$ or $c_T$ using Newton-Raphson's method on (\ref{huberare}) or (\ref{tukeyare}), respectively.
		\State $\hat{\sigma}=MAD/\xi$
		\State $c=c_H\hat{\sigma}$ or $c=c_{T}\hat{\sigma}$.
		\EndIf
		\State $\lambda = \mathrm{min}(0.1,\lambda_{max}/\rVert \nabla_pL_{\rho}\lVert)$ 
		\While{termination condition}
		\State $p_{new} = \mathrm{Exp}(p,-\lambda\nabla_pL_{\rho})$
		\State $V_{new} = \Gamma_{p\rightarrow p_{new}}(V-\lambda\nabla_VL_{\rho})$
		\If {$E_{\rho}(p,V) \geq E_{\rho}(p_{new},V_{new})$}
		\State $p=p_{new}$ and $V=V_{new}$
		\If {$\rho=\rho_H$ or $\rho_{T}$}
		\For {i in 1 to N}
		\State $e_i=\mathrm{Log}(\mathrm{Exp}(p,Vx_i),y_i)$
		\EndFor
		\State $MAD=\mathrm{Median}(\lVert e_1\rVert,\ldots,\lVert e_N\rVert)$
		\State $\hat{\sigma}=MAD/\xi$
		\State $c=c_H\hat{\sigma}$ or $c=c_T\hat{\sigma}$
		\EndIf
		\State $\lambda=\mathrm{min}(2\lambda,\lambda_{max}/\rVert \nabla_pL_{\rho}\lVert)$
		\Else
		\State $\lambda=\lambda/2$
		\EndIf
		\EndWhile
	\end{algorithmic}
\end{algorithm}

\section{Numerical Experiments} \label{experiments}

\subsection{Simulations on $S^n$ and $\mathbb{H}^n$} \label{syntheticexperiments}

The $n$-sphere $S^n$ is a space of constant positive sectional curvature, and hyperbolic space $\mathbb{H}^n$ is the analogous space of constant negative sectional curvature. Unlike $S^n$, it is not compact, and the exponential growth of the surface of a sphere in $\mathbb{H}^n$ with respect to its radius provides the basis for many interesting applications involving hierarchical data and tree-like structures, in which the number of nodes increases exponentially with depth. Our calculations on $\mathbb{H}^n$ were done using the hyperboloid model. Expressions for the the exponential maps and their derivatives, the logarithmic maps, parallel transports and other information about these manifolds can be found in Appendices \ref{appenbsphere} and \ref{appenbhyp}.

Our goal here is to evaluate the efficacy of the proposed M-type estimators on both positively and negatively curved spaces by performing simple geodesic regression on $S^2$ and $\mathbb{H}^2$ and multiple geodesic regression on $S^3$ and $\mathbb{H}^3$ using simulated data. The experimental setup was similar, but not identical, to the one used in \cite{Fletcher2013}. One difference was that we generated data points using the exact Riemannian normal distribution of (\ref{normal}) on $S^n$ and $\mathbb{H}^n$ as opposed to an approximation using an isotropic multivariate Gaussian distribution in the tangent space of the mean. The derivation of these exact distribution\textcolor{blue}{s} and how to generate random data from them are explained in Appendix \ref{spherehypnormal}. The parameters for the simple regression simulations on the two-dimensional spaces were set to $p=(1,0,0)$, $v=(0,\pi/4,0)$. For the multiple regression experiments on $S^3$ and $\mathbb{H}^3$, the parameters were set to $p=(1,0,0,0)$, $v^1=(0,\pi/4,0,0)$, and $v^2=(0,0,0,-\pi/6)$. Several different sample sizes were considered: $N=2^h$ for $h=2,3,\ldots,8$. The $x_i$ were generated from the uniform distribution on $[-1/2,1/2]$. The following types of noise were considered: 
\begin{itemize}
\item N: a Riemannian normal distribution with $\sigma=\pi/8$, 
\item T: a multivariate $t$-distribution in the tangent space with $\Sigma=(\pi/16)^2I_n$ and $\nu=4$, and
\item C: a contaminated Riemannian normal mixture distribution, that is, a mixture of two normal distributions, one with $\sigma=\pi/24$ and a probability of $0.9$, the other with $\sigma=\pi/6$ and a probability of $0.1$. 
\end{itemize}

The distributions in scenarios T and C are useful for examining robustness as they have heavier tails than the normal distribution, producing outliers.

We set the $\xi$, $c_H$, and $c_T$ values from Section \ref{calculating}, calculated to give an asymptotic efficiency of 95\% relative to the $L_2$ estimator, to 1.17741, 1.50114, and 5.12299, respectively, on $S_2$. On $S_3$, we used 1.53817, 1.62799, and 5.49025, respectively. For each $h$, $L=1024$ datasets were simulated. Then for each simulated set, four regression estimates were obtained by applying the $L_2$, $L_1$, Huber, and Tukey biweight estimators. For evaluation, we utilized the mean squared errors (MSE) for $\hat{p}$ and each $\hat{v}^j$, defined as
\begin{equation}\label{mses}
\mathrm{MSE}(\hat{p}):=\frac{1}{L}\sum_{\ell=1}^Ld(\hat{p}_\ell,p)^2, ~~\mbox{and}~~\mathrm{MSE}(\hat{v}^j):=\frac{1}{L}\sum_{\ell=1}^L\lVert\Gamma_{\hat{p}_\ell\rightarrow p}(\hat{v}_\ell^j)-v^j\rVert^2,
\end{equation}
where $\hat{p}_\ell$ and $\hat{v}_\ell^j$ are the estimates for $p$ and $v^j$ from the $\ell$th trial. These equations were taken, with modification, from the definitions of $\mathrm{MSE}(\hat{p})$ and $\mathrm{MSE}(\hat{v})$ in Section 5.1.2 of \cite{Fletcher2013}. 

\begin{figure}[p]
	\centering
	\includegraphics[width=0.95\linewidth]{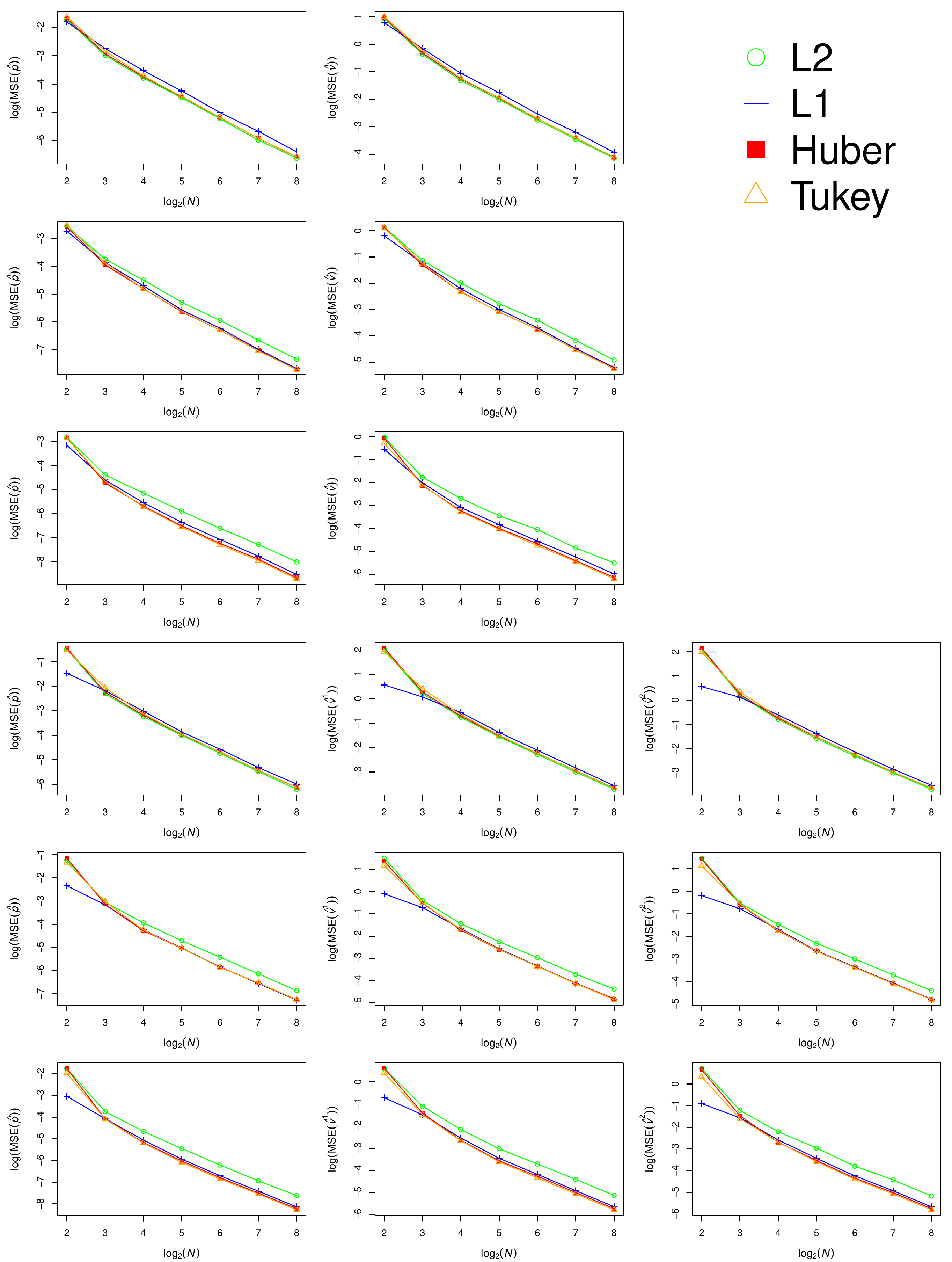}
	\caption{The effect of sample size, $N$, on various MSEs estimated from synthetic data. Both axes use logarithmic scales. The first three rows show the results on $S^2$; the last three rows show the results on $S^3$. The errors are type N in the first and fourth rows, type T in the second and fifth and type C in the third and sixth.}
	\label{fig:spheregraphs}
\end{figure}

\begin{figure}[p]
	\centering
	\includegraphics[width=0.95\linewidth]{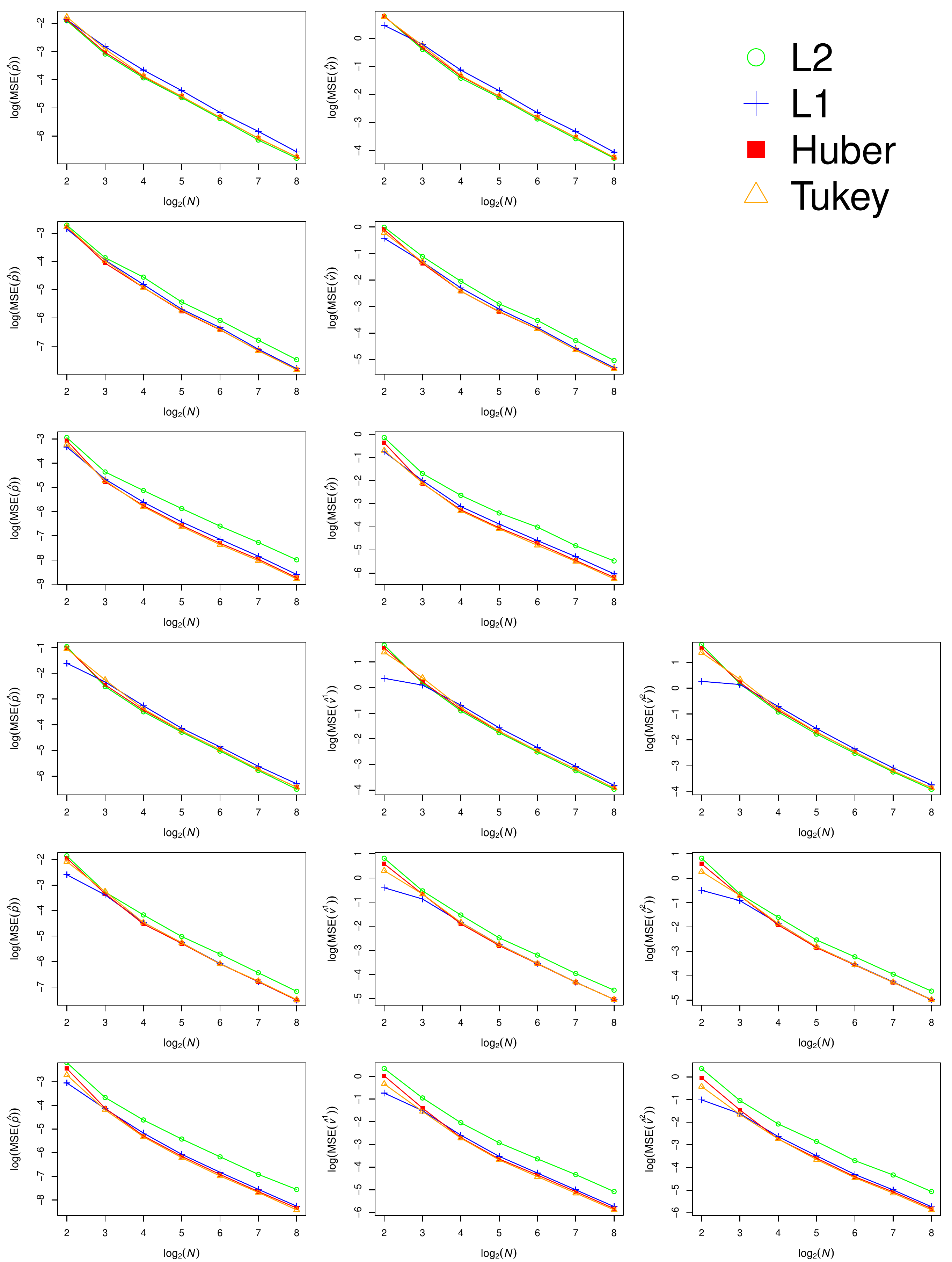}
	\caption{The effect of sample size, $N$, on various MSEs estimated from synthetic data. Both axes use logarithmic scales. The first three rows show the results on $\mathbb{H}^2$; the last three rows show the results on $\mathbb{H}^3$. The errors are type N in the first and fourth rows, type T in the second and fifth and type C in the third and sixth.}
	\label{fig:hyperbolicgraphs}
\end{figure}

Figures \ref{fig:spheregraphs} and \ref{fig:hyperbolicgraphs} show the results on the spheres and hyperbolic spaces, respectively. In every case, the MSEs all approached zero as sample size increases. We focus on the experiments in which the sample size is reasonably large (at least $2^3=8$ or $2^4=16$). The least-squares $L_2$ estimator performed the best for the normal type N errors, but the Huber and Tukey biweight estimators are almost as good. On the three-dimensional manifolds, even the $L_1$ estimator does not perform significantly worse than the other three estimators. For the noisy type T data, the $L_2$ estimator performs very poorly, while Huber's and Tukey's estimators have almost identical MSE values in both 2D and 3D cases and the $L_1$ estimator is slightly worse than these two in the 2D cases. For the contaminated mixture case C, the estimators, in order from worst to best, are the $L_2$, $L_1$, Huber, and Tukey biweight estimators, with the latter two having very similar results and the $L_2$ estimator being completely outclassed. When $N$ was small ($N=2^2=4$ or $2^3=8$), the $L_1$ estimator outperforms the others, significantly so in the 3D manifold-based experiments, regardless of the distribution of the errors.

As our robust estimators and their asymptotic relative efficiencies were derived using tangent space approximations (see Section \ref{calculating}), one might ask how well these approximations hold up and how they are affected by curvature. One way to test this empirically is by calculating the sample relative efficiencies of the robust estimators of location to the $L_2$ estimator, which is simply the Fr\'echet mean, with a large number $N$ of data points that are generated from the exact Riemannian normal distributions in Appendix \ref{spherehypnormal} for various values of $\sigma$; we will use $\pi/32$, $\pi/16$, $\pi/8$, $\pi/4$, and $\pi/2$. As $\sigma$ increases, the errors get larger and the influence of curvature becomes more pronounced. For example, a normal distribution with $\sigma = \pi/4$ on the unit sphere $S^n$, which has a constant sectional curvature of 1, is equivalent to a normal distribution with $\sigma = \pi/8$ on a sphere of radius $1/2$ with a constant sectional curvature of 4.

We generated data points from the normal distribution with $\mu=(1,0,0,0)$ on $S^3$ and $\mathbb{H}^3$ and $\sigma$ assuming the values of $\pi/32$, $\pi/16$, $\pi/8$, $\pi/4$, and $\pi/2$. The parameters are estimated using $N=2^8=256$ data points and this simulation is repeated $L=1024$ times as before to calculate sample variance. We used (\ref{frechetvariance}) to calculate the sample variances $s_{y,L_2}^2$, $s_{y,L_1}^2$, $s_{y,H}^2$ and $s_{y,T}^2$ for the $L_2$, $L_1$, Huber and Tukey biweight estimators, respectively, and calculated the relevant sample relative efficiencies by taking the appropriate ratios. Table \ref{efficiencytable} displays these results. Up to $\sigma=\pi/4$, these figures match closely with our expectations of an ARE to the $L_2$ estimator of 95\% for the Huber and Tukey biweight estimators and 84.88\% for the $L_1$ estimator, as listed in Table \ref{table}. However, when $\sigma=\pi/2$, the robust estimators are  all slightly more efficient than expected, with Huber's and Tukey's estimators approaching parity with the $L_2$ estimator, suggesting that if anything, the efficiencies of the robust estimators increase with more curvature, whether positive or negative. 

	\begin{table}
		\centering
		\caption{Relative efficiencies of the three robust estimators to the $L_2$ estimator when errors are normal.}
		{\small
			\begin{tabular}{|c|c|c|c|c|c|c|}\cline{3-7}
				\multicolumn{2}{c|}{\multirow{2}{*}{}} & \multicolumn{5}{c|}{$\sigma$} \\\cline{3-7}
				\multicolumn{2}{c|}{} & $\pi/32$ & $\pi/16$ & $\pi/8$ & $\pi/4$ & $\pi/2$ \\\hhline {|=|=|=|=|=|=|=|}
				\multirow{3}*{$S^3$} & $s_{y,L_2}^2/s_{y,L_1}^2$ & 0.8408682 & 0.8346334 & 0.8502927 & 0.8490783 & 0.9502378 \\\cline{2-7}
				& $s_{y,L_2}^2/s_{y,H}^2$ & 0.9431518 & 0.9495724 & 0.9471038 & 0.9612558 & 0.9839542 \\\cline{2-7}
				& $s_{y,L_2}^2/s_{y,T}^2$ & 0.9454626 & 0.9487272 & 0.9456605 & 0.9637054 & 1.0052173 \\ \hhline {|=|=|=|=|=|=|=|}
				\multirow{3}*{$\mathbb{H}^3$} & $s_{y,L_2}^2/s_{y,L_1}^2$ & 0.8408384 & 0.8347415 & 0.8487818 & 0.8508665 & 0.9126134 \\\cline{2-7}
				& $s_{y,L_2}^2/s_{y,H}^2$ & 0.9431112 & 0.9495070 & 0.9458978 & 0.9643541 & 0.9757066 \\\cline{2-7}
				& $s_{y,L_2}^2/s_{y,T}^2$ & 0.9454040 & 0.9487923 & 0.9449126 & 0.9654057 & 0.9833246 \\\hline
			\end{tabular}
		}
		\label{efficiencytable}
	\end{table}
\vskip 3mm

\begin{figure}[!h]
	\centering
	\begin{subfigure}[b]{0.32\linewidth}
		\includegraphics[width=\linewidth]{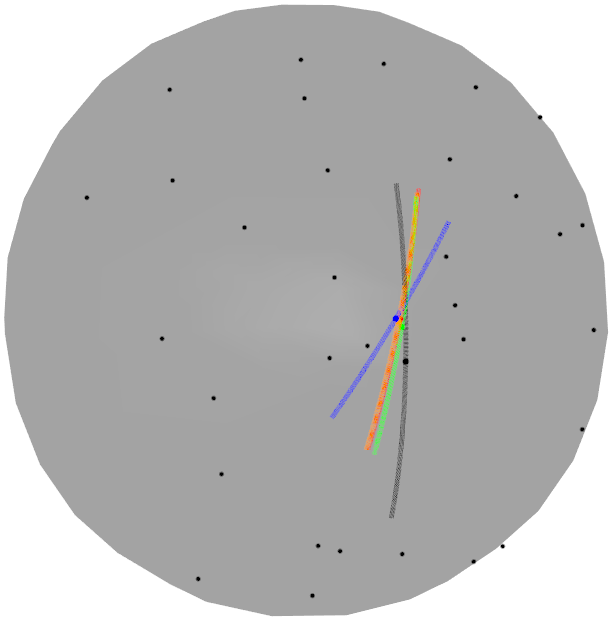}
		\caption{Type N errors on $S^2$}
		\label{fig:spheres1}
	\end{subfigure}
	\begin{subfigure}[b]{0.32\linewidth}
		\includegraphics[width=\linewidth]{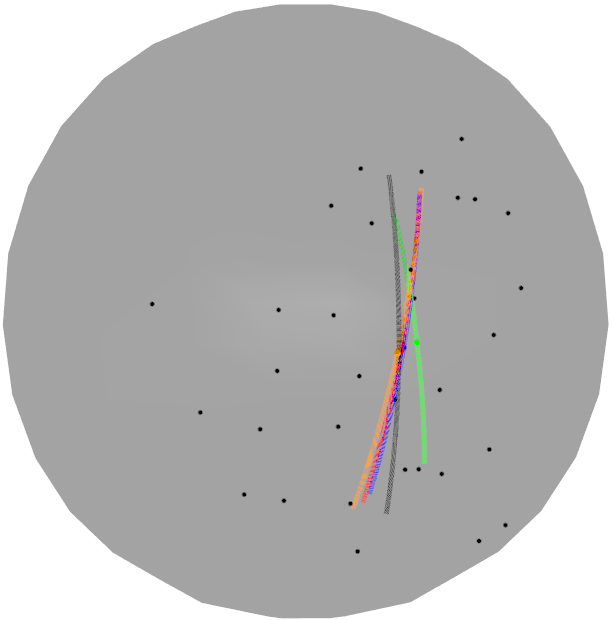}
		\caption{Type T errors on $S^2$}
		\label{fig:spheres2}
	\end{subfigure}
	\begin{subfigure}[b]{0.32\linewidth}
		\includegraphics[width=\linewidth]{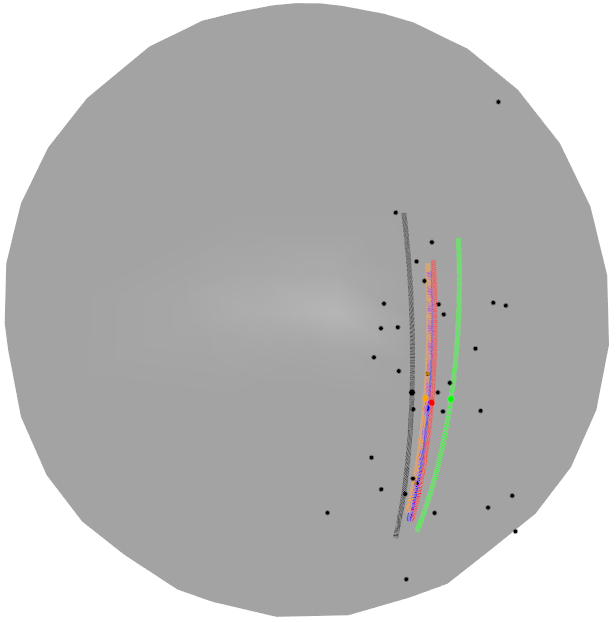}
		\caption{Type C errors on $S^2$}
		\label{fig:spheres3}
	\end{subfigure}

	\begin{subfigure}[b]{0.32\linewidth}
	    \includegraphics[width=\linewidth]{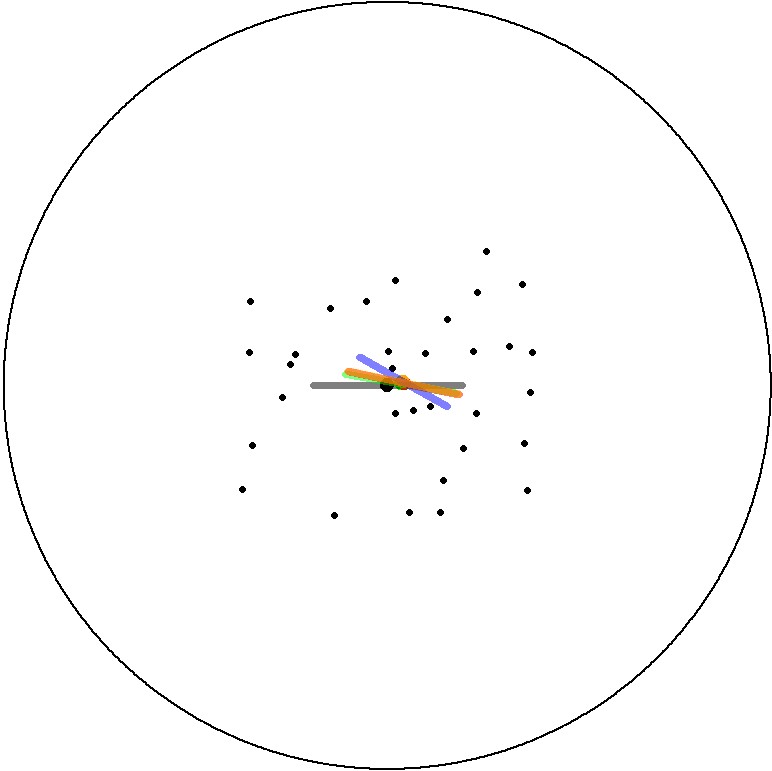}
	    \caption{Type N errors on $\mathbb{H}^2$}
	    \label{fig:hyperbolics1}
    \end{subfigure}
	    \begin{subfigure}[b]{0.32\linewidth}
		\includegraphics[width=\linewidth]{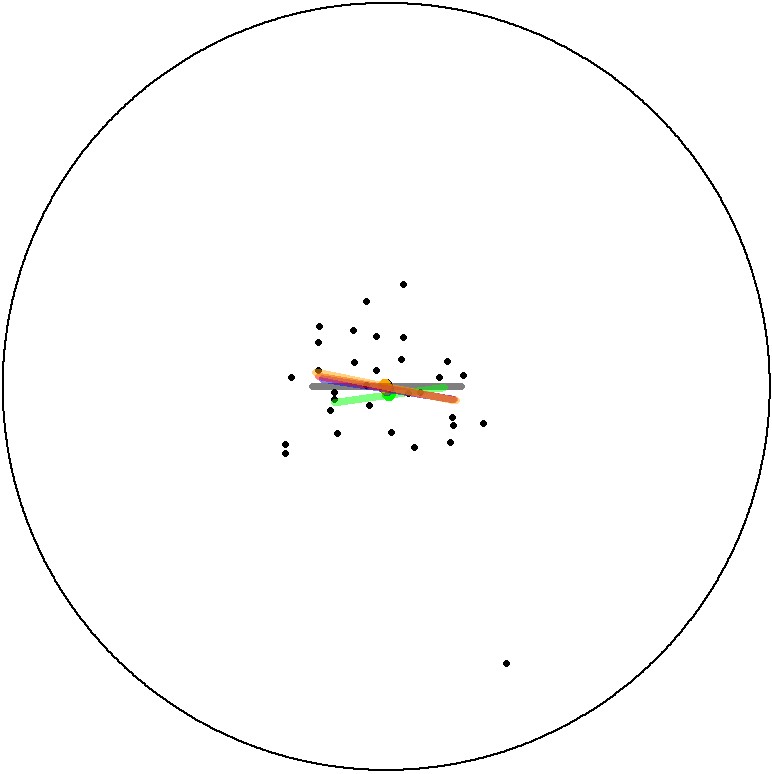}
		\caption{Type T errors on $\mathbb{H}^2$}
		\label{fig:hyperbolics2}
	\end{subfigure}
	\begin{subfigure}[b]{0.32\linewidth}
		\includegraphics[width=\linewidth]{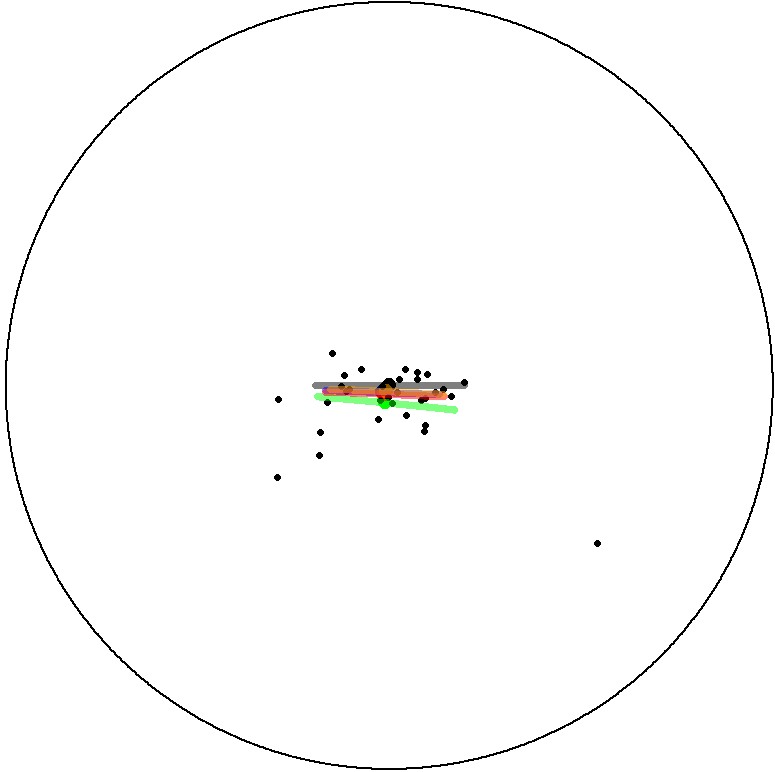}
		\caption{Type C errors on $\mathbb{H}^2$}
		\label{fig:hyperbolics3}
	\end{subfigure}
	\caption{Examples of simulations in the simple regression case on $S^2$ and $\mathbb{H}^2$ using different types of noise. The sample size is $2^5=32$ and the small black dots are the $y_i$. The images each show 5 geodesics from $\gamma(-\frac{1}{2})$ to $\gamma(\frac{1}{2})$; that is, from $\mathrm{Exp}(p,-\frac{1}{2}v)$ to $\mathrm{Exp}(p,\frac{1}{2}v)$. $\gamma(0)$, or $p$, is indicated by a large dot. The true geodesics are black, the $L_2$ solutions are green, the $L_1$ solutions blue, the Huber solutions red and the Tukey biweight solutions orange.} 
		\label{fig:examples}
\end{figure}

Figure \ref{fig:examples} shows an example simulation for each of the N, T, and C scenarios on $S^2$ and $\mathbb{H}^2$. Hyperbolic space has been visualized using the Poincar\'e ball model, which is briefly introduced in Appendix \ref{appenbhyp}. Note that in this model, distances increase exponentially as one approaches the boundary of the circle, which is at infinite, and that the geodesics, if extended, would appear as either arcs of circles that intersect the boundary at right angles or diameters. Figures \ref{fig:spheres2}, \ref{fig:spheres3}, \ref{fig:hyperbolics2} and \ref{fig:hyperbolics3}, in which the presence of outliers is clearly visible, illustrate the superior robustness properties of the other three estimators over the $L_2$ estimator, while Figure \ref{fig:spheres1} demonstrates that even in the normal case, the other estimators, with arguably the exception of the $L_1$ estimator, do not perform significantly worse than the $L_2$ estimator.

\subsection{Real Data Analysis: Corpus Callosum Shape Data} \label{analysis}

Mathematically, a shape refers to the geometry of an object after translation, scaling, and rotation have been removed. Kendall's two-dimensional shape space $\Sigma_2^K$ is the set of two-dimensional $K$-gon shapes, that is, the set of all possible non-coincident $K$-configurations in the two-dimensional plane modulo translation, scaling, and rotation, and is a compact symmetric space. For details on the structure of $\Sigma_2^K$, including the exponential map and its derivative, the logarithmic map and parallel transport, refer to Appendix \ref{appenbkendall}.

The corpus callosum, the largest white matter structure in the human brain, is a major nerve tract that connects the two cerebral hemispheres, facilitating interhemispheric communication. In this section, we perform simple geodesic regression with M-type estimators to analyze the relationship between the shape of the corpus callosum and age in older females with Alzheimer's disease (AD). We have used the preprocessed data provided by \cite{Cornea2017} on their website~ \url{http://www.bios.unc.edu/research/bias/software.html}. The planar shape data, obtained from the mid-sagittal slices of magnetic resonance images (MRI), are from the Alzheimer's disease neuroimaging initiative (ADNI) study. As mentioned above, the 88 female subjects with AD, whose ages range from 55 to 92, are the focus of this analysis, though the dataset contains data for both males and females with and without AD. Each shape is extracted from the MRI and segmented using the \texttt{FreeSurfer} and \texttt{CCseg} packages, resulting in a 50-by-2 matrix. The rows of this matrix give the planar coordinates of $K=50$ landmark points on the boundary of the shape, with enforced correspondences between the landmarks of different subjects.

Because the real dimension of the manifold is $2K-4=96\geq10$, the $L_1$ estimator is already efficient enough to make the Huber estimator unnecessary. Indeed, under the Euclidean, tangent space approximation, $\mathrm{ARE}_{L_1,L_2}=0.99481$. Therefore, we have only used the $L_2$, $L_1$, and Tukey biweight estimators to analyze this dataset. Using (\ref{xi}) and (\ref{tukeyare}), we calculated $\xi$ and $c_T$ to be 9.763 and 14.723, respectively. Geodesic regression is carried out six times. First, we apply the three estimators to the original data, giving $(\hat{p}_{L_2}, \hat{v}_{L_2})$, $(\hat{p}_{L_1}, \hat{v}_{L_1})$, and $(\hat{p}_T, \hat{v}_T)$; we use $(\hat{p}_{L_2}, \hat{v}_{L_2})$ as the baseline for comparison. Then we intentionally generate outliers by tampering with the data: for 20 of the 88 subjects, the shapes of their corpus callosums are flipped (reflected shapes are not considered equivalent in Kendall's shape space, for good reason). This causes the flipped points to be quite distant from the unflipped ones; the average distance between the 68 untampered points is 0.0802, while the average distance between those points and the 20 tampered ones is over 8 times larger at 0.6462. The three estimators are applied to this tampered dataset, resulting in $(\hat{p}^\prime_{L_2}, \hat{v}^\prime_{L_2})$, $(\hat{p}^\prime_{L_1}, \hat{v}^\prime_{L_1})$, and $(\hat{p}_T^\prime, \hat{v}_T^\prime)$.

\begin{figure}[!h]
	\centering
	\begin{subfigure}[b]{0.22\linewidth}
		\includegraphics[width=\linewidth]{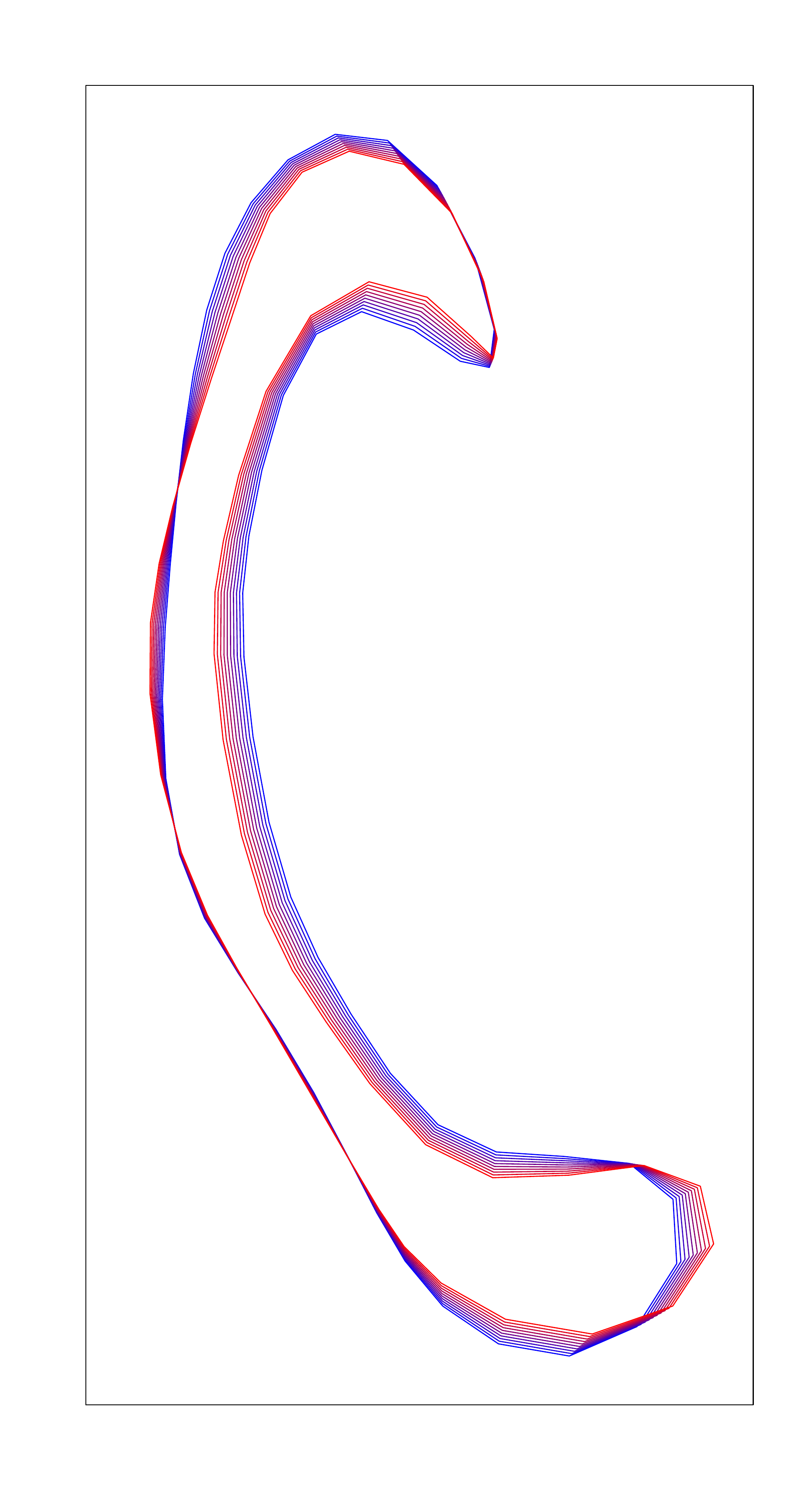}
		\caption{Untampered $L_2$}
		\label{fig:cc0}
	\end{subfigure}
	\begin{subfigure}[b]{0.22\linewidth}
		\includegraphics[width=\linewidth]{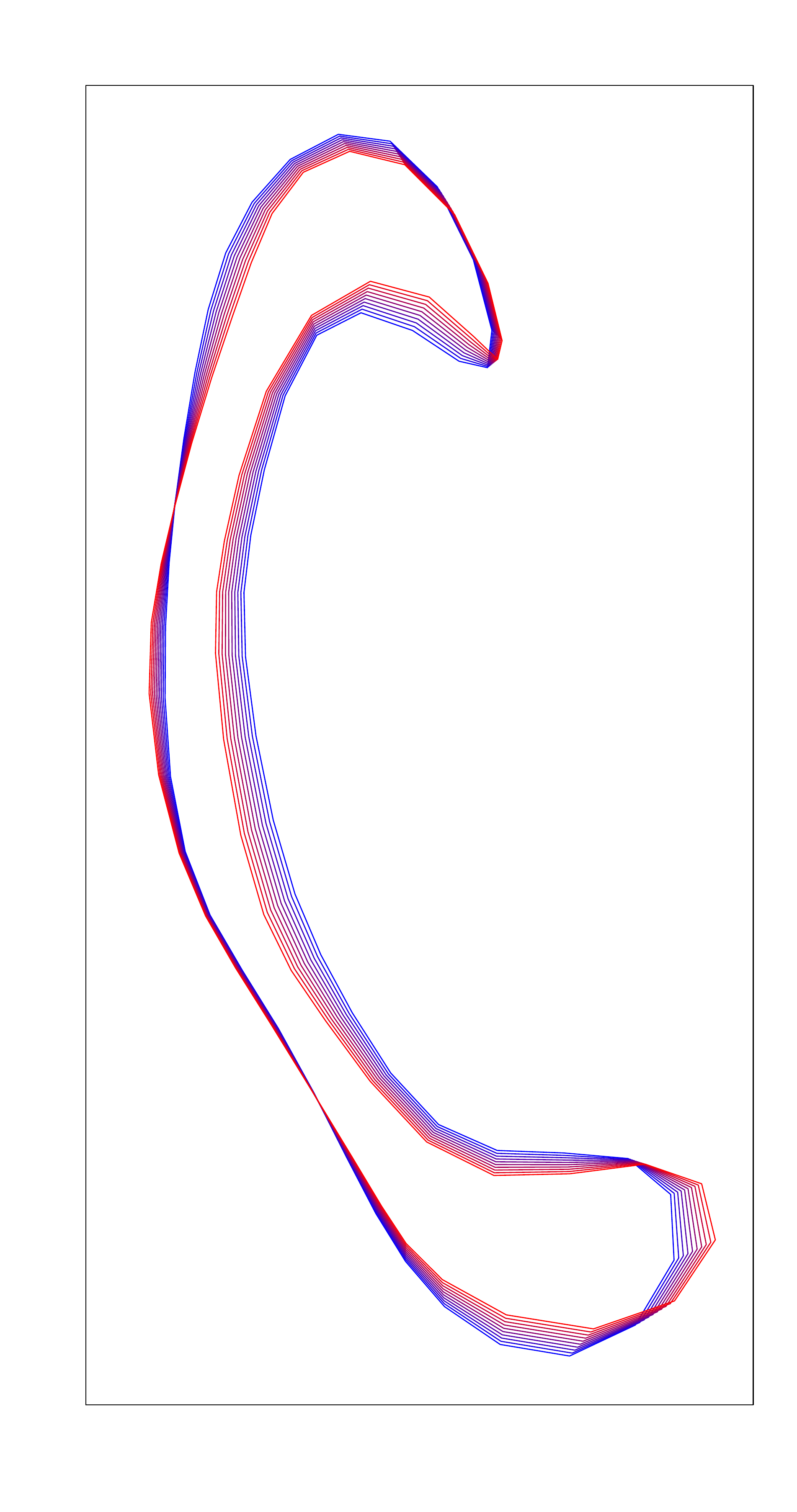}
		\caption{Untampered $L_1$}
		\label{fig:cc1}
	\end{subfigure}
	\begin{subfigure}[b]{0.22\linewidth}
		\includegraphics[width=\linewidth]{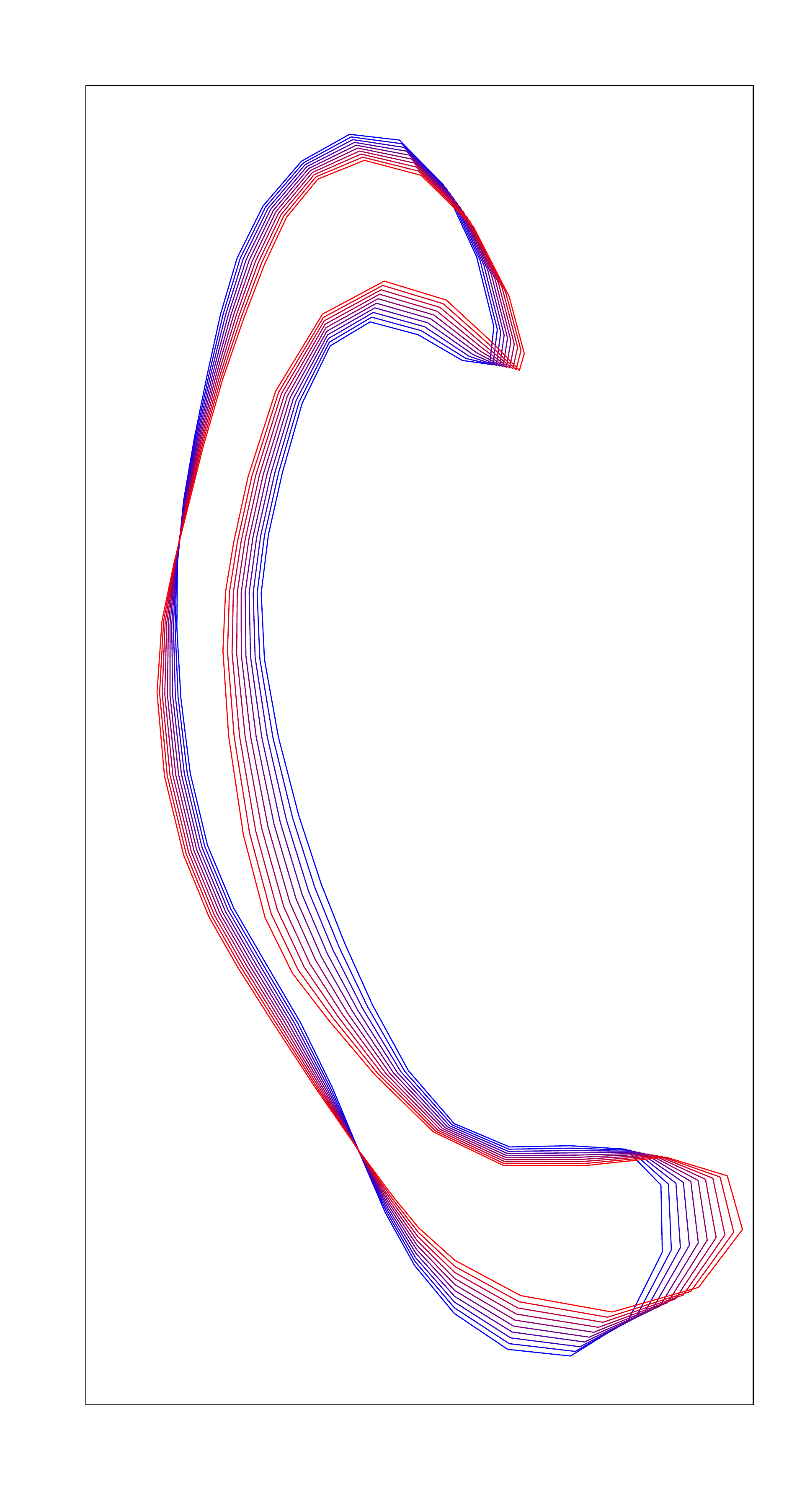}
		\caption{Untampered Tukey}
		\label{fig:cc2}
	\end{subfigure}

	\begin{subfigure}[b]{0.22\linewidth}
		\includegraphics[width=\linewidth]{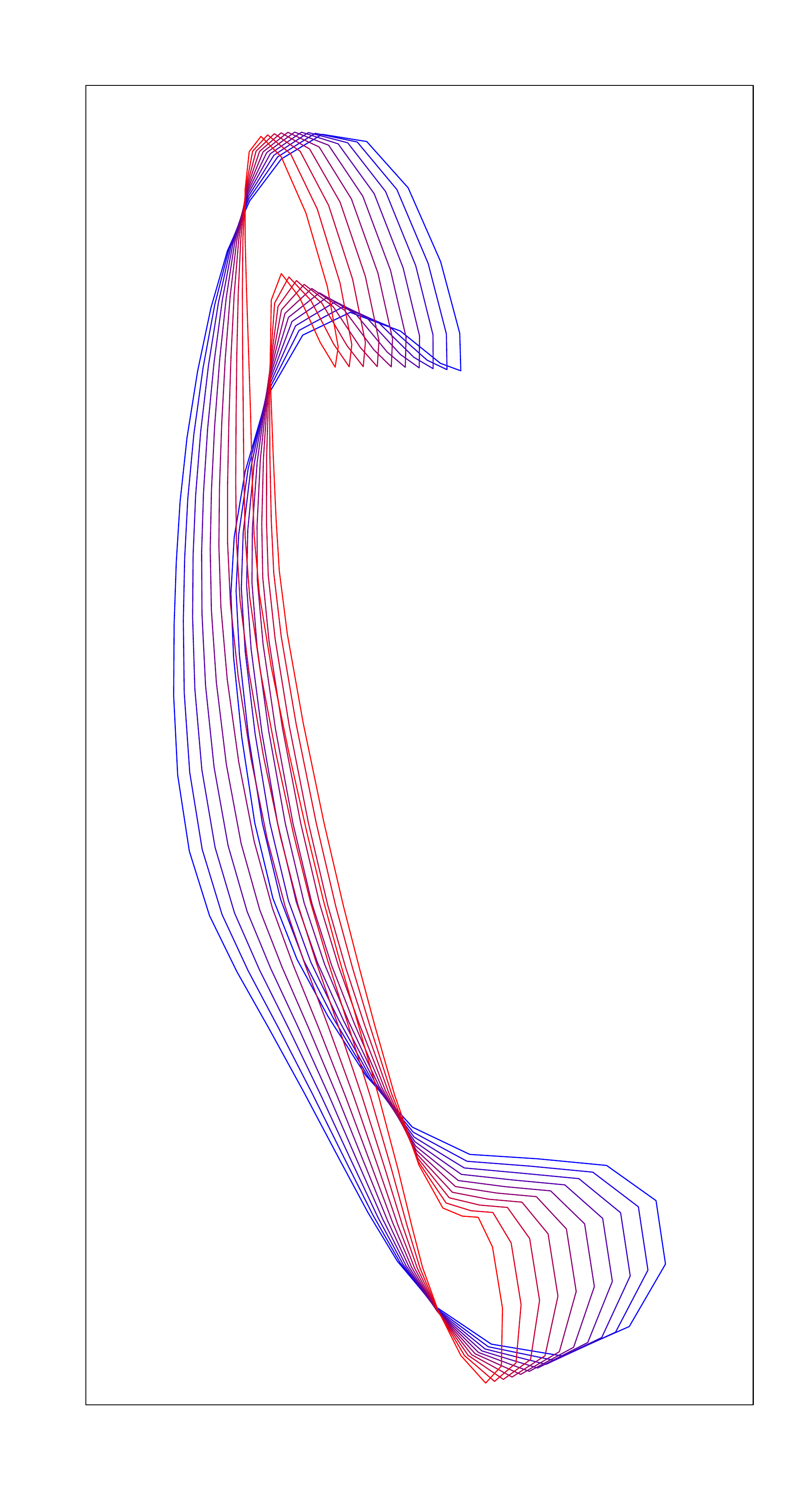}
		\caption{Tampered $L_2$}
		\label{fig:cc3}
	\end{subfigure}
	\begin{subfigure}[b]{0.22\linewidth}
		\includegraphics[width=\linewidth]{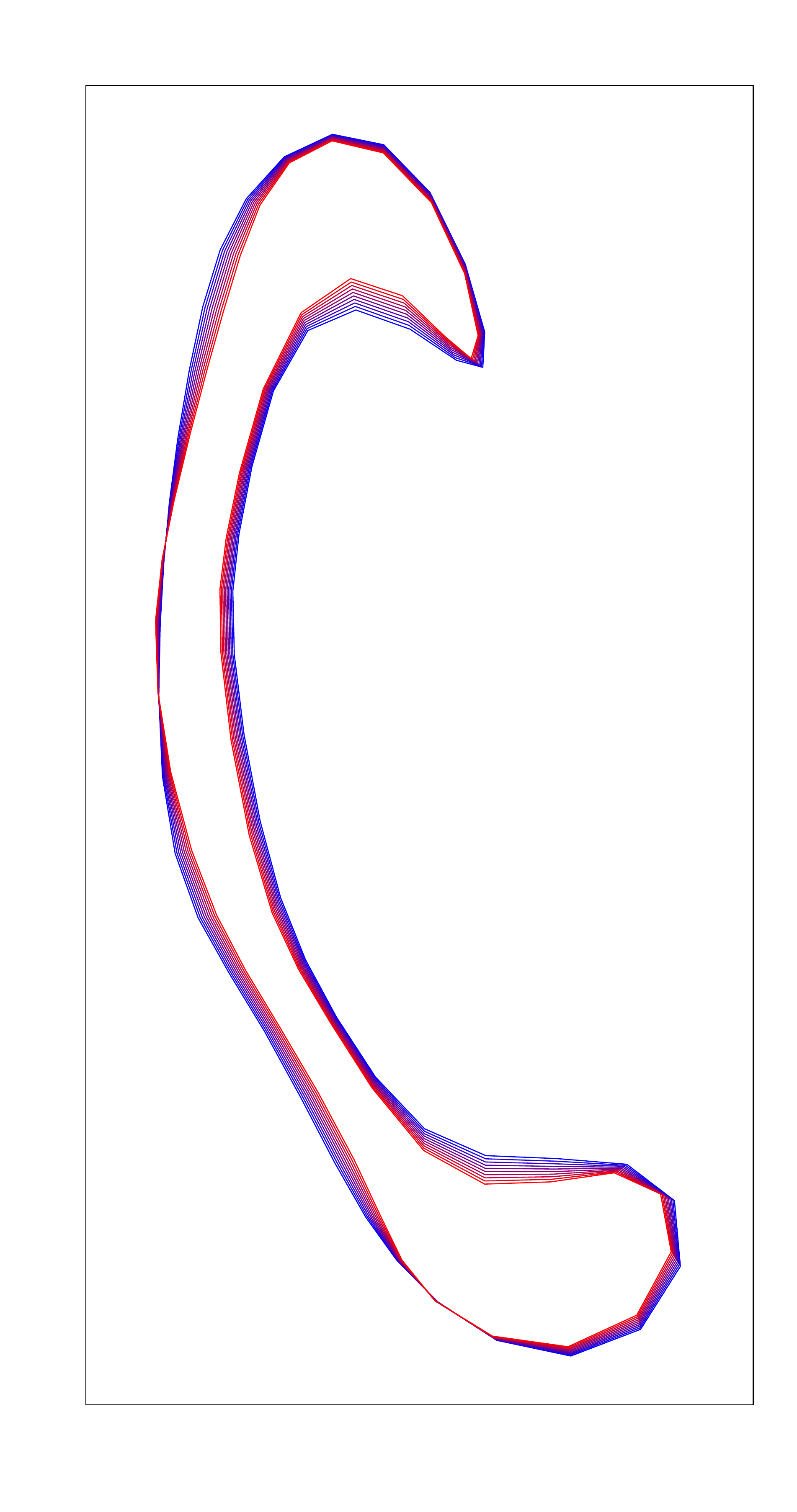}
		\caption{Tampered $L_1$}
		\label{fig:cc4}
	\end{subfigure}
	\begin{subfigure}[b]{0.22\linewidth}
		\includegraphics[width=\linewidth]{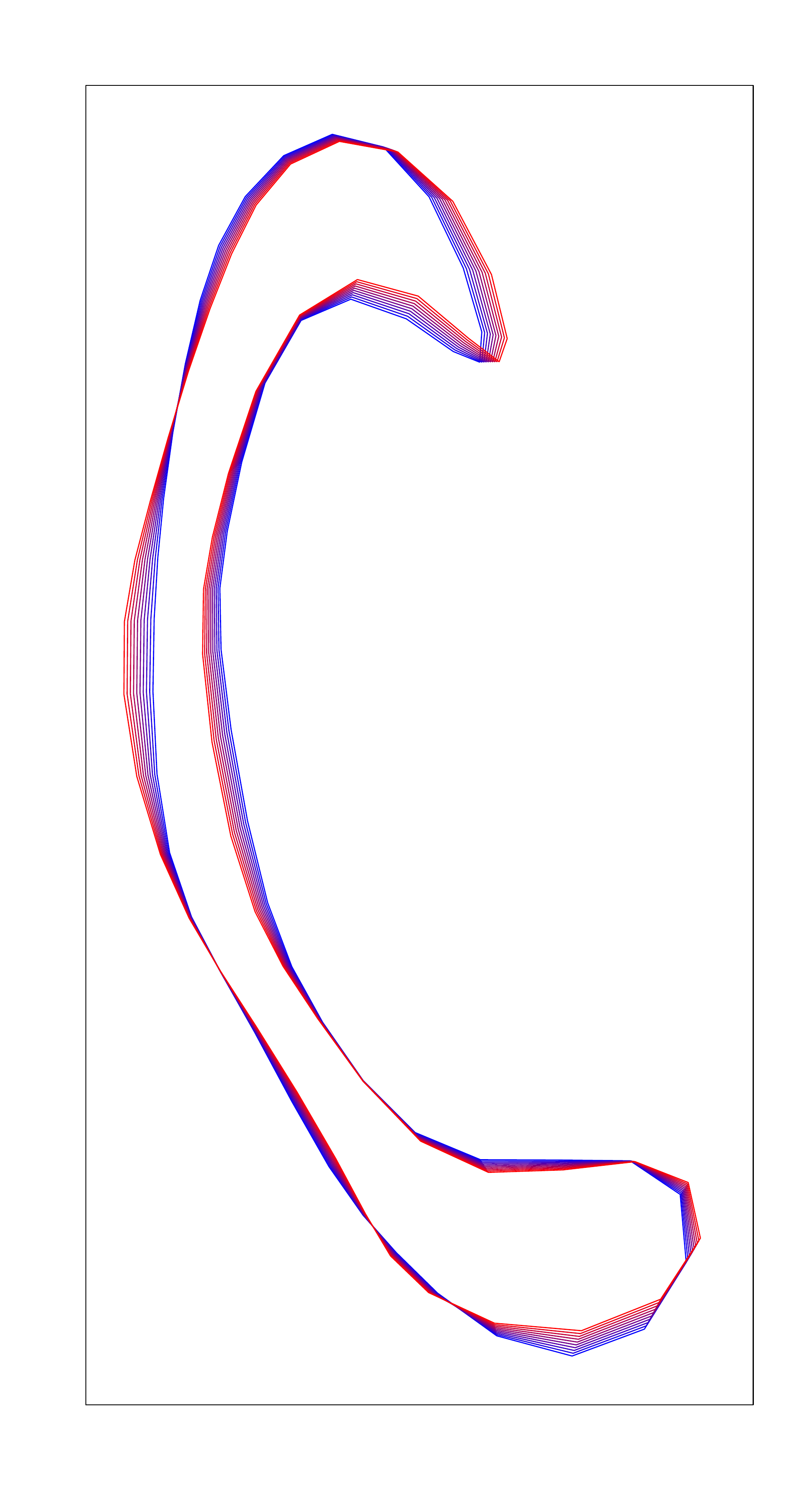}
		\caption{Tampered Tukey}
		\label{fig:cc5}
	\end{subfigure}
	\caption{The resulting geodesics displayed as a sequence of shapes. Each subfigure contains ten shapes, representing the estimated shape at every five years from age 50 (blue) to age 95 (red).}
	\label{fig:cc}
\end{figure}

\begin{table}[h]
	\centering
	\caption{Comparing the various regression parameter estimates against $(\hat{p}_{L_2}, \hat{v}_{L_2})$.}
	{\small 
		\begin{tabular}{|c|c||c|c|}\hline
			$d_{\Sigma_2^{50}}(\hat{p}_{L_1},\hat{p}_{L_2})$ & 0.0018924 & $\lVert\Gamma_{\hat{p}_{L_1}\rightarrow \hat{p}_{L_2}}(\hat{v}_{L_1})-\hat{v}_{L_2}\rVert$ & 0.0002177 \\ \hline
			$d_{\Sigma_2^{50}}(\hat{p}_T,\hat{p}_{L_2})$ & 0.0061325 & $\lVert\Gamma_{\hat{p}_T\rightarrow \hat{p}_{L_2}}(\hat{v}_T)-\hat{v}_{L_2}\rVert$ & 0.0011544 \\ \hline
			$d_{\Sigma_2^{50}}(\hat{p}^\prime_{L_2},\hat{p}_{L_2})$ & 0.1444551 & $\lVert\Gamma_{\hat{p}^\prime_{L_2}\rightarrow \hat{p}_{L_2}}(\hat{v}^\prime_{L_2})-\hat{v}_{L_2}\rVert$ & 0.0051700 \\ \hline
			$d_{\Sigma_2^{50}}(\hat{p}^\prime_{L_1},\hat{p}_{L_2})$ & 0.0182806 & $\lVert\Gamma_{\hat{p}^\prime_{L_1}\rightarrow \hat{p}_{L_2}}(\hat{v}^\prime_{L_1})-\hat{v}_{L_2}\rVert$ & 0.0009981 \\ \hline 
			$d_{\Sigma_2^{50}}(\hat{p}_T^\prime,\hat{p}_{L_2})$ & 0.0129771 & $\lVert\Gamma_{\hat{p}_T^\prime\rightarrow \hat{p}_{L_2}}(\hat{v}_T^\prime)-\hat{v}_{L_2}\rVert$ & 0.0008360 \\ \hline
		\end{tabular}
	}
	\label{cctable}
\end{table}

These results are displayed in Figure \ref{fig:cc} and Table \ref{cctable}. In Figure \ref{fig:cc}, each of the six geodesics are visualized as a sequence of ten shapes, $\mathrm{Exp}(\hat{p},(t-\bar{x})\hat{v})$, where $t=50,55,\ldots,90,95$, $\bar{x}$ is the mean age 74.75, and $(\hat{p}, \hat{v})$ is the regression estimate. Most of the figures look similar to Figure \ref{fig:cc0}, while Figure \ref{fig:cc3} is highly distorted. Table \ref{cctable} provides a more precise comparison through the actual parameter estimates. The first two rows show that the two robust estimators perform reasonably well on the untampered dataset, though the $L_1$ estimator performs significantly better. We observe in the last two rows that the reverse is true, to a much lesser extent, on the tampered dataset. The $L_2$ estimator, on the other hand, performs almost an order of magnitude worse than either robust estimator on the tampered data, as seen in the third row. All of these observations fall in line with our expectations about the three estimators on data with and without outliers in a very high-dimensional compact manifold; namely, that the $L_1$ and Tukey biweight estimators would be much more robust than the $L_2$ estimator, and that the $L_1$ estimator would fare better than the Tukey biweight estimator on data without outliers.

\section{Conclusion}

In this paper, we have proposed robust estimators for geodesic regression that are resistant to outliers. These methods adapted M-type estimators, including the $L_1$, Huber and Tukey biweight estimators, to a manifold setting. For the M-type estimators, we have developed a method, using tangent space approximations which are exact for dependent variables in $\mathbb{R}^n$, for calculating the tuning parameters that emphasizes efficiency in the case of normal errors while providing protection against outliers. We have also provided justification for a general preference for the $L_1$ estimator over the $L_2$ estimator and other estimators on high-dimensional spaces and derived the Riemannian normal distribution on the $n$-dimensional spheres and hyperbolic spaces, including a method for generating random points from this distribution. Finally, the proposed methods have been evaluated on synthetic and real data. 

This paper is only a first step into robust manifold statistics. We have mentioned the breakdown point on compact manifolds in Section \ref{symm}, but a study of the breakdown point on non-compact manifolds and the influence function, another tool for measuring the robustness of estimators, would be valuable. While our robust estimators use Euclidean approximations and have empirically been shown to be able to handle significant levels of curvature, further studies could also include the development of robust estimators derived with curvature, including negative curvature.

Beyond robust regression, a potentially fruitful avenue for future research is asymmetric loss functions on Riemannian manifolds. For example, quantile regression would require developing the notion of quantiles for manifold-valued data. One could also explore pseudo-quantiles, such as expectiles and M-quantiles, on manifolds.

\section*{Acknowledgement}
This research was supported by the National Research Foundation of Korea (NRF) funded by the Korea government (2020R1A4A1018207; 2021R1A2C1091357).

\begin{appendices}
\section{} \label{appena}
\subsection{Proofs of Propositions \ref{symmspaceprop} and \ref{constcurvspaceprop}} \label{symmspacepropsproof}
\subsubsection{Proof of Proposition \ref{symmspaceprop}}
\begin{proof}
We first note that the term in ($\ref{factor}$) is finite because $\rho$ is bounded below and so there exists some real $B>-\infty$ such that $\rho(t)>B$ for all $t\in\mathbb{R}$, which means that
\begin{equation}
D_M(\mu,b,\rho)\leq\int_M\mathrm{exp}\bigg(-\frac{B}{b}\bigg)dy=\mathrm{exp}\bigg(-\frac{B}{b}\bigg)\mbox{Vol}(M)<\infty, \notag
\end{equation}
where $\mbox{Vol}(M)$, the volume of $M$, is finite. So the function in (\ref{density}) is a well-defined density function.

The log-likelihood of the observations $\{(x_i,y_i)\}_{1,\ldots,N}$ under the distribution in ($\ref{density}$) is
\vspace{-3mm}
\begin{equation} \label{loglikelihood}
\sum_{i=1}^N\mathrm{log}\{D_M(\mathrm{Exp}(p,Vx_i),b,\rho)\}-\frac{1}{b}\sum_{i=1}^N\rho(d(\mathrm{Exp}(p,Vx_i),y_i)).
\end{equation}
Because $M$ is a symmetric space, it is also a homogeneous space, meaning that for any two points on the manifold, there exists an isometry which maps one to the other. Since the integral in  ($\ref{factor}$) depends only on the distance from $\mu$ to $y$, it is invariant to isometries, so the expression is independent of $\mu$. Therefore, the first sum in (\ref{loglikelihood}) is constant with respect to $p$ and $V$. Comparing the second sum to ($\ref{mestimator}$), we find that the parameters $(p,V)\in M\times (T_pM)^k$ that minimize  $L_\rho(p,V)$ also maximize the log-likelihood.
\end{proof}

\subsubsection{Proof of Proposition \ref{constcurvspaceprop}}
\begin{proof}
It is known that if $M$ is a complete and simply-connected Riemannian manifold of constant sectional curvature, it is isomorphic to either a sphere $S^n$, a Euclidean space $\mathbb{R}^n$, or a hyperbolic space $\mathbb{H}^n$, which are all symmetric spaces. The proposition is true on $S^n$ by Proposition \ref{symmspaceprop}. In $\mathbb{R}^n$, the $L_2$ estimator is equivalent to the isotropic $n$-variate distribution with variance $bI_n$. For the $L_1$ estimator,
\begin{equation}
D_{\mathbb{R}^n}(\mu,b,\rho)=\int_{\mathbb{R}^n}\mathrm{exp}\bigg(-\frac{\rho(d(\mu,y))}{b}\bigg)dy=\frac{2\pi^{\frac{n}{2}}}{\Gamma(\frac{n}{2})}\int_0^\infty r^{n-1}\mathrm{exp}\bigg(-\frac{\lvert r\rvert}{b}\bigg)dr \notag
\end{equation}
because the surface area of an $(n-1)$-sphere embedded in $\mathbb{R}^n$ is $((2\pi^{\frac{n}{2}}/\Gamma(\frac{n}{2})))r^{n-1}$. For any $b>0$, there exists some $R_b$ such that $r^{n-1}<\mathrm{exp}(r/2b)$ for all $r>R_b$, so 
\begin{align}
D_{\mathbb{R}^n}(\mu,b,\rho)&=\frac{2\pi^{\frac{n}{2}}}{\Gamma(\frac{n}{2})}\Bigg(\int_0^{R_b} r^{n-1}\mathrm{exp}\bigg(-\frac{r}{b}\bigg)dr+\int_{R_b}^\infty r^{n-1}\mathrm{exp}\bigg(-\frac{r}{b}\bigg)dr\Bigg)\nonumber\\
&\leq\frac{2\pi^{\frac{n}{2}}}{\Gamma(\frac{n}{2})}\Bigg(\int_0^{R_b}R_b^{n-1}\mathrm{exp}\bigg(-\frac{0}{b}\bigg)dr+\int_{R_b}^\infty \mathrm{exp}\bigg(\frac{r}{2b}\bigg)\mathrm{exp}\bigg(-\frac{r}{b}\bigg)dr\Bigg)\nonumber\\
&=\frac{2\pi^{\frac{n}{2}}}{\Gamma(\frac{n}{2})}\Bigg(\int_0^{R_b}R_b^{n-1}dr+\int_{R_b}^\infty \mathrm{exp}\bigg(-\frac{r}{2b}\bigg)dr\Bigg)\nonumber\\
&=\frac{2\pi^{\frac{n}{2}}}{\Gamma(\frac{n}{2})}\Bigg(R_b^{n-1}r\big|_0^{R_b}-2b\cdot \mathrm{exp}\bigg(-\frac{r}{2b}\bigg)\Big|_{R_b}^\infty\Bigg)\nonumber\\
&=\frac{2\pi^{\frac{n}{2}}}{\Gamma(\frac{n}{2})}\Bigg(R_b^n+2b\cdot \mathrm{exp}\bigg(-\frac{R_b}{2b}\bigg)\Bigg)\nonumber\\
&<\infty,\nonumber
\end{align}
so the density function is well-defined for all $b>0$.

As noted in Remark 3.1 in \cite{Cotton2002}, the surface area of an $(n-1)$-sphere of radius $x$ on $\mathbb{H}^n$ is
	\begin{equation} \label{hypsa}
	A_{\mathbb{H}^n}(x):=\frac{2\pi^{\frac{n}{2}}}{\Gamma(\frac{n}{2})}\mathrm{sinh}^{n-1}(x)
	\end{equation}
for $x\geq0$. We explicitly calculate the value of the corresponding normalizing constant corresponding to the $L_2$ estimator on $\mathbb{H}^n$ in (\ref{hypnormalconstant}) from Proposition \ref{hypnormalprop}(a) in Appendix \ref{hypnormal} with $\sigma^2$ replacing $b$. This expression is clearly finite for any $\sigma^2=b>0$. For the $L_1$ estimator, 
\begin{align} \label{l1hypconstant}
D_{\mathbb{H}^n}(\mu,b,\rho)&=\int_{\mathbb{H}^n} \mathrm{exp}\bigg(-\frac{\rho(d(y,\mu))}{b}\bigg)dy\nonumber\\
&=\int_0^\infty A_{\mathbb{H}}(r)\mathrm{exp}\bigg(-\frac{\lvert r\rvert}{b}\bigg)dr \nonumber\\
&=\frac{2\pi^{\frac{n}{2}}}{\Gamma(\frac{n}{2})}\int_0^\infty \mathrm{sinh}^{n-1}(r)\mathrm{exp}\bigg(-\frac{r}{b}\bigg)dr\nonumber\\
&=\frac{2\pi^{\frac{n}{2}}}{\Gamma(\frac{n}{2})}\int_0^\infty \bigg(\frac{(\mathrm{exp}(-r)-\mathrm{exp}(r))}{2}\bigg)^{n-1}\mathrm{exp}\bigg(-\frac{r}{b}\bigg)dr\nonumber\\
&=\frac{2\pi^{\frac{n}{2}}}{\Gamma(\frac{n}{2})}\int_0^\infty \frac{1}{2^{n-1}}\bigg(\sum_{j=0}^{n-1} {n-1 \choose j}(-1)^j\mathrm{exp}(-jr)\mathrm{exp}((n-1-j)r)\bigg)\mathrm{exp}\bigg(-\frac{r}{b}\bigg)dr\nonumber\\
&=\frac{2\pi^{\frac{n}{2}}}{\Gamma(\frac{n}{2})}\frac{1}{2^{n-1}}\bigg(\sum_{j=0}^{n-1} {n-1 \choose j}(-1)^j \int_0^\infty \mathrm{exp}\bigg((n-1-2j-\frac{1}{b})r\bigg)dr.
\end{align}
Because $\int_0^\infty \mathrm{exp}(cx)dx=(1/c)\mathrm{exp}(cx)|_0^\infty=-1/c$ is finite for any constant $c<0$, the expression in (\ref{l1hypconstant}) is finite if $n-1-2j-1/b<0$ for all $j=1,\ldots,n-1$. So (\ref{density}) is a well-defined density for any $b\in(0,1/(n-1))$.

When $\rho$ is Huber's loss, the finiteness of $D_M(\mu,b,\rho)$ for some $b>0$ on $M=\mathbb{R}^n$ or $\mathbb{H}^n$  easily follows from the above results and the definition of Huber's loss as a mixture of the $L_2$ and $L_1$ losses.
\end{proof}

\subsection{Derivations for Cutoff Parameters and Efficiency of the $L_1$ Estimator} \label{cutoffderivations}
This section expands upon Section \ref{calculating}, using the same notation and approximations. We make use of the beta function $B(x,y)$, the gamma function $\Gamma(a)$, the lower incomplete gamma function $\gamma(a,z)$, the upper incomplete gamma function $\Gamma(a,z)$, the lower and upper regularized gamma function $P(a,z)=\gamma(a,z)/\Gamma(z)$ and $Q(a,z)=\Gamma(a,z)/\Gamma(a)$, respectively, and the inverses of the two regularized gamma functions $P^{-1}(a,z)$ and $Q^{-1}(a,z)$. We also require partial derivatives of the upper and lower incomplete gamma functions: $\frac{\partial}{\partial a}\Gamma(a,z)=-a^{z-1}e^{-a}$ and $\frac{\partial}{\partial a}\gamma(a,z)=-\frac{\partial}{\partial a}\Gamma(a,z)=a^{z-1}e^{-a}$, respectively. We assume $k\geq2$. However, as mentioned in Section \ref{calculating}, the formulae for $\xi$ and the approximate AREs for the Tukey biweight and $L_1$ estimators, including their derivatives, turn out to still be valid in the $n=1$ case, and similarly for the Huber estimator if the second summands in (\ref{huberderivdiag}), (\ref{h1}), and (\ref{h3}) are set to zero. The main problem when $n=1$ in these summands is that the upper gamma function $\Gamma(a,z)$ is undefined when $a=0$.

\subsubsection{Identities} \label{identities}
Before proceeding, four identities related to integrals are derived. Recall that the density of a standard $k$-variate Gaussian random variable is defined as $\phi_n=(2\pi)^{-\frac{n}{2}}\mathrm{exp}(-\frac{1}{2}\sum_{j=1}^n(y^j)^2)$. Using the spherical coordinate system,  $r^2=\sum_{j=1}^n(y^j)^2$, $~y^1=r\mathrm{sin}(\theta_1)\cdots\mathrm{sin}(\theta_{n-2})\mathrm{sin}(\theta_{n-1})$ and $~y^j=r\mathrm{sin}(\theta_1)\cdots\mathrm{sin}(\theta_{n-j})\mathrm{cos}(\theta_{n-j+1})$ for $j=2,\ldots,n$, so that $dy=dy_1\cdots dy_n=r^{n-1}\mathrm{sin}^{n-2}(\theta_1)\cdots\mathrm{sin}(\theta_{n-2})d\theta_{n-1}\cdots d\theta_1$. Take a function $g:\mathbb{R}^+\rightarrow\mathbb{R}$. Letting $B_R\subset\mathbb{R}^n$ denote the $k$-ball centered at 0 of radius $R$, it follows that 
\vspace{-3mm}
\begin{align} \label{identity1}
& \int_{B_R}g(r)\phi_n(y)dy\nonumber\\ 
&= \int_0^R\int_0^\pi\cdots\int_0^\pi\int_0^{2\pi}g(r)\frac{1}{(2\pi)^{\frac{n}{2}}}{\mathrm{exp}\Big(-\frac{r^2}{2}\Big)}r^{n-1}\sin^{n-2}(\theta_1)\cdots\mathrm{sin}(\theta_{n-2})d\theta_{n-1}\cdots d\theta_1dr\nonumber\\
&=\frac{1}{(2\pi)^{\frac{n}{2}}}\Big(\int_0^Rg(r)r^{n-1}\mathrm{exp}\Big(-\frac{r^2}{2}\Big)dr\Big)\Big(\int_0^\pi\mathrm{sin}^{n-2}(\theta_1)d\theta_1\Big)\cdots \notag\\
&\qquad\cdots\Big(\int_0^\pi\mathrm{sin}^2(\theta_{n-3})d\theta_{n-3}\Big)\Big(\int_0^\pi\mathrm{sin}(\theta_{n-2})d\theta_{n-2}\Big)\Big(\int_0^{2\pi}d\theta_{n-1}\Big) \nonumber\\
&=\frac{1}{(2\pi)^{\frac{n}{2}}}\Big(\int_0^Rg(r)r^{n-1}\mathrm{exp}\Big(-\frac{r^2}{2}\Big)dr\Big)\Big(2\int_0^{\pi/2}\mathrm{sin}^{n-2}(\theta_1)d\theta_1\Big)\cdots \notag\\
&\qquad\cdots\Big(2\int_0^{\pi/2}\mathrm{sin}^2(\theta_{n-3})d\theta_{n-3}\Big)\Big(2\int_0^{\pi/2}\mathrm{sin}(\theta_{n-2})d\theta_{k-2}\Big)\Big(4\int_0^{\pi/2}d\theta_{n-1}\Big) \nonumber\\
&=\frac{1}{(2\pi)^{\frac{n}{2}}}\Big(\int_0^Rg(r)r^{n-1}\mathrm{exp}\Big(-\frac{r^2}{2}\Big)dr\Big)B\Big(\frac{n-1}{2},\frac{1}{2}\Big)\cdots B\Big(\frac{2}{2},\frac{1}{2}\Big)\cdot2B\Big(\frac{1}{2},\frac{1}{2}\Big) \nonumber\\
&=\frac{1}{(2\pi)^{\frac{n}{2}}}\Big(\int_0^Rg(r)r^{n-1}\mathrm{exp}\Big(-\frac{r^2}{2}\Big)dr\Big)\frac{\Gamma(\frac{n-1}{2})\Gamma(\frac{1}{2})}{\Gamma(\frac{k}{2})}\cdots\frac{\Gamma(\frac{2}{2})\Gamma(\frac{1}{2})}{\Gamma(\frac{3}{2})}\cdot2\frac{\Gamma(\frac{1}{2})\Gamma(\frac{1}{2})}{\Gamma(\frac{2}{2})} \nonumber\\
&=\frac{1}{(2\pi)^{\frac{n}{2}}}\Big(\int_0^Rg(r)r^{n-1}\mathrm{exp}\Big(-\frac{r^2}{2}\Big)dr\Big)\frac{2\pi^{\frac{n}{2}}}{\Gamma(\frac{n}{2})} \nonumber\\
&=2^{-\frac{n}{2}}\frac{2}{\Gamma(\frac{n}{2})}\Big(\int_0^Rg(r)r^{n-1}{\mathrm{exp}\Big(-\frac{r^2}{2}\Big)}dr\Big) \nonumber\\
&=2^{-\frac{n}{2}}\cdot\frac{n}{\Gamma(\frac{n+2}{2})}\Big(\int_0^Rg(r)r^{n-1}{\mathrm{exp}\Big(-\frac{r^2}{2}\Big)}dr\Big),\vspace{-3mm}
\end{align}
where $\Gamma(1/2)=\pi^{\frac{1}{2}},~\Gamma(1)=1$ and $\Gamma(z+1)=z\Gamma(z)$. The next two identities are derived in similar fashion:
\vspace{-3mm}
\begin{align} \label{identity2}
&\int_{B_R}g(r)(y^1)^2\phi_n(y)dy \notag\\
&=\int_0^R\int_0^\pi\cdots\int_0^\pi\int_0^{2\pi}g(r)(r\mathrm{sin}(\theta_1)\cdots\mathrm{sin}(\theta_{n-2})\mathrm{sin}(\theta_{n-1}))^2\frac{1}{(2\pi)^{\frac{n}{2}}}{\mathrm{exp}\Big(-\frac{r^2}{2}\Big)}r^{n-1}\notag\\
&\qquad\qquad\qquad\qquad\qquad \mathrm{sin}^{n-2}(\theta_1)\cdots\mathrm{sin}(\theta_{n-2})d\theta_{n-1}\cdots d\theta_1dr \notag\\
&=\int_0^R\int_0^\pi\cdots\int_0^\pi\int_0^{2\pi}g(r)\frac{1}{(2\pi)^{\frac{n}{2}}}{\mathrm{exp}\Big(-\frac{r^2}{2}\Big)}r^{n+1}\mathrm{sin}^{n}(\theta_1)\cdots\mathrm{sin}^2(\theta_{n-1})d\theta_{n-1}\cdots d\theta_1dr \notag\\
&=\frac{1}{(2\pi)^{\frac{n}{2}}}\Big(\int_0^Rg(r)r^{n+1}{\mathrm{exp}\Big(-\frac{r^2}{2}\Big)}dr\Big)\frac{\Gamma(\frac{n+1}{2})\Gamma(\frac{1}{2})}{\Gamma(\frac{n+2}{2})}\cdots\frac{\Gamma(\frac{4}{2})\Gamma(\frac{1}{2})}{\Gamma(\frac{5}{2})}\cdot2\frac{\Gamma(\frac{3}{2})\Gamma(\frac{1}{2})}{\Gamma(\frac{4}{2})} \notag\\
&=2^{-\frac{n}{2}}\cdot\frac{1}{\Gamma(\frac{n+2}{2})}\Big(\int_0^Rg(r)r^{n+1}{\mathrm{exp}\Big(-\frac{r^2}{2}\Big)}dr\Big)
\end{align}
and
\begin{align} \label{identity3}
&\int_{B_R}g(r)y^1y^2\phi_n(y)dy \notag\\
&=\int_0^R\int_0^\pi\cdots\int_0^\pi\int_0^{2\pi}g(r)(r\mathrm{sin}(\theta_1)\cdots\mathrm{sin}(\theta_{n-2})\mathrm{sin}(\theta_{n-1}))(r\mathrm{sin}(\theta_1)\cdots\mathrm{sin}(\theta_{n-2})\mathrm{cos}(\theta_{n-1}))\notag\\
&\qquad\qquad\qquad\qquad\qquad\frac{1}{(2\pi)^{\frac{n}{2}}}{\mathrm{exp}\Big(-\frac{r^2}{2}\Big)}r^{n-1}\mathrm{sin}^{n-2}(\theta_1)\cdots\mathrm{sin}(\theta_{n-2})d\theta_{n-1}\cdots d\theta_1dr \notag\\
&=\frac{1}{(2\pi)^{\frac{n}{2}}}\Big(\int_0^Rg(r)r^{n-1}{\mathrm{exp}\Big(-\frac{r^2}{2}\Big)}dr\Big)\Big(\int_0^\pi\mathrm{sin}^{n}(\theta_1)d\theta_1\Big)\cdots\notag\\
&\qquad\cdots\Big(\int_0^\pi\mathrm{sin}^3(\theta_{n-2})d\theta_{n-2}\Big)\Big(\int_0^{2\pi}\mathrm{sin}(\theta_{n-1})\mathrm{cos}(\theta_{n-1})d\theta_{n-1}\Big) \notag\\
&=0,
\end{align}
\vspace{-7mm}
because $\mathrm{sin}(\theta_{n-1})\mathrm{cos}(\theta_{n-1})=\mathrm{sin}(2\theta_{n-1})/2$, so the last factor is zero. The final identity uses the substitution $r^\prime=r^2/2$ and $dr=[(r^\prime)^{-\frac{1}{2}}/\sqrt2]dr^\prime$, 
\begin{align} \label{identity4}
\int_0^Rr^m{\mathrm{exp}\Big(-\frac{r^2}{2}\Big)}dr&=\int_0^{\frac{R^2}{2}}2^{\frac{m-1}{2}}(r^\prime)^{\frac{m-1}{2}}e^{-r^\prime}dr^\prime \notag\\
&=2^{\frac{m-1}{2}}\cdot\gamma\Big(\frac{m+1}{2},\frac{R^2}{2}\Big) \notag\\
&=2^{\frac{m-1}{2}}\cdot\Big[\Gamma\Big(\frac{m+1}{2}\Big)-\Gamma\Big(\frac{m+1}{2},\frac{R^2}{2}\Big)\Big].
\end{align}

\subsubsection{Detailed Steps} 
The first step uses $MAD=\mathrm{Median}(\lVert e_1\rVert,\ldots,\lVert e_N\rVert)$ to find a robust estimate of $\sigma$ in (\ref{normal}). In the manifold case, $e_i=\mathrm{Log}(\mathrm{Exp}(p,x_iv),y_i)$. For a random variable $Y^*$ distributed according to $f(y)=\phi_n(y)$, the goal is to find a factor $\xi$ such that $\mathrm{Pr}(\lVert Y^*\rVert<\xi)=1/2$. Letting  $g(r)=1$ in (\ref{identity1}) and $m=n-1$ in (\ref{identity4}), we have 
\begin{align*}
\mathrm{Pr}(\lVert Y^*\rVert<\xi)=\int_{B_\xi}\phi_n(y)dy &=2^{-\frac{n}{2}}\frac{n}{\Gamma(\frac{n+2}{2})}\Big(\int_0^\xi r^{n-1}{\mathrm{exp}\Big(-\frac{r^2}{2}\Big)}dr\Big) \notag\\
&=2^{-\frac{n}{2}}\frac{n}{\Gamma(\frac{n+2}{2})}\cdot 2^{\frac{n-2}{2}}\cdot\gamma\Big(\frac{n}{2},\frac{\xi^2}{2}\Big) \notag\\
&=2^{-1}\frac{2}{\Gamma(\frac{n}{2})}\cdot\gamma\Big(\frac{n}{2},\frac{\xi^2}{2}\Big) \notag\\
&=P\Big(\frac{n}{2},\frac{\xi^2}{2}\Big)=\frac{1}{2}. 
\end{align*}
The solution to this equation is given by (\ref{xi}). Finally, we obtain $\hat{\sigma}=MAD/\xi$.

The next step finds the multiple of $\sigma$ that gives an ARE to the sample mean of 95\%, assuming a normal distribution. It requires the four identities (\ref{identity1}), (\ref{identity2}), (\ref{identity3}) and (\ref{identity4}). We take a manifold-valued random variable $W\in M$ with intrinsic mean $\mu_W$. If $W^*:=\mathrm{Log}(\mu_W,W)$ has an isotropic Gaussian distribution in $\mathbb{R}^k$, i.e., its covariance $\Sigma_W=\sigma_W^2I_n$ is a multiple of the identity matrix, then
\vspace{-3mm}
\begin{equation}\label{chisquare}
\frac{1}{\sigma_W^2}\mathrm{E}(\lVert \mathrm{Log}(\mu_W,W)\rVert^2)=\mathrm{E}((W^*)^T\Sigma_W^{-1}W^*)=k \implies \mathrm{Var}(W)=n\sigma_W^2,
\end{equation}
as $(W^*)^T\Sigma_W^{-1}W^*\sim\chi^2(n)$. Recall that $Y_i$, $i=1,\ldots,N$, are distributed according to (\ref{normal}) and $Y_i^*:=\mathrm{Log}(\mu,\hat{Y})$. Let $\bar{Y}$ be the sample intrinsic mean of $Y_i$ and $\hat{Y}$ be an M-type estimator. Then we define $\bar{Y}^*=\mathrm{Log}(\mu,\bar{Y})$ and $\hat{Y}^*=\mathrm{Log}(\mu,\hat{Y})$. Assuming the latter two converge in distribution to $N(0,\sigma_1^2I_n)$ and $N(0,\sigma_2^2I_n)$, respectively, 
\vspace{-3mm}
\begin{equation} \label{are}
\mathrm{ARE}(\hat{Y},\bar{Y})\approx\frac{n\sigma_1^2}{n\sigma_2^2}=\frac{\sigma_1^2}{\sigma_2^2}
\end{equation}
by (\ref{chisquare}), so we just need to find $\sigma_1^2$ and $\sigma_2^2$.

The covariance matrix of an M-type estimator can be obtained using its related influence function. For a loss function $\rho:\mathbb{R}\rightarrow\mathbb{R}$, define $\lVert\rho\rVert:\mathbb{R}^n\rightarrow\mathbb{R}$ by $\Vert\rho\rVert(y)=\rho(\lVert y\rVert)$. Then for differentiable $\lVert\rho\rVert$, define $\psi:\mathbb{R}^n\rightarrow\mathbb{R}^n$ by $\psi(y)=\nabla\lVert\rho\rVert(y)$. Note that this coincides with the definition of $\psi$ as $\rho^\prime$ in the $n=1$ case for $\rho$ symmetric around 0. If $F$ is the distribution of $e$, and $T(F)$, the statistical functional at $F$ representing the M-type estimator, is the solution to $\mathrm{E}_F[\psi(y-T(F))]=0$, then the influence function at $y_0\in\mathbb{R}^k$ is defined as
\vspace{-3mm}
\[
IF(y_0;T,F)=\mathrm{E}\big(J_\psi(y-T(F)))^{-1}\psi(y_0-T(F)\big),
\]
where $J_\psi$ denotes the Jacobian matrix of $\psi$. Letting $\hat{F_N}$ represent the empirical distribution for $N$ independent samples from $F$, $T(\hat{F_N})$ is the sample M-estimator for these data points, and it is known by the central limit theorem that
\vspace{-3mm}
\[
\sqrt{N}\big(T(\hat{F_{N}})-T(F)\big)\Rightarrow N\Big(0,\int IF(y;T,F)IF(y;T,F)^TdF(y)\Big). 
\]
Taking our M-type estimator to be either Huber's or Tukey's estimator and $F$ to represent the multivariate normal distribution, $T(F)=\mu=0$ and the covariance of the sample M-type estimator $T(\hat{F_{N}})$ is asymptotically given by 
\begin{equation}\label{twoterms}
\Sigma_\psi=\frac{1}{N}\big(\mathrm{E}(J_\psi(y))^{-1}\big)^2\mathrm{E}\big[\psi(y)\psi(y)^T\big].
\end{equation}
The covariance of the sample mean $\bar{Y}^*=(1/N)\sum_{i=1}^NY_i^*$ is simply 
\vspace{-3mm}
\begin{equation} \label{meanvariance}
\frac{1}{N}\mathrm{Cov}(Y_1^*)=\frac{1}{N}I_k,
\end{equation}
so $\sigma_1^2=1/N$ in (\ref{are}).

\vskip 3mm
\noindent(a)~{\bf Huber estimator}:~
In the case of the Huber estimator, we have 
\begin{equation} \label{huberpsi}
\psi_H(y)=\begin{cases}
y &\mbox{if $\lVert y\rVert<c$} \\
c\cdot\frac{y}{\lVert y\rVert} &\mbox{otherwise},
\end{cases}
~~~\mbox{ and }~~~
J_{\psi_H}(y)=\begin{cases}
I_k &\mbox{if $\lVert y\rVert<c$} \\
c\big(\frac{1}{\lVert y\rVert}I_n-\frac{1}{\lVert y\rVert^3}yy^T\big) &\mbox{otherwise}.
\end{cases} 
\end{equation}

We first consider the first matrix term in (\ref{twoterms}). Using the identity of (\ref{identity3}), $\mathrm{E}(J_{\psi_H}(y))_{12}=-\int_{B_c^c}\frac{1}{\lVert y\rVert^3}(y^1)(y^2)\phi_kn(y)dy=0$. On the other hand, using the identities (\ref{identity1}), (\ref{identity2}), and (\ref{identity4}),
\begin{align} \label{huberderivdiag}
\mathrm{E}(J_{\psi_H}(y))_{11}&=\int_{B_c}\phi_n(y)dy+c\int_{B_c^c}\frac{1}{\lVert y\rVert}\phi_n(y)dy-c\int_{B_c^c}\frac{1}{\lVert y\rVert^3}(y^1)^2\phi_n(y)dy \notag&\\
&=2^{-\frac{n}{2}}\cdot\frac{n}{\Gamma\big(\frac{n+2}{2}\big)}\Big(\int_0^cr^{n-1}{\mathrm{exp}\Big(-\frac{r^2}{2}\Big)}dr\Big) \notag \\
&\qquad +c\cdot 2^{-\frac{n}{2}}\cdot\frac{n}{\Gamma\big(\frac{n+2}{2}\big)}\Big(\int_c^{\infty}\frac{1}{r}r^{n-1}{\mathrm{exp}\Big(-\frac{r^2}{2}\Big)}dr\Big) \notag \\
&\qquad-c\cdot 2^{-\frac{n}{2}}\cdot\frac{1}{\Gamma\big(\frac{n+2}{2}\big)}\Big(\int_c^{\infty}\frac{1}{r^3}r^{n-1}{\mathrm{exp}\Big(-\frac{r^2}{2}\Big)}dr\Big) \notag\\
&=\frac{1}{\Gamma\big(\frac{n+2}{2}\big)}\Bigg\{2^{-\frac{n}{2}}\cdot n\cdot 2^{\frac{n-2}{2}}\cdot\gamma\Big(\frac{n}{2},\frac{c^2}{2}\Big)+c\cdot 2^{-\frac{n}{2}}\cdot n\cdot 2^{\frac{n-3}{2}}\cdot\Big[\Gamma\Big(\frac{n-1}{2}\Big) \notag \\
&\qquad-\gamma\Big(\frac{n-1}{2},\frac{c^2}{2}\Big)\Big]-c\cdot 2^{-\frac{n}{2}}\cdot 2^{\frac{n-3}{2}}\cdot\Big[\Gamma\Big(\frac{n-1}{2}\Big)-\gamma\Big(\frac{n-1}{2},\frac{c^2}{2}\Big)\Big]\Bigg\} \notag\\
&=\frac{1}{\Gamma\big(\frac{n+2}{2}\big)}\Bigg\{\frac{n}{2}\gamma\Big(\frac{n}{2},\frac{c^2}{2}\Big)+2^{-\frac{3}{2}}c(n-1)\Gamma\Big(\frac{n-1}{2},\frac{c^2}{2}\Big)\Bigg\}.
\end{align}
By symmetry, $\mathrm{E}(J_{\psi_H}(y))_{jj}=\mathrm{E}(J_{\psi_H}(y))_{11}$ for $j=1,\ldots,n$, and $\mathrm{E}(J_{\psi_H}(y))_{lj}=\mathrm{E}(J_{\psi_H}(y))_{12}$ for all $j,~l=1,\ldots,n$, $l\neq j$, so the covariance of the sample mean is a scalar multiple of the identity matrix; namely, $\mathrm{E}(J_{\psi_H}(y))$ is $I_n$ multiplied by the result of (\ref{huberderivdiag}).

We now consider the second matrix term in (\ref{twoterms}). The non-diagonal terms can again be shown to be zero using identity (\ref{identity3}) and symmetry, and the diagonal terms can be shown to be equal by symmetry. Then with $\psi_H=(\psi_H^1,\ldots,\psi_H^n)$ in (\ref{huberpsi}), it follows that 
\begin{align} \label{huberdiag}
\mathrm{E}[\psi_H(y)\psi_H(y)^T]_{11}&=\mathrm{E}[(\psi_H^1(y))^2] \notag \\
&=\int_{B_c}(y^1)^2\phi_n(y)dy+c^2\int_{B_c^c}\frac{1}{\lVert y\rVert^2}(y^1)^2\phi_n(y)dy \notag \\
&=\frac{2^{-\frac{n}{2}}}{\Gamma\big(\frac{n+2}{2}\big)}\Big(\int_0^cr^{n+1}{\mathrm{exp}\Big(-\frac{r^2}{2}\Big)}dr\Big)+c^2\cdot\frac{2^{-\frac{n}{2}}}{\Gamma\big(\frac{n+2}{2}\big)}\Big(\int_c^{\infty}r^{kn-1}{\mathrm{exp}\Big(-\frac{r^2}{2}\Big)}dr\Big) \notag\\
&=\frac{1}{\Gamma\big(\frac{n+2}{2}\big)}\Bigg\{\gamma\Big(\frac{n+2}{2},\frac{c^2}{2}\Big)+\frac{c^2}{2} \Gamma\Big(\frac{n}{2},\frac{c^2}{2}\Big)\Bigg\},
\end{align}
using (\ref{identity2}) and (\ref{identity4}). Thus, the matrix $\mathrm{E}[\psi_H(y)\psi_H(y)^T]$ is the above expression multiplied by $I_n$, and the variance $\Sigma_\psi$ in (\ref{twoterms}) can be calculated using (\ref{huberderivdiag}) and (\ref{huberdiag}),
\begin{equation} \label{hubervar}
\Sigma_{\psi_H}=\frac{\mathrm{E}[\psi_H(y)\psi_H(y)^T]_{11}}{N(\mathrm{E}(J_{\psi_H}(y))_{11})^2}\cdot I_n,
\end{equation}
giving $\sigma_2^2$ in (\ref{are}). Hence, from (\ref{are}), (\ref{meanvariance}), (\ref{huberderivdiag}), (\ref{huberdiag}), and (\ref{hubervar}), the approximate ARE to the sample mean is given by (\ref{huberare})
\begin{equation} \label{h}
\mbox{ARE}_{H,L_2}(c,n)\approx A_H(c,n):=\frac{H_1^2}{\Gamma\big(\frac{n+2}{2}\big)H_2}, 
\end{equation}
where 
\begin{eqnarray}
H_1&=&\Gamma\Big(\frac{n+2}{2}\Big)\mathrm{E}(J_{\psi_H}(y))_{11}=\frac{n}{2}\gamma\Big(\frac{n}{2},\frac{c^2}{2}\Big)+2^{-\frac{3}{2}}c(n-1)\Gamma\Big(\frac{n-1}{2},\frac{c^2}{2}\Big), \label{h1}\\
H_2&=&\Gamma\Big(\frac{n+2}{2}\Big)\mathrm{E}[\psi_H(y)\psi_H(y)^T]_{11}=\gamma\Big(\frac{n+2}{2},\frac{c^2}{2}\Big)+\frac{c^2}{2} \Gamma\Big(\frac{n}{2},\frac{c^2}{2}\Big). \label{h2} 
\end{eqnarray}

Lastly, we apply the Newton-Raphson method to find the value of $c$ for which the ARE is approximately 95\%, that is, the solution in $c$ to the equation $A_H(c,n)-0.95=0$. This requires the partial derivative of $A_H(c,n)$ with respect to $c$, 
\begin{equation}
\frac{\partial}{\partial c}A_H(c,n)=\frac{2H_1H_3H_2-H_1^2H_4}{\Gamma\big(\frac{n+2}{2}\big)H_2^2}, \notag 
\end{equation}
where $H_1$ and $H_2$ are as above and
\begin{eqnarray}
H_3&=&\frac{\partial}{\partial c}H_1\notag \\
&=&\frac{cn}{2}\Big(\frac{c^2}{2}\Big)^{\frac{n-2}{2}}{\mathrm{exp}\Big(-\frac{c^2}{2}\Big)}+2^{-\frac{3}{2}}(n-1)\Gamma\Big(\frac{n-1}{2},\frac{c^2}{2}\Big)-2^{-\frac{3}{2}}c^2(n-1)\Big(\frac{c^2}{2}\Big)^{\frac{n-3}{2}}{\mathrm{exp}\Big(-\frac{c^2}{2}\Big)} \notag \\
&=&2^{-\frac{n}{2}}c^{k-1}{\mathrm{exp}\Big(-\frac{c^2}{2}\Big)}+2^{-\frac{3}{2}}(n-1)\Gamma\Big(\frac{n-1}{2},\frac{c^2}{2}\Big), \label{h3}\\
H_4&=& \frac{\partial}{\partial c}H_2 \notag \\
&=&c\Big(\frac{c^2}{2}\Big)^{\frac{n}{2}}{\mathrm{exp}\Big(-\frac{c^2}{2}\Big)}+c\Gamma\Big(\frac{n}{2},\frac{c^2}{2}\Big)-c\Big(\frac{c^2}{2}\Big)\Big(\frac{c^2}{2}\Big)^{\frac{n-2}{2}}{\mathrm{exp}\Big(-\frac{c^2}{2}\Big)} \notag \\
&=&c\Gamma\Big(\frac{n}{2},\frac{c^2}{2}\Big). \label{h4}
\end{eqnarray}

\vskip 3mm
\noindent(b)~{\bf Tukey biweight estimator}:~
For this estimator, it is easy to show that
\begin{equation}
\psi_B(y)=\begin{cases}
\Big[1-\big(\frac{\lVert y\rVert}{c}\big)^2\Big]^2\cdot y &~~\mbox{if $\lVert y\rVert<c$}\\
0 &~~ \mbox{otherwise},
\end{cases} \notag
\end{equation}
and
\begin{equation}
J_{\psi_B}(y)=\begin{cases}
\Big[1-\big(\frac{\lVert y\rVert^2}{c^2}\big)^2\Big]^2I_n-\frac{4}{c^2}\Big[1-\big(\frac{\lVert y\rVert^2}{c^2}\big)^2\Big]yy^T &~~\mbox{if $\lVert y\rVert<c$}\\
0 &~~ \mbox{otherwise}.
\end{cases} \notag
\end{equation}

By similar arguments to the ones used for the Huber estimator, we have $\mathrm{E}(J_{\psi_B}(y))_{12}=0$, $\mathrm{E}[\psi_H(y)\psi_H(y)^T]_{12}=0$, 
\begin{equation}
\mathrm{E}(J_{\psi_H}(y))_{11}=\frac{1}{\Gamma\big(\frac{n+2}{2}\big)}\Bigg\{\frac{2(n+4)}{c^4}\gamma\Big(\frac{n+4}{2},\frac{c^2}{2}\Big)-\frac{2(n+2)}{c^2}\gamma\Big(\frac{n+2}{2},\frac{c^2}{2}\Big)+\frac{n}{2}\gamma\Big(\frac{n}{2},\frac{c^2}{2}\Big)\Bigg\}, \label{tukeyderivdiag} 
\end{equation}
\begin{align}
\mathrm{E}[\psi_H(y)\psi_H(y)^T]_{11}&=\frac{1}{\Gamma\big(\frac{n+2}{2}\big)}\Bigg\{\gamma\Big(\frac{n+2}{2},\frac{c^2}{2}\Big)-\frac{8}{c^2}\gamma\Big(\frac{n+4}{2},\frac{c^2}{2}\Big)+\frac{24}{c^4}\gamma\Big(\frac{n+6}{2},\frac{c^2}{2}\Big)\notag\\
&\qquad\qquad\qquad-\frac{32}{c^6}\gamma\Big(\frac{n+8}{2},\frac{c^2}{2}\Big)+\frac{16}{c^8}\gamma\Big(\frac{n+10}{2},\frac{c^2}{2}\Big)\Bigg\}. \label{tukeydiag}
\end{align}
Thus, the variance $\Sigma_\psi$ in (\ref{twoterms}) can be calculated using (\ref{tukeyderivdiag}) and (\ref{tukeydiag}), 
\vspace{-3mm}
\begin{equation} \label{tukeyvar}
\Sigma_{\psi_B}=\frac{\mathrm{E}[\psi_B(y)\psi_B(y)^T]_{11}}{N(\mathrm{E}(J_{\psi_B}(y))_{11})^2}\cdot I_n.
\end{equation}
giving $\sigma_2^2$ in (\ref{are}). Therefore, from (\ref{are}), (\ref{meanvariance}), (\ref{tukeyderivdiag}), (\ref{tukeydiag}), and (\ref{tukeyvar}), the approximate ARE to the sample mean is given by (\ref{tukeyare}), 
\begin{equation}
\mbox{ARE}_{T,L_2}(c,n)\approx A_T(c,n):=\frac{T_1^2}{\Gamma\big(\frac{n+2}{2}\big)T_2}, \notag
\end{equation}
where 
\begin{eqnarray*}
T_1&=&\Gamma\Big(\frac{n+2}{2}\Big)\mathrm{E}(J_{\psi_T}(y))_{11}=\frac{2(n+4)}{c^4}\gamma\Big(\frac{n+4}{2},\frac{c^2}{2}\Big)-\frac{2(n+2)}{c^2}\gamma\Big(\frac{n+2}{2},\frac{c^2}{2}\Big) \\
&&\qquad\qquad\qquad\qquad\qquad\qquad+\frac{n}{2}\gamma\Big(\frac{n}{2},\frac{c^2}{2}\Big), \\
T_2&=&\Gamma\Big(\frac{n+2}{2}\Big)\mathrm{E}[\psi_T(y)\psi_H(y)^T]_{11}=\gamma\Big(\frac{n+2}{2},\frac{c^2}{2}\Big)-\frac{8}{c^2}\gamma\Big(\frac{n+4}{2},\frac{c^2}{2}\Big)+\frac{24}{c^4}\gamma\Big(\frac{n+6}{2},\frac{c^2}{2}\Big) \\
&&\qquad\qquad\qquad\qquad\qquad\qquad\qquad\qquad-\frac{32}{c^6}\gamma\Big(\frac{n+8}{2},\frac{c^2}{2}\Big)+\frac{16}{c^8}\gamma\Big(\frac{n+10}{2},\frac{c^2}{2}\Big).
\end{eqnarray*}

We solve for the root of the function $A_T(c,n)-0.95$ by utilizing $\frac{\partial}{\partial c}A_T(c,n)$ in the Newton-Raphson method, 
\begin{equation}
\frac{\partial}{\partial c}A_T(c,n)=\frac{2T_1T_3T_2-T_1^2T_4}{\Gamma\big(\frac{n+2}{2}\big)T_2^2}, \notag 
\end{equation}
where $T_1$ and $T_2$ are as above and
\begin{eqnarray*}
T_3&=&\frac{\partial}{\partial c}T_1 \\
&=&-\frac{8(n+4)}{c^5}\gamma\Big(\frac{n+4}{2},\frac{c^2}{2}\Big)+\frac{2(n+2)}{c^3}\Big(\frac{c^2}{2}\Big)^{\frac{n+2}{2}}{\mathrm{exp}\Big(-\frac{c^2}{2}\Big)}\\
&&+\frac{4(n+2)}{c^3}\gamma\Big(\frac{n+2}{2},\frac{c^2}{2}\Big)-\frac{2(n+2)}{c}\Big(\frac{c^2}{2}\Big)^{\frac{n}{2}}{\mathrm{exp}\Big(-\frac{c^2}{2}\Big)}+\frac{cn}{2}\Big(\frac{c^2}{2}\Big)^{\frac{n-2}{2}}{\mathrm{exp}\Big(-\frac{c^2}{2}\Big)} \\
&=&-\frac{8(n+4)}{c^5}\gamma\Big(\frac{n+4}{2},\frac{c^2}{2}\Big)+\frac{4(n+2)}{c^3}\gamma\Big(\frac{n+2}{2},\frac{c^2}{2}\Big)-2^{-\frac{n-2}{2}}c^{n-1}{\mathrm{exp}\Big(-\frac{c^2}{2}\Big)}, \\
T_4&=&\frac{\partial}{\partial c}T_2 \\
&=&c\Big(\frac{c^2}{2}\Big)^{\frac{n}{2}}{\mathrm{exp}\Big(-\frac{c^2}{2}\Big)}+\frac{16}{c^3}\gamma\Big(\frac{n+4}{2},\frac{c^2}{2}\Big)-\frac{8}{c}\Big(\frac{c^2}{2}\Big)^{\frac{n+2}{2}}{\mathrm{exp}\Big(-\frac{c^2}{2}\Big)} \\
&&\qquad-\frac{96}{c^5}\gamma\Big(\frac{n+6}{2},\frac{c^2}{2}\Big)d+\frac{24}{c^3}\Big(\frac{c^2}{2}\Big)^{\frac{n+4}{2}}{\mathrm{exp}\Big(-\frac{c^2}{2}\Big)}+\frac{192}{c^7}\gamma\Big(\frac{n+8}{2},\frac{c^2}{2}\Big) \\
&&\qquad-\frac{32}{c^5}\Big(\frac{c^2}{2}\Big)^{\frac{n+6}{2}}{\mathrm{exp}\Big(-\frac{c^2}{2}\Big)}-\frac{128}{c^9}\gamma\Big(\frac{n+10}{2},\frac{c^2}{2}\Big)+\frac{16}{c^7}\Big(\frac{c^2}{2}\Big)^{\frac{n+8}{2}}{\mathrm{exp}\Big(-\frac{c^2}{2}\Big)} \\
&=&\frac{16}{c^3}\gamma\Big(\frac{n+4}{2},\frac{c^2}{2}\Big)-\frac{96}{c^5}\gamma\Big(\frac{n+6}{2},\frac{c^2}{2}\Big)+\frac{192}{c^7}\gamma\Big(\frac{n+8}{2},\frac{c^2}{2}\Big)-\frac{128}{c^9}\gamma\Big(\frac{n+10}{2},\frac{c^2}{2}\Big).
\end{eqnarray*}

\subsubsection{Proof of Proposition \ref{l1prop}} \label{l1propproof}

\begin{proof}[Proof of Proposition \ref{l1prop}(a)] Using (\ref{h}), (\ref{h1}), (\ref{h2}), (\ref{h3}), and (\ref{h4}), and two applications of L'H\^{o}pital's rule, we obtain 
\begin{align*}
\lim_{c\rightarrow 0}A_H(c,n)&=\lim_{c\rightarrow 0}\frac{H_1^2}{\Gamma\big(\frac{n+2}{2}\big)H_2} =\lim_{c\rightarrow 0}\frac{2H_1H_3}{\Gamma\big(\frac{n+2}{2}\big)H_4} =\lim_{c\rightarrow 0}\frac{2H_3^2+2H_1\frac{\partial}{\partial c}H_3}{\Gamma\big(\frac{n+2}{2}\big)\frac{\partial}{\partial c}H_4} \\
&=\lim_{c\rightarrow 0}\frac{2\Big\{2^{-\frac{n}{2}}c^{n-1}{\mathrm{exp}\Big(-\frac{c^2}{2}\Big)}+2^{-\frac{3}{2}}(n-1)\Gamma\big(\frac{n-1}{2},\frac{c^2}{2}\big)\Big\}^2-2H_12^{-\frac{n}{2}}c^k{\mathrm{exp}\Big(-\frac{c^2}{2}\Big)}}{\Gamma\big(\frac{n+2}{2}\big)\Big\{\Gamma\big(\frac{n}{2},\frac{c^2}{2}\big)-2^{-\frac{n-2}{2}}c^{n-1}{\mathrm{exp}\Big(-\frac{c^2}{2}\Big)}\Big\}} \\
&=\frac{2\Big\{2^{-\frac{3}{2}}(n-1)\Gamma\big(\frac{n-1}{2}\big)\Big\}^2}{\Gamma\big(\frac{n+2}{2}\big)\Gamma\big(\frac{n}{2}\big)}=\frac{\big(\frac{n-1}{2}\big)^2\Gamma^2\big(\frac{n-1}{2}\big)}{\Gamma\big(\frac{n}{2}\big)\Gamma\big(\frac{n+2}{2}\big)}=\frac{\Gamma^2\big(\frac{n+1}{2}\big)}{\Gamma\big(\frac{n}{2}\big)\Gamma\big(\frac{n+2}{2}\big)}. \qedhere
\end{align*}
\end{proof}

\begin{lemma} \label{lemma} It follows that 
\[
\frac{\Gamma^2\big(\frac{n+1}{2}\big)}{\Gamma\big(\frac{n}{2}\big)\Gamma\big(\frac{n+2}{2}\big)}<\sqrt{\frac{n}{n+1}}. 
\]
\end{lemma}

\begin{proof}
Theorem 3 in \cite{Mortici2012} states that, for $x\geq1$,
\begin{equation} \label{mortici}
\frac{1}{\sqrt{x\bigg(1+\frac{1}{4x-\frac{1}{2}+\frac{3}{16x+\frac{15}{4x}}}\bigg)}}<\frac{\Gamma\big(x+\frac{1}{2}\big)}{\Gamma\big(x+1\big)}<\frac{1}{\sqrt{x\bigg(1+\frac{1}{4x-\frac{1}{2}+\frac{3}{16x}}\bigg)}}.
\end{equation}
Because $x\geq1$, it follows that $4x-\frac{1}{2}+\frac{3}{16x}\leq4x-\frac{1}{2}+\frac{3}{16}=4x-\frac{5}{16}<4x$, so we have 
\begin{equation} \label{quarter}
\frac{1}{\sqrt{x\bigg(1+\frac{1}{4x-\frac{1}{2}+\frac{3}{16x}}\bigg)}}<\frac{1}{\sqrt{x\Big(1+\frac{1}{4x}\Big)}}=\frac{1}{\sqrt{x+\frac{1}{4}}}.
\end{equation}
Therefore, using (\ref{l1are}) and letting $x=\frac{n}{2}\geq1$ in (\ref{mortici}) and (\ref{quarter}), we obtain 
\[
\frac{\Gamma^2\big(\frac{n+1}{2}\big)}{\Gamma\big(\frac{n}{2}\big)\Gamma\big(\frac{n+2}{2}\big)}=\frac{\Gamma^2\big(\frac{n+1}{2}\big)}{\frac{2}{n}\Gamma\big(\frac{n+2}{2}\big)\Gamma\big(\frac{n+2}{2}\big)}=\frac{n}{2}\Bigg(\frac{\Gamma\big(\frac{n+1}{2}\big)}{\Gamma\big(\frac{n+2}{2}\big)}\Bigg)^2<\frac{\frac{n}{2}}{\frac{n}{2}+\frac{1}{4}}=\frac{2n}{2n+1}.
\]
Now, for $n\geq2$, it follows that $4n^3+4n^2<4n^3+4n^2+n$. Thus, we have 
\[
\frac{2n}{2n+1}<\sqrt{\frac{n}{n+1}}, 
\]
which completes the proof. 
\end{proof}

\begin{proof}[Proof of Proposition \ref{l1prop}(b)] 
By Lemma \ref{lemma} and (\ref{l1are}), we have 
\begin{align*}
A_{L_1}(n+1)&=\frac{\Gamma^2\big(\frac{n+2}{2}\big)}{\Gamma\big(\frac{n+1}{2}\big)\Gamma\big(\frac{n+3}{2}\big)} =\frac{\frac{n}{2}\Gamma\big(\frac{n}{2}\big)\Gamma\big(\frac{n+2}{2}\big)}{\Gamma\big(\frac{n+1}{2}\big)\cdot\frac{n+1}{2}\Gamma\big(\frac{n+1}{2}\big)}=\frac{n}{n+1}\frac{1}{A_{L_1}(n)} \\
&>\frac{n}{n+1}\sqrt{\frac{n+1}{n}}=\sqrt{\frac{n}{n+1}} \\
&>A_{L_1}(n)
\end{align*}
for $n\geq 2$.
\end{proof}

\begin{proof}[Proof of Proposition \ref{l1prop}(c)]
We again use (\ref{mortici}). Because $x\geq1>0$, it follows that $4x-\frac{1}{2}+\frac{3}{16x+\frac{15}{4x}}>4x-\frac{1}{2}>3x$, and so we have 
\begin{equation} \label{third}
\frac{1}{\sqrt{x\bigg(1+\frac{1}{4x-\frac{1}{2}+\frac{3}{16x+\frac{15}{4x}}}\bigg)}}>\frac{1}{\sqrt{x\Big(1+\frac{1}{3x}\Big)}}=\frac{1}{\sqrt{x+\frac{1}{3}}}.
\end{equation}
Combining (\ref{mortici}), (\ref{quarter}), and (\ref{third}), we obtain 
\begin{equation} \label{loose}
\frac{1}{\sqrt{x+\frac{1}{3}}}<\frac{\Gamma\big(x+\frac{1}{2}\big)}{\Gamma\big(x+1\big)}<\frac{1}{\sqrt{x+\frac{1}{4}}}
\end{equation}
for $x\geq1$. Taking the reciprocal of (\ref{loose}) and replacing $x$ with $x-\frac{1}{2}$ gives
\begin{equation*}
\sqrt{(x-\frac{1}{2})+\frac{1}{4}}<\frac{\Gamma\big((x-\frac{1}{2})+1\big)}{\Gamma\big((x-\frac{1}{2})+\frac{1}{2}\big)}<\sqrt{(x-\frac{1}{2})+\frac{1}{3}}
\end{equation*}
or
\begin{equation}\label{recip}
\sqrt{x+\frac{1}{2}}<\frac{\Gamma\big(x+\frac{1}{2}\big)}{\Gamma\big(x\big)}<\sqrt{x-\frac{1}{6}}
\end{equation}
for $x-\frac{1}{2}\geq1$, or $x\geq\frac{3}{2}$. Then multiplying (\ref{loose}) and (\ref{recip}) gives
\begin{equation} \label{limit}
\sqrt{\frac{x+\frac{1}{2}}{x+\frac{1}{3}}}<\frac{\Gamma^2\big(x+\frac{1}{2}\big)}{\Gamma\big(x\big)\Gamma\big(x+1\big)}<\sqrt{\frac{x-\frac{1}{6}}{x+\frac{1}{4}}}
\end{equation}
for $x\geq\frac{3}{2}$. The limits as $x\rightarrow \infty$ of the left- and right-hand expressions in (\ref{limit}) are both 1, and letting $x=\frac{n}{2}$, the central expression is (\ref{l1are}), completing the proof.
\end{proof}

\subsection{Riemannian Normal Distribution on $S^n$ and $\mathbb{H}^n$} \label{spherehypnormal}
\noindent
Here we derive the normalizing constant for the Riemannian normal distribution on $S^n$ and $\mathbb{H}^n$, which leads to a full description of the Riemannian normal density on those manifolds. We further describe how to randomly generate points from this distribution.

\subsubsection{Riemannian Normal Distribution on $S^n$}
\noindent
For $0\leq R\leq \pi$ and integer $m\geq 0$, define 
\begin{align} \label{gr}
& G_{m, \sigma^2}(R):=\bigg(\frac{i}{2}\bigg)^m\sqrt{\frac{\pi\sigma^2}{2}}\sum_{j=0}^{m} {m \choose j}(-1)^j \mathrm{exp}\Big(-\frac{(m-2j)^2\sigma^2}{2}\Big)\mathrm{erf}\Big(\frac{R}{\sqrt{2\sigma^2}}\notag\\
&\qquad\qquad\qquad+\sqrt{\frac{\sigma^2}{2}}\big(m-2j\big)i\Big),
\end{align}
where $i=\sqrt{-1}$ and $\mathrm{erf}(z):=(2/\sqrt{\pi})\int_0^z e^{-t^2} dt$ is the error function for complex $z$.
\begin{prop} \label{spherenormalprop}
(a) When $M=S^n$, the normalizing constant in (\ref{normal}) is given by
\begin{equation}
C_{S^n}(\mu,\sigma^2)=\frac{2\pi^{\frac{n}{2}}}{\Gamma(\frac{n}{2})} \big(G_{n-1,\sigma^2}(\pi)-G_{n-1,\sigma^2}(R)\big). \nonumber
\end{equation}
Substituting this normalizing constant into (\ref{normal}) gives the Riemannian normal density on $S^n$.
(b) The distribution function of $d(y,\mu)$ is
\begin{equation} \label{spheref}
	F_{S^n}(R):=\mathrm{Pr}(d(y,\mu)\leq R)=\left\{\begin{array}{ll} 0 &~~ \mbox{if $R<0$}\\
	\frac{G_{n-1,\sigma^2}(R)-G_{n-1,\sigma^2}(0)}{G_{n-1,\sigma^2}(\pi)-G_{n-1,\sigma^2}(0)} &~~ \mbox{if $0\leq R\leq \pi$}\\
	1 &~~  \mbox{if $R>\pi$}. \end{array}  \right.
\end{equation}
\end{prop}
\begin{lemma} \label{lemma2}
	\begin{equation}
	G_{m, \sigma^2}(R)-G_{m, \sigma^2}(0)=\int_0^R \mathrm{sin}^m(r)\mathrm{exp}\bigg(-\frac{r^2}{2\sigma^2}\bigg)dr \nonumber
	\end{equation}
\end{lemma}
\begin{proof}
   	\begin{align}
   	& \frac{d}{dr} \Bigg(\bigg(\frac{i}{2}\bigg)^m\sqrt{\frac{\pi\sigma^2}{2}} \sum_{j=0}^{m} {m \choose j}(-1)^j \mathrm{exp}\Big(-\frac{(m-2j)^2\sigma^2}{2}\Big)\mathrm{erf}\Big(\frac{r}{\sqrt{2\sigma^2}}+\sqrt{\frac{\sigma^2}{2}}\big(m-2j\big)i\Big)\Bigg) \nonumber\\
   	&= \bigg(\frac{i}{2}\bigg)^m\sqrt{\frac{\pi\sigma^2}{2}} \sum_{j=0}^{m} {m \choose j}(-1)^j \mathrm{exp}\Big(-\frac{(m-2j)^2\sigma^2}{2}\Big) \frac{d}{dr} \Bigg(\mathrm{erf}\Big(\frac{r}{\sqrt{2\sigma^2}}+\sqrt{\frac{\sigma^2}{2}}\big(m-2j\big)i\Big)\Bigg)\nonumber\\
   	&= \bigg(\frac{i}{2}\bigg)^m\sqrt{\frac{\pi\sigma^2}{2}} \sum_{j=0}^{m}{m \choose j}(-1)^j \mathrm{exp}\Big(2\sigma^2mj-2\sigma^2j^2-\frac{\sigma^2 m^2}{2}\Big) \frac{1}{\sqrt{2\sigma^2}}\frac{2}{\sqrt{\pi}}\mathrm{exp}\Big(-\Big(\frac{r}{\sqrt{2\sigma^2}}\notag\\
   	&\qquad+\sqrt{\frac{\sigma^2}{2}}\big(m-2j\big)i\Big)^2\Big)\nonumber\\
   	&= \bigg(\frac{i}{2}\bigg)^m\sum_{j=0}^{m} {m \choose j}(-1)^j \mathrm{exp}\Big(2\sigma^2mj-2\sigma^2j^2-\frac{\sigma^2 m^2}{2}-\bigg(\frac{r^2}{2\sigma^2} + r(m-2j)i\notag\\
   	&\qquad-2\sigma^2\big(\frac{m^2}{4}-mj+j^2\big)\bigg)\Big)\nonumber\\
   	&= \bigg(\frac{i}{2}\bigg)^m\sum_{j=0}^{m} {m \choose j}(-1)^j \mathrm{exp}\Big(2\sigma^2mj-2\sigma^2j^2-\frac{\sigma^2 m^2}{2}-\frac{r^2}{2\sigma^2}-r(m-2j)i+\frac{\sigma^2m^2}{2}\notag\\
   	&\qquad-2\sigma^2mj+2\sigma^2j^2\Big)\nonumber\\
   	&= \bigg(\frac{i}{2}\bigg)^m\sum_{j=0}^{m} {m \choose j}(-1)^j \mathrm{exp}\Big(-\frac{r^2}{2\sigma^2}-r(m-2j)i\Big)\nonumber\\
   	&= \bigg(\frac{i}{2}\bigg)^m \mathrm{exp}\Big(-\frac{r^2}{2\sigma^2}\Big)\mathrm{exp}(mri)\sum_{j=0}^{m} {m \choose j}(-1)^j \mathrm{exp}(-r(2m-2j)i)\nonumber\\
   	&= \mathrm{exp}\Big(-\frac{r^2}{2\sigma^2}\Big)\bigg(\frac{\mathrm{exp}(ri)i}{2}\bigg)^m\sum_{j=0}^{m} {m \choose j}(-1)^j \big(\mathrm{exp}(-2ri)\big)^{m-j}\nonumber\\
   	&= \mathrm{exp}\Big(-\frac{r^2}{2\sigma^2}\Big)\bigg(\frac{\mathrm{exp}(ri)i}{2}\bigg)^m \big(\mathrm{exp}(-2ri)-1\big)^m\nonumber\\
   	&= \mathrm{exp}\Big(-\frac{r^2}{2\sigma^2}\Big)\bigg(\frac{i(\mathrm{exp}(-ri)-\mathrm{exp}(ri))}{2}\bigg)^m\nonumber\\
   	&= \mathrm{exp}\Big(-\frac{r^2}{2\sigma^2}\Big)\bigg(\frac{i(\mathrm{cos}(r)-i\mathrm{sin}(r)-\mathrm{cos}(r)-i\mathrm{sin}(r))}{2}\bigg)^m\nonumber\\
   	&= \mathrm{sin}^m(r)\mathrm{exp}\Big(-\frac{r^2}{2\sigma^2}\Big) \nonumber
   	\end{align}
\end{proof}
\vspace{-5mm}
Lemma \ref{lemma2} implies that despite the presence of imaginary terms in (\ref{gr}), $G_{m, \sigma^2}(R)$ is always real.
\begin{proof}[Proof of Proposition \ref{spherenormalprop}(a)] 
As noted in Remark 3.1 in \cite{Cotton2002}, the surface area of an $(n-1)$-sphere of radius $x$ on $S^n$ is
\begin{equation}
A_{S^n}(x):=\frac{2\pi^{\frac{n}{2}}}{\Gamma(\frac{n}{2})}\mathrm{sin}^{n-1}(x) \nonumber
\end{equation}
for $0\leq x\leq\pi$. Then
\begin{align}
C_{S^n}(\mu,\sigma^2)&=\int_{S^n} \mathrm{exp}\bigg(-\frac{d(y,\mu)^2}{2\sigma^2}\bigg)dy =\int_0^\pi A_{S^n}(r)\mathrm{exp}\bigg(-\frac{r^2}{2\sigma^2}\bigg)dr \nonumber\\
&=\frac{2\pi^{\frac{n}{2}}}{\Gamma(\frac{n}{2})}\int_0^\pi \mathrm{sin}^{n-1}(r)\mathrm{exp}\bigg(-\frac{r^2}{2\sigma^2}\bigg)dr=\frac{2\pi^{\frac{n}{2}}}{\Gamma(\frac{n}{2})} \big(G_{n-1,\sigma^2}(\pi)-G_{n-1,\sigma^2}(R)\big), \nonumber
\end{align}
where the last equality comes from setting $m=n-1$ in Lemma \ref{lemma2}.
\end{proof}
\begin{proof}[Proof of Proposition \ref{spherenormalprop}(b)]
In a similar vein to the above, it is clear that the distribution function of $d(y,\mu)$ is
\[
F_{S^n}(R):=\mathrm{Pr}(d(y,\mu)\leq R)=\frac{\frac{2\pi^{\frac{n}{2}}}{\Gamma(\frac{n}{2})} \big(G_{n-1,\sigma^2}(R)-G_{n-1,\sigma^2}(0)\big)}{\frac{2\pi^{\frac{n}{2}}}{\Gamma(\frac{n}{2})} \big(G_{k-1,\sigma^2}(\pi)-G_{n-1,\sigma^2}(0)\big)}=\frac{G_{n-1,\sigma^2}(R)-G_{n-1,\sigma^2}(0)}{G_{n-1,\sigma^2}(\pi)-G_{n-1,\sigma^2}(0)},
\]
for $R\in[0,\pi]$.
\end{proof}

\subsubsection{Riemannian Normal Distribution on $\mathbb{H}^n$} \label{hypnormal}
\noindent
For $R\geq0$ and integer $m\geq 0$, define 
\begin{align} \label{hr}
& H_{m, \sigma^2}(R):=\frac{1}{2^m}\sqrt{\frac{\pi\sigma^2}{2}}\sum_{j=0}^{m} {m \choose j}(-1)^j\mathrm{exp}\Big(\frac{(m-2j)^2\sigma^2}{2}\Big)\mathrm{erf}\Big(\frac{R}{\sqrt{2\sigma^2}}-\sqrt{\frac{\sigma^2}{2}}\big(m-2j\big)\Big), 
\end{align}
with $i$ and $\mathrm{erf}$ defined as before.
\begin{prop} \label{hypnormalprop}
	(a) When $M=\mathbb{H}^n$, the normalizing constant in (\ref{normal}) is given by
	\begin{equation} \label{hypnormalconstant}
	C_{\mathbb{H}^n}(\mu,\sigma^2)=\frac{2\pi^{\frac{n}{2}}}{\Gamma(\frac{n}{2})} \big(\lim_{R^\prime\rightarrow\infty} H_{n-1,\sigma^2}(R^\prime)-H_{n-1,\sigma^2}(0)\big),
	\end{equation}
	where 
	\begin{equation} \label{limithr}
	\lim_{R\rightarrow\infty} H_{n-1,\sigma^2}(R)=\frac{1}{2^{n-1}}\sqrt{\frac{\pi\sigma^2}{2}}\sum_{j=0}^{n-1} {n-1 \choose j}(-1)^j\mathrm{exp}\Big(\frac{(n-1-2j)^2\sigma^2}{2}\Big).
	\end{equation}
	Substituting this normalizing constant into (\ref{normal}) gives the Riemannian normal density on $\mathbb{H}^n$.
	(b) The distribution function of $d(y,\mu)$ is
		\begin{equation} \label{hypf}
		F_{\mathbb{H}^n}(R):=\mathrm{Pr}(d(y,\mu)\leq R)=\left\{\begin{array}{ll} 0 &~~ \mbox{if $R<0$}\\
		\frac{H_{n-1,\sigma^2}(R)-H_{n-1,\sigma^2}(0)}{\lim_{R^\prime\rightarrow\infty} H_{n-1,\sigma^2}(R^\prime)-H_{n-1,\sigma^2}(0)} &~~ \mbox{if $R\geq0$}. \end{array}  \right.
		\end{equation}
\end{prop}
\begin{lemma} \label{lemma3}
	\begin{equation}
	\int_0^R \mathrm{exp}(-ar^2+br)dr=\sqrt{\frac{\pi}{4a}}\mathrm{exp}\Big(\frac{b^2}{4a}\Big)\mathrm{erf}\Big(\sqrt{a}r-\frac{b}{2\sqrt{a}}\Big), \nonumber
	\end{equation}
	for $a>0$, $b\in\mathbb{R}$.
\end{lemma}
\begin{proof}
	\begin{align}
	& \frac{d}{dr} \Bigg(\sqrt{\frac{\pi}{4a}}\mathrm{exp}\Big(\frac{b^2}{4a}\Big)\mathrm{erf}\Big(\sqrt{a}r-\frac{b}{2\sqrt{a}}\Big)\Bigg) \nonumber\\
	&= \sqrt{\frac{\pi}{4a}}\mathrm{exp}\Big(\frac{b^2}{4a}\Big)\frac{d}{dr}\Bigg(\mathrm{erf}\Big(\sqrt{a}r-\frac{b}{2\sqrt{a}}\Big)\Bigg) \nonumber\\
	&= \sqrt{\frac{\pi}{4a}}\mathrm{exp}\Big(\frac{b^2}{4a}\Big)\sqrt{a}\frac{2}{\sqrt{\pi}}\mathrm{exp}\Big(-\Big(\sqrt{a}r-\frac{b}{2\sqrt{a}}\Big)^2\Big) \nonumber\\
	&= \mathrm{exp}\Big(\frac{b^2}{4a}\Big)\mathrm{exp}\Big(-ar^2+br-\frac{b^2}{4a}\Big) \nonumber\\
	&= \mathrm{exp}(-ar^2+br). \nonumber
	\end{align}
\end{proof}
\vspace{-5mm}
\begin{lemma} \label{lemma4}
	\begin{equation}
	H_{m, \sigma^2}(R)-H_{m,\sigma^2}(0)=\int_0^R \mathrm{sinh}^m(r)\mathrm{exp}\bigg(-\frac{r^2}{2\sigma^2}\bigg)dr. \nonumber
	\end{equation}
\end{lemma}
\begin{proof}
	\begin{align}
	& \int_0^R \mathrm{sinh}^m(r)\mathrm{exp}\bigg(-\frac{r^2}{2\sigma^2}\bigg)dr\nonumber\\
	&= \int_0^R \bigg(\frac{(\mathrm{exp}(-r)-\mathrm{exp}(r))}{2}\bigg)^m\mathrm{exp}\bigg(-\frac{r^2}{2\sigma^2}\bigg)dr\nonumber\\
	&= \int_0^R \frac{1}{2^m}\bigg(\sum_{j=0}^{m} {m \choose j}(-1)^j\mathrm{exp}(-jr)\mathrm{exp}((m-j)r)\bigg)\mathrm{exp}\bigg(-\frac{r^2}{2\sigma^2}\bigg)dr\nonumber\\
	&= \frac{1}{2^m}\bigg(\sum_{j=0}^{m} {m \choose j}(-1)^j \int_0^R \mathrm{exp}\bigg(-\frac{r^2}{2\sigma^2}+(m-2j)r\bigg)dr\nonumber\\
	&= \frac{1}{2^m}\bigg(\sum_{j=0}^{m} {m \choose j}(-1)^j \sqrt{\frac{\pi\sigma^2}{2}}\mathrm{exp}\Big(\frac{(m-2j)^2\sigma^2}{2}\Big)\mathrm{erf}\Big(\frac{R}{\sqrt{2\sigma^2}}-\sqrt{\frac{\sigma^2}{2}}\big(m-2j\big)\Big)\nonumber\\
	&= H_{m, \sigma^2}(R), \nonumber
	\end{align}
	where the second to last inequality comes from letting $a=\frac{1}{2\sigma^2}$ and $b=m-2j$ in Lemma \ref{lemma3}.
\end{proof}
\begin{proof}[Proof of Proposition \ref{hypnormalprop}(a)] 
	Using equation (\ref{hypsa}) for the surface area of an $(n-1)$-sphere in $\mathbb{H}^n$,
	\begin{align}
	C_{\mathbb{H}}(\mu,\sigma^2)&=\int_{\mathbb{H}^n} \mathrm{exp}\bigg(-\frac{d(y,\mu)^2}{2\sigma^2}\bigg)dy =\int_0^\infty A_{\mathbb{H}}(r)\mathrm{exp}\bigg(-\frac{r^2}{2\sigma^2}\bigg)dr \nonumber\\
	&=\frac{2\pi^{\frac{n}{2}}}{\Gamma(\frac{n}{2})}\int_0^\infty \mathrm{sinh}^{n-1}(r)\mathrm{exp}\bigg(-\frac{r^2}{2\sigma^2}\bigg)dr=\frac{2\pi^{\frac{n}{2}}}{\Gamma(\frac{n}{2})} \big(\lim_{R^\prime\rightarrow\infty}H_{n-1,\sigma^2}(R^\prime)-H_{n-1,\sigma^2}(0)\big), \nonumber
	\end{align}
	where the last equality comes from setting $m=n-1$ in Lemma \ref{lemma4}. (\ref{limithr}) follows from (\ref{hr}) and the fact that $\lim_{x\rightarrow\infty} \mathrm{erf}(x)=1$.
\end{proof}
\begin{proof}[Proof of Proposition \ref{hypnormalprop}(b)]
	In a similar vein to the above, it is clear that the distribution function of $d(y,\mu)$ is
\[
		F_{\mathbb{H}^n}(R):=\mathrm{Pr}(d(y,\mu)\leq R)=\frac{\frac{2\pi^{\frac{n}{2}}}{\Gamma(\frac{n}{2})} H_{n-1,\sigma^2}(R)}{\frac{2\pi^{\frac{n}{2}}}{\Gamma(\frac{n}{2})} \lim_{R\rightarrow\infty}H_{k-1,\sigma^2}(\pi)}=\frac{H_{n-1,\sigma^2}(R)-H_{n-1,\sigma^2}(0)}{\lim_{R^\prime\rightarrow\infty}H_{n-1,\sigma^2}(R^\prime)-H_{n-1,\sigma^2}(0)},
		\]
		for $R\geq0$.
\end{proof}

\subsubsection{Generating Random Points from the Riemannian Normal Distribution}
To generate a random $y$ from the Riemannian normal distribution on $M=S^n$ or $\mathbb{H}^n$:
\begin{itemize}
	\item[1.] If $M=S^n$, draw a random $R\in[0,\pi]$ from $F_{S^n}$ in (\ref{spheref}) with $F_{S^n}^{-1}|_{(0,1)}(t)\in[0,\pi]$, where $t\in (0,1)$ is drawn from the uniform distribution $U(0,1)$. $R$ is the distance from $\mu$. Similarly, if $M=\mathbb{H}^n$, draw a random $R\geq0$ from $F_{\mathbb{H}^n}$ in (\ref{hypf}) with $F_{\mathbb{H}^n}^{-1}|_{(0,1)}(t)>0$, where $t$ is from $U(0,1)$.

	\item[2.] Draw a random unit vector $u\in T_\mu M$ from the uniform distribution on the unit $(n-1)$-sphere in $T_\mu M\cong \mathbb{R}^n$ by drawing a vector from an isotropic $n$-variate Gaussian distribution and dividing it by its magnitude. This works because all points on $M$ that are a fixed distance from $\mu$ are equally likely.

	\item[3.] Multiply the randomly drawn $R$ (magnitude) from step 1 and unit vector $u$ (direction) from step 2 to give $Ru\in T_\mu M$, a tangent vector at $\mu$, and finally $y=\mathrm{Exp}(\mu,Ru)$.
\end{itemize}

\section{} \label{appenb}
\subsection{Details about the Sphere, $S^n$} \label{appenbsphere}
\noindent
The $n$-sphere can be represented as the unit sphere embedded in $(n+1)$-dimensional Euclidean space:
\vspace{-3mm}
\[
S^n=\{p\in\mathbb{R}^{n+1}\big|\lVert p\rVert=1\},
\]
where $\lVert p\rVert=\sqrt{\langle p,p\rangle}$ and $\langle \cdot,\cdot\rangle$ is the usual dot product defined by
$
\langle p,q\rangle=\sum_{j=1}^{n+1}p^jq^j.
$
with $p=(p^1,\ldots,p^{n+1}),q=(q^1,\ldots,q^{n+1})\in\mathbb{R}^{n+1}$. The tangent space at $p\in S^n$ then consists of the vectors in $\mathbb{R}^{n+1}$ orthogonal to $p$ with respect to the dot product:
\vspace{-3mm}
\[
T_pS^n=\{v\in\mathbb{R}^{n+1}|\langle p,v\rangle=0\}.
\]
The exponential map for $S^n$ is given by 
\vspace{-3mm}
\[
\mathrm{Exp}(p,v)=\mathrm{cos}(\lVert v\rVert)p+\mathrm{sin}(\lVert v\rVert)\frac{v}{\lVert v\rVert}
\]
for $p\in S^n,~v\in T_pS^n$. For $p_1,~p_2\in S^n,~p_1\neq -p_2$, the logarithmic map is given by 
\vspace{-3mm}
\[
\mathrm{Log}(p_1,p_2)=\mathrm{cos}^{-1}( \langle p_1,p_2\rangle)\frac{p_2-\langle p_1,p_2\rangle p_1}{\lVert p_2-\langle p_1,p_2\rangle p_1\rVert}, 
\]
meaning $d(p_1,p_2)=\mathrm{cos}^{-1}( \langle p_1,p_2\rangle)$, and the parallel transport of a vector $v\in T_{p_1}S^n$ along the unique minimizing geodesic from $p_1$ to $p_2$ (provided $p_2\neq-p_1$) is given by 
\begin{equation}
\Gamma_{p_1\rightarrow p_2}(v)=v-\frac{\langle\mathrm{Log}(p_1,p_2),v\rangle}{(\lVert \mathrm{Log}(p_1,p_2)\rVert)^2}\big(\mathrm{Log}(p_1,p_2)+\mathrm{Log}(p_2,p_1)\big),
\end{equation}
or equivalently
\begin{equation}
\Gamma_{p_1\rightarrow p_2}(v)=v^\perp+\Big\langle v,\frac{\mathrm{Log}(p_1,p_2)}{\lVert \mathrm{Log}(p_1,p_2)\rVert}\Big\rangle\Big(\mathrm{cos}(\lVert \mathrm{Log}(p_1,p_2)\rVert)\frac{\mathrm{Log}(p_1,p_2)}{\lVert \mathrm{Log}(p_1,p_2)\rVert}-\mathrm{sin}(\lVert \mathrm{Log}(p_1,p_2)\rVert)p_1\Big), \notag
\end{equation}
where
\begin{equation}
v^\top=\Big\langle v,\frac{\mathrm{Log}(p_1,p_2)}{\lVert \mathrm{Log}(p_1,p_2)\rVert}\Big\rangle\frac{\mathrm{Log}(p_1,p_2)}{\lVert \mathrm{Log}(p_1,p_2)\rVert},~\mbox{and}~~v^\perp=v-v^\top, \notag
\end{equation}
which denote the parts of $v$ that are parallel and orthogonal to $\mathrm{Log}(p_1,p_2)$, respectively. The gradients with respect to $p$ and each $v^j$, calculated using Jacobi fields, are
\begin{eqnarray*}
	\nabla_pE_\rho &=& -\sum_{i=1}^N\frac{\rho^\prime(\lVert e_i\rVert)}{\lVert e_i\rVert}d_p\mathrm{Exp}(p,Vx_i)^\dag e_i \\
	& =&-\sum_{i=1}^N\frac{\rho^\prime(\lVert e_i\rVert)}{\lVert e_i\rVert}\Big(\mathrm{cos}(\lVert Vx_i\rVert)\Gamma_{\hat{y}_i\rightarrow p}^\perp(e_i)+\Gamma_{\hat{y}_i\rightarrow p}^\top(e_i)\Big),~\mbox{and}\\ 
	\nabla_{v^j}E_\rho &=&-\sum_{i=1}^Nx_i^j\frac{\rho^\prime(\lVert e_i\rVert)}{\lVert e_i\rVert}d_v\mathrm{Exp}(p,Vx_i)^\dag e_i\\
	& =&-\sum_{i=1}^Nx_i^j\frac{\rho^\prime(\lVert e_i\rVert)}{\lVert e_i\rVert}\Big(\frac{\mathrm{sin}(\lVert Vx_i\rVert)}{\lVert Vx_i\rVert}\Gamma_{\hat{y}_i\rightarrow p}^\perp(e_i)+\Gamma_{\hat{y}_i\rightarrow p}^\top(e_i)\Big),
\end{eqnarray*}
where $e_i=\mathrm{Log}(\hat{y}_i,y_i)$, and $\Gamma_{\hat{y}_i\rightarrow p}^\top(e_i)$ and $\Gamma_{\hat{y}_i\rightarrow p}^\perp(e_i)$ are defined by
\begin{equation}
\Gamma_{\hat{y}_i\rightarrow p}^\top(e_i)=\Big\langle \Gamma_{\hat{y}_i\rightarrow p}(e_i),\frac{v}{\lVert v\rVert}\Big\rangle \frac{v}{\lVert v\rVert},~\mbox{and}~~\Gamma_{\hat{y}_i\rightarrow p}^\perp(e_i)=\Gamma_{\hat{y}_i\rightarrow p}(e_i)-\Gamma_{\hat{y}_i\rightarrow p}(e_i)^\top. \notag
\end{equation}

\subsection{Details about the Hyperbolic Space, $\mathbb{H}^n$} \label{appenbhyp}
\noindent
Unlike $S^n$, hyperbolic space cannot be embedded in Euclidean space without distortion, so there exist several equivalent models for visualizing and performing calculations on this manifold. We will consider the hyperbolic model and the Poincar\'e ball model.

The hyperboloid model is particularly convenient for use in our gradient descent algorithm because several formulae are simple and analogous to the spherical case, as we will see. In this model, $\mathbb{H}^n$ is embedded in the pseudo-Euclidean $(n+1)$-dimensional Minkowski space. Originally used to model 4-dimensional spacetime in the special theory of relativity, it is equipped with the Minkowski (pseudo-)inner product, defined by
\vspace{-3mm}
\[
\langle p,q\rangle_{\textbf{M}}=-p^1q^1+\sum_{j=2}^{n+1}p^jq^j
\]
with $p=(p^1,\ldots,p^{n+1}),q=(q^1,\ldots,q^{n+1})\in\mathbb{R}^{n+1}$, instead of the usual dot product. This symmetric bilinear form is a pseudo-inner product because while it is non-degenerate, it is not positive definite. $\mathbb{H}^n$ is represented as the upper sheet of a two-sheeted $n$-dimensional hyperboloid embedded in $\mathbb{R}^{n+1}$:
\vspace{-3mm}
\[
\mathbb{H}^n=\{p=(p^1,\ldots,p^{n+1})\in\mathbb{R}^{n+1}\big|\langle p,p\rangle_{\textbf{M}}=-1,\ p^1>0\}. 
\]
The tangent space at $p\in \mathbb{H}^n$ then consists of the vectors in $\mathbb{R}^{n+1}$ orthogonal to $p$ with respect to the Minkowski inner product:
\vspace{-3mm}
\[
T_p\mathbb{H}^n=\{v\in\mathbb{R}^{n+1}|\langle p,v\rangle_{\textbf{M}}=0\}.
\]
Even though $\langle \cdot,\cdot\rangle_{\textbf{M}}$ is not positive-definite, its restriction to $T_p\mathbb{H}^n$ is, so $\mathbb{H}^n$ is a Riemannian manifold embedded in the pseudo-Riemannian Minkowski space and we can define a norm on the tangent space by $\lVert v\rVert_{\textbf{M}}=\sqrt{\langle v,v\rangle_{\textbf{M}}}$ for $v\in T_p\mathbb{H}^n$. The exponential map is then given by
\vspace{-3mm}
\[
\mathrm{Exp}(p,v)=\mathrm{cosh}(\lVert v\rVert_{\textbf{M}})p+\mathrm{sinh}(\lVert v\rVert_{\textbf{M}})\frac{v}{\lVert v\rVert_{\textbf{M}}}. 
\]
For $p_1,~p_2\in \mathbb{H}^n$, the logarithmic map is given by 
\vspace{-3mm}
\[
\mathrm{Log}(p_1,p_2)=\mathrm{cosh}^{-1}( -\langle p_1,p_2\rangle_{\textbf{M}})\frac{p_2+\langle p_1,p_2\rangle_{\textbf{M}} p_1}{\lVert p_2+\langle p_1,p_2\rangle_{\textbf{M}} p_1\rVert_{\textbf{M}}}, 
\]
meaning $d(p_1,p_2)=\mathrm{cosh}^{-1}(-\langle p_1,p_2\rangle_{\textbf{M}})$, and the parallel transport of a vector $v\in T_{p_1}\mathbb{H}^n$ along the unique minimizing geodesic from $p_1$ to $p_2$ is given by
\begin{equation}
\Gamma_{p_1\rightarrow p_2}(v)=v-\frac{\langle\mathrm{Log}(p_1,p_2),v\rangle_{\textbf{M}}}{(\lVert \mathrm{Log}(p_1,p_2)\rVert_{\textbf{M}})^2}\big(\mathrm{Log}(p_1,p_2)+\mathrm{Log}(p_2,p_1)\big),
\end{equation}
or equivalently
\begin{equation}
v^\perp+\Big\langle v,\frac{\mathrm{Log}(p_1,p_2)}{\lVert \mathrm{Log}(p_1,p_2)\rVert_{\textbf{M}}}\Big\rangle_{\textbf{M}}\Big(\mathrm{cosh}(\lVert \mathrm{Log}(p_1,p_2)\rVert_{\textbf{M}})\frac{\mathrm{Log}(p_1,p_2)}{\lVert \mathrm{Log}(p_1,p_2)\rVert_{\textbf{M}}}+\mathrm{sinh}(\lVert \mathrm{Log}(p_1,p_2)\rVert_{\textbf{M}})p_1\Big), \notag
\end{equation}
where
\begin{equation}
v^\top=\Big\langle v,\frac{\mathrm{Log}(p_1,p_2)}{\lVert \mathrm{Log}(p_1,p_2)\rVert_{\textbf{M}}}\Big\rangle_{\textbf{M}}\frac{\mathrm{Log}(p_1,p_2)}{\lVert \mathrm{Log}(p_1,p_2)\rVert_{\textbf{M}}},~\mbox{and}~~v^\perp=v-v^\top. \notag
\end{equation}
The gradients with respect to $p$ and each $v^j$, calculated using Jacobi fields, are
\begin{eqnarray*}
	\nabla_pE_\rho &=& -\sum_{i=1}^N\frac{\rho^\prime(\lVert e_i\rVert_{\textbf{M}})}{\lVert e_i\rVert_{\textbf{M}}}d_p\mathrm{Exp}(p,Vx_i)^\dag e_i \\
	& =&-\sum_{i=1}^N\frac{\rho^\prime(\lVert e_i\rVert_{\textbf{M}})}{\lVert e_i\rVert_{\textbf{M}}}\Big(\mathrm{cosh}(\lVert Vx_i\rVert_{\textbf{M}})\Gamma_{\hat{y}_i\rightarrow p}^\perp(e_i)+\Gamma_{\hat{y}_i\rightarrow p}^\top(e_i)\Big),~\mbox{and}\\ 
	\nabla_{v^j}E_\rho &=&-\sum_{i=1}^Nx_i^j\frac{\rho^\prime(\lVert e_i\rVert_{\textbf{M}})}{\lVert e_i\rVert_{\textbf{M}}}d_v\mathrm{Exp}(p,Vx_i)^\dag e_i\\
	& =&-\sum_{i=1}^Nx_i^j\frac{\rho^\prime(\lVert e_i\rVert_{\textbf{M}})}{\lVert e_i\rVert_{\textbf{M}}}\Big(\frac{\mathrm{sinh}(\lVert Vx_i\rVert_{\textbf{M}})}{\lVert Vx_i\rVert_{\textbf{M}}}\Gamma_{\hat{y}_i\rightarrow p}^\perp(e_i)+\Gamma_{\hat{y}_i\rightarrow p}^\top(e_i)\Big),
\end{eqnarray*}
where $e_i=\mathrm{Log}(\hat{y}_i,y_i)$, and $\Gamma_{\hat{y}_i\rightarrow p}^\top(e_i)$ and $\Gamma_{\hat{y}_i\rightarrow p}^\perp(e_i)$ are defined by
\begin{equation}
\Gamma_{\hat{y}_i\rightarrow p}^\top(e_i)=\Big\langle \Gamma_{\hat{y}_i\rightarrow p}(e_i),\frac{v}{\lVert v\rVert_{\textbf{M}}}\Big\rangle \frac{v}{\lVert v\rVert_{\textbf{M}}},~\mbox{and}~~\Gamma_{\hat{y}_i\rightarrow p}^\perp(e_i)=\Gamma_{\hat{y}_i\rightarrow p}(e_i)-\Gamma_{\hat{y}_i\rightarrow p}(e_i)^\top. \notag
\end{equation}

\vspace{-7mm}

The Poincar\'e ball model, along with the so-called Beltraim-Klein model, is useful for visualization. In it, hyperbolic space is represented as the interior of the unit ball in $\mathbb{R}^n$:
\vspace{-3mm}
\[
\mathbb{P}^n=\{q=(q^1,\ldots,q^n)\in\mathbb{R}^n\big|\lVert q\rVert<1\},
\]
and a geodesic is represented as either an arc of a circle that are orthogonal to the boundary of the unit ball, or a diameter of the ball. The distance between two points $q_1,q_2\in\mathbb{P}^n$ increases exponentially as they get closer to the boundary:
\vspace{-3mm}
\[
d(q_1,q_2)=\mathrm{cosh}^{-1}\Bigg(1+2\frac{\lVert q_1-q_2\rVert^2}{(1-\lVert q_1\rVert^2)(1-\lVert q_2\rVert^2)}\Bigg).
\]
The Poincar\'e ball can be constructed from the hyperboloid model via the function $g:\mathbb{H}^n\rightarrow \mathbb{P}^n$ defined by
\vspace{-3mm}
\[
g((p^1,p^2,\ldots,p^{n+1}))=\frac{(p^2,\ldots,p^{n+1})}{p^1+1}.
\]
That is, a point $p=(p^1,\ldots,p^{n+1})$ in the hyperboloid model is projected onto the interior of the unit ball in the hyperplane $x^1=0$ through the line connecting that point to $(-1,0,\ldots,0)$. The inverse of this function, $g^{-1}:\mathbb{P}^n\rightarrow\mathbb{H}^n$, mapping the open ball back to the hyperboloid model is
\vspace{-3mm}
\[
g^{-1}((q^1,q^2,\ldots,q^n))=\frac{(1+\sum_{j=1}^n (q^j)^2, 2q^1,\ldots,2q^n)}{1-\sum_{j=1}^n (q^j)^2}.
\]
This easy conversion between the two models allows one to take advantage of the strengths of both.

\subsection{Details about Kendall's 2-Dimensional Shape Space, $\Sigma_2^K$}\label{appenbkendall}
\noindent
Much of this section has been written with reference to Section 3.11 of the online supplementary document of \cite{Cornea2017} and Section 5.2.1 of \cite{Fletcher2013}.

As mentioned in Section \ref{analysis}, a shape is the geometry of an object after the effects of translation, scaling and rotation have been removed. A $K$-configuration in the two-dimensional plane can be expressed as a $K$-by-2 matrix, or equivalently as a complex $K$-vector $z=(z^1,\ldots,z^K)\in\mathbb{C}^K$. Translation is removed by subtracting the centroid $\frac{1}{n}\sum_{m=1}^K z^m$ from each element of $z$ and scaling is removed by dividing $z$ by its norm $\lVert z\rVert = \sqrt{\langle z,z\rangle}$; recall that the standard complex inner product is given by $\langle z_1,z_2\rangle = \overline{z_2}^T z_1 = \sum_{m=1}^K z_1^m \overline{z_2^m}$. In this way, we limit our consideration to $D^K=\{z\in\mathbb{C}^K | \sum_{m=1}^K z^m=0$, $\sum_{m=1}^K z^m \overline{z^m}=1\}$, which can be thought of as a unit sphere of real dimension $2K-3$. This set is called the \textit{pre-shape space}, and its elements \textit{pre-shapes}.

As only rotation remains, pre-shapes have the same shape if they are planar rotations of each other. We define an equivalence relation on $D^K$ such that all pre-shapes of the same shape are equivalent. Then two pre-shapes $z_1, z_2\in D^K$ are equivalent ($z_1\sim z_2$) if $z_1 = z_2e^{i\theta}$ for some angle $\theta$, as rotation in the complex plane is performed by multiplication by $e^{i\theta}$. So a shape is the equivalence class $p=[z_p]_\sim=\{z^\prime=z_pe^{i\theta}|\theta\in [0,2\pi)\}\subset D^K$, the set of all rotations of a pre-shape $z_p$, and is an element of the quotient space $\Sigma_2^K=D^K/S^1$, a Riemannian manifold of real dimension $(2K-4)$. This space is equivalent to $\mathbb{C}P^{K-2}$, the set of complex lines through the origin in $\mathbb{C}^{K-1}$, as the space of centered $K$-configurations is equivalent to $\mathbb{C}^{K-1}$, and scaling and rotation together are equivalent to multiplication by a complex number $re^{i\theta}$.

The manifold is endowed with the complex inner product and the tangent space at $y=[z_y]_\sim\in\Sigma_2^K$ is given by 
\begin{eqnarray*}
	T_y \Sigma_2^K&=&\{v=(v^1,\ldots,v^K)|\frac{1}{K}\sum_{m=1}^K v^m=0 \mbox{ and Re}(\langle z_ye^{i\theta},v\rangle)=0, \forall \theta\in [0,2\pi)\}\\
	&=&\{v=(v^1,\ldots,v^K)|\sum_{m=1}^K v^m=0, \langle z^\prime,v\rangle=0 \mbox{ for any } z^\prime\in [z_y]_\sim\},
\end{eqnarray*}
where Re($\langle \cdot,\cdot\rangle$) gives the real inner product when the complex $k$-vectors are instead conceptualized as real $2k$-vectors.

All calculations in shape space are done using representatives in pre-shape space. Given $z_{p_1},z_{p_2}\in D^K$, $z_{p_2}^*=\argmin_{z_{p_2}^\prime\in[z_{p_2}]_\sim} d_{D^K}(z_{p_1},z_{p_2}^\prime)$, where $d_{D^K}$ is the spherical geodesic distance on $D^K$, is the optimal rotational alignment of $z_{p_2}$ to $z_{p_1}$. It can be shown that
\begin{equation} \label{zstar}
z_{p_2}^*=z_{p_2}e^{i\theta^*} \mbox{, where } e^{i\theta^*}=\frac{\langle z_{p_1}, z_{p_2}\rangle}{\lvert\langle z_{p_1}, z_{p_2}\rangle\rvert},
\end{equation}
so that $\theta^*$ is the argument of $\langle z_{p_1}, z_{p_2}\rangle$; note that this means $\langle z_{p_1}, z_{p_2}^*\rangle=\lvert\langle z_{p_1}, z_{p_2}\rangle\rvert$ is real and positive. Then the geodesic distance $d_{\Sigma_2^K}$ between $p_1=[z_{p_1}]_\sim$ and $p_2=[z_{p_2}]_\sim$ on $\Sigma_2^K$ is 
\begin{equation}
d_{\Sigma_2^K}(p_1,p_2)=\min_{z_{p_2}^\prime\in[z_{p_2}]} d_{D^K}(z_{p_1},z_{p_2}^\prime)=d_{D^K}(z_{p_1},z_{p_2}^*)=\mathrm{cos}^{-1}(\langle z_{p_1},z_{p_2}^* \rangle) = \mathrm{cos}^{-1}(\lvert\langle z_{p_1},z_{p_2} \rangle\rvert),\notag
\end{equation}
where $z_{p_2}$ can be any element of $[z_{p_2}]_\sim$ and the geodesic distance does not depend on the choice of the representative pre-shapes. The exponential map for $\Sigma_2^K$ is given by 
\begin{equation}
\mathrm{Exp}(p,v)=\Big[\mathrm{cos}(\lVert v\rVert)z_p+\mathrm{sin}(\lVert v\rVert)\frac{v}{\lVert v\rVert}\Big]_\sim\notag,
\end{equation}
where $p=[z_p]_\sim\in\Sigma_2^K$, $v\in T_p\Sigma_2^K$. This is similar to the exponential map for the $k$-sphere. Note that the resulting pre-shape in the square brackets is optimally aligned to the representative pre-shape $z_p$. The logarithmic map is given by 
\begin{equation}
\mathrm{Log}(p_1,p_2)=\mathrm{cos}^{-1}(\langle z_{p_1},z_{p_2}^*\rangle)\frac{z_{p_2}^*-\langle z_{p_1},z_{p_2}^*\rangle z_{p_1}}{\lVert z_{p_2}^*-\langle z_{p_1},z_{p_2}^*\rangle z_{p_1}\rVert}=\mathrm{cos}^{-1}( \lvert\langle z_{p_1},z_{p_2}\rangle\rvert)\frac{z_{p_2}^*-\lvert\langle z_{p_1},z_{p_2}\rangle\rvert z_{p_1}}{\lVert z_{p_2}^*-\lvert\langle z_{p_1},z_{p_2}\rangle\rvert z_{p_1}\rVert}, \notag
\end{equation}
where $p_1=[z_{p_1}]_\sim$ and $p_2=[z_{p_2}]_\sim$ are in $\Sigma_2^K$ and $z_{p_2}^*$ is as defined in (\ref{zstar}). Note that this depends on the choice of $z_{p_1}$ but not $z_{p_2}$, and so is only valid at the at this particular representation of $p$. Parallel transport of $v\in T_p\Sigma_2^K$ along the geodesic from $p_1=[z_{p_1}]_\sim$ to $p_2=[z_{p_2}]_\sim$ is
\begin{flalign}
\Gamma_{p_1\rightarrow p_2}(v)&=e^{-i\theta^*}\Bigg\{v-\langle v,z_{p_1}\rangle z_{p_1} - \langle v,\tilde{z_{p_2}^*}\rangle \tilde{z_{p_2}^*} + \Big(\langle z_{p_2}^*,z_{p_1}\rangle \langle v,z_{p_1}\rangle - \sqrt{1-\lvert\langle z_{p_2}^*,z_{p_1}\rangle \rvert^2}\langle v,\tilde{z_{p_2}^*}\rangle\Big)z_{p_1} \notag &\\
&\qquad+\Big(\sqrt{1-\lvert\langle z_{p_2}^*,z_{p_1}\rangle \rvert^2} \langle v,z_{p_1}\rangle - \overline{\langle z_{p_2}^*,z_{p_1}\rangle}\langle v,\tilde{z_{p_2}^*}\rangle\Big)\tilde{z_{p_2}^*}\Bigg\}\notag&\\
&=\frac{\overline{\langle z_{p_1}, z_{p_2}\rangle}}{\lvert\langle z_{p_1}, z_{p_2}\rangle\rvert}\Bigg\{v-\langle v,z_{p_1}\rangle z_{p_1} - \langle v,\tilde{z_{p_2}^*}\rangle \tilde{z_{p_2}^*} + \Big(\lvert\langle z_{p_1}, z_{p_2}\rangle\rvert \langle v,z_{p_1}\rangle \notag &\\
&\qquad - \sqrt{1-\lvert\langle z_{p_1}, z_{p_2}\rangle\rvert^2}\langle v,\tilde{z_{p_2}^*}\rangle\Big)z_{p_1}+\Big(\sqrt{1-\lvert\langle z_{p_1}, z_{p_2}\rangle\rvert^2} \langle v,z_{p_1}\rangle - \lvert\langle z_{p_1}, z_{p_2}\rangle\rvert\langle v,\tilde{z_{p_2}^*}\rangle\Big)\tilde{z_{p_2}^*}\Bigg\},\notag
\end{flalign}
where $\tilde{z_{p_2}^*}=(z_{p_2}^*-\langle z_{p_2}^*, z_{p_1}\rangle z_{p_1})/\sqrt{1-\langle z_{p_2}^*, z_{p_1}\rangle^2}=(z_{p_2}^*-\lvert\langle z_{p_1}, z_{p_2}\rangle\rvert z_{p_1})/\sqrt{1-\lvert\langle z_{p_1}, z_{p_2}\rangle\rvert^2}$ and $z_{p_2}^*, \theta*$ are as defined in (\ref{zstar}). Parallel transport uses the special unitary group. Note that this depends on the choice of both $z_{p_1}$ and $z_{p_2}$, so care must be taken.

The gradients with respect to $p$ and each $v^j$, calculated using Jacobi fields, are
\begin{eqnarray*}
	\nabla_pE_\rho &=& -\sum_{i=1}^N\frac{\rho^\prime(\lVert e_i\rVert)}{\lVert e_i\rVert}d_p\mathrm{Exp}(p,Vx_i)^\dag e_i \\
	& =&-\sum_{i=1}^N\frac{\rho^\prime(\lVert e_i\rVert)}{\lVert e_i\rVert}\Big(\mathrm{cos}(\lVert Vx_i\rVert)u_i^\perp+\mathrm{cos}(\lVert 2 Vx_i\rVert)w_i^\perp+u_i^\top+w_i^\top\Big),\\ 
	\nabla_{v^j}E_\rho &=&-\sum_{i=1}^Nx_i^{j}\frac{\rho^\prime(\lVert e_i\rVert)}{\lVert e_i\rVert}d_{v}\mathrm{Exp}(p,Vx_i)^\dag e_i\\
	& =&-\sum_{i=1}^Nx_i^{j}\frac{\rho^\prime(\lVert e_i\rVert)}{\lVert e_i\rVert}\Big(\frac{\mathrm{sin}(\lVert Vx_i\rVert)}{\lVert Vx_i\rVert}u_i^\perp+\frac{\mathrm{sin}(\lVert 2Vx_i\rVert)}{\lVert 2Vx_i\rVert}w_i^\perp+u_i^\top+w_i^\top\Big),
\end{eqnarray*}
where $e_i=\mathrm{Log}(\hat{y}_i,y_i)$ and $u_i$, $w_i$ are defined as follows: Define a function $j:\mathbb{C}\rightarrow \mathbb{C}$ by $j(v)=iv$, where $i=\sqrt{-1}$, not the index. Separate $\Gamma_{\hat{y}_i\rightarrow p}(e_i)$ into components $u_i$ and $w_i$ that are orthogonal and parallel to $j(Vx_i)$ respectively, where all these vectors are conceptualized as real $2K$-vectors rather than complex $K$-vectors i.e. 
\[
w_i=\mathrm{Re}\Big(\Big\langle \Gamma_{\hat{y}_i\rightarrow p}(e_i),\frac{j(Vx_i)}{\lVert j(Vx_i)\rVert}\Big\rangle\Big) \frac{j(Vx_i)}{\lVert j(Vx_i)\rVert},~\mbox{and}~~u_i=\Gamma_{\hat{y}_i\rightarrow p}(e_i)-w_i.
\]
Then $u_i^\perp$ and $u_i^\top$ are defined by 
\begin{equation}
u_i^\top=\mathrm{Re}\Big(\Big\langle u_i,\frac{v}{\lVert v\rVert}\Big\rangle\Big) \frac{v}{\lVert v\rVert},~\mbox{and}~~u_i^\perp=u_i-u_i^\top, \notag
\end{equation}
again treating the complex $K$-vectors as real $2K$-vectors, and $w_i^\perp$ and $w_i^\top$ are defined similarly.
\end{appendices}


\begin{thebibliography}{}

\bibitem[{Banerjee et al.(2016)}]{Banerjee2016}
Banerjee, M., Chakraborty, R., Ofori, E., Okun, M. S., Vaillancourt, D. E. and Vemuri, B. C. (2016). A nonlinear regression technique for manifold valued data with applications to medical image analysis. {\it 2016 IEEE Conference on Computer Vision and Pattern Recognition (CVPR)}, 4424--4432.

\bibitem[{Cheng and Vemuri(2013)}]{Cheng2013} 
Cheng, G. and Vemuri, B. C. (2013). A novel dynamic system in the space of SPD matrices with applications to appearance tracking. {\it SIAM Journal on Imaging Sciences}, {\bf 6}, 592--615. 

\bibitem[{Cornea et al.(2017)}]{Cornea2017} 
Cornea, E., Zhu, H., Kim, P. and Ibrahim, J. G. (2017). Regression models on {Riemannian} symmetric spaces. {\it Journal of the Royal Statistical Society: Series B}, {\bf 79}, 463--482.

\bibitem[{Cotton et al.(2002)}]{Cotton2002} 
Cotton, A. and Freeman, D (2002). The double bubble problem in spherical and hyperbolic space. {\it International Journal of Mathematics and Mathematical Sciences}, {\bf 32}, 641--699.

\bibitem[{Davis et al.(2010)}]{Davis2010}
Davis, B. C., Fletcher, P. T., Bullitt, E. and Joshi, S. (2010). Population shape regression from random design data. {\it International Journal of Computer Vision}, {\bf 90}, 255--266.

\bibitem[{do Carmo(1992)}]{doCarmo1992}
do Carmo, M. (1992). {\it Riemannian Geometry}. Birkh\"auser, Boston. 

\bibitem[{Du et al.(2014)}]{Du2014}
Du, J., Goh, A., Kushnarev, S. and Qiu, A. (2014). Geodesic regression on orientation distribution functions with its application to an aging study. {\it NeuroImage}, {\bf 87}, 416--426. 

\bibitem[{Fletcher(2013)}]{Fletcher2013}
Fletcher, P. T. (2013). Geodesic regression and the theory of least squares on {Riemannian} manifolds. {\it International Journal of Computer Vision}, {\bf 105}, 171--185.

\bibitem[{Fletcher(2020)}]{Fletcher2020}
Fletcher, T. (2020). Statistics on manifolds. In {\it Riemannian Geometric Statistics in Medical Image Analysis} Edited by X. Pennec, S. Sommer and T. Fletcher, 39--74. Academic Press, London.
	
\bibitem[{Fletcher et al.(2004)}]{Fletcher2004}
Fletcher, P. T., Lu, C., Pizer, S.M. and Joshi, S. (2004). Principal geodesic analysis for the study of nonlinear statistics of shape. {\it IEEE Transactions on Medical Imaging}, {\bf 23}, 995--1005. 

\bibitem[{Fr\'echet(1948)}]{Frechet1948}
Fr\'echet, M. (1948). Les \'el\'ements al\'eatoires de nature quelconque dans un espace distanci\'e. {\it Annales de l’Institut Henri Poincar\'e}, {\bf 10}, 215–310.

\bibitem[{Hein(2009)}]{Hein2009}
Hein, M. (2009). Robust nonparametric regression with metric-space valued output. {\it Advances in Neural Information Processing Systems 22}.

\bibitem[{Hinkle et al.(2014)}]{Hinkle2014}
Hinkle, J., Fletcher, P. T. and Joshi, S. (2014). Intrinsic polynomials for regression on Riemannian manifolds. {\it Journal of Mathematical Imaging and Vision}, {\bf 50}, 32--52. 

\bibitem[{Hong et al.(2016)}]{Hong2016}
Hong, Y., Singh, N., Kwitt, R., Vasconcelos, N. and Niethammer, M. (2016). Parametric regression on the Grassmannian. {\it IEEE Transactions on Pattern Analysis and Machine Intelligence}, {\bf 38}, 2284--2297. 

\bibitem[{Kim et~al.(2014)}]{Kim2014}
Kim, H. J.,  Adluru, N., Collins, M. D., Chung, M. K., Bendin, B. B., Johnson, S. C., Davidson, R. J. and Singh, V. (2014).  Multivariate general linear models ({MGLM}) on {Riemannian} manifolds with applications to statistical analysis of diffusion weighted images. {\it 2014 IEEE Conference on Computer Vision and Pattern Recognition}, 2705--2712. 

\bibitem[{Mortici(2012)}]{Mortici2012}
Mortici, C. (2012). Completely monotone functions and the Wallis ratio. {\it Applied Mathematics Letters}, {\bf 25}, 717--722. 

\bibitem[{Shin(2020)}]{Shin2020}
Shin, H.-Y. (2020). Robust geodesic regression. M.S. Thesis, Seoul National University. SNU Open Repository.

\bibitem[{Steinke and Hein(2008)}]{Steinke2008} 
Steinke, F. and Hein, M. (2008). Non-parametric regression between manifolds. {\it Advances in Neural Information Processing Systems 21}.

\bibitem[{Steinke et al.(2010)}]{Steinke2010} 
Steinke, F., Hein, M. and Sch\"olkopf, B. (2010). Nonparametric regression between general Riemannian manifolds. {\it SIAM Journal on Imaging Sciences}, {\bf 3}, 527--563. 

\bibitem[{Zhang et~al.(2019)}]{Zhang2019}
Zhang, X., Shi, X., Sun, Y. and Cheng, L. (2019). Multivariate regression with gross errors on manifold-valued data. {\it IEEE Transactions on Pattern Analysis and Machine Intelligence}, {\bf 41}, 444--458. 


\end{thebibliography}
\end{document}